\documentclass[11pt]{article}
\usepackage{fullpage}

\usepackage{xr-hyper}
\RequirePackage{amsthm,amsmath,amsfonts,amssymb}
\RequirePackage[authoryear]{natbib}%% uncomment this for author-year citations
\RequirePackage[colorlinks,citecolor=blue,urlcolor=blue,filecolor=red]{hyperref}
\RequirePackage{graphicx}
\usepackage{enumitem}
\usepackage{soul}
\usepackage{booktabs}
\usepackage{bm}
\usepackage{algorithm,algorithmic}
\usepackage{booktabs,multirow}
\usepackage{adjustbox}

\usepackage{fullpage}
\usepackage[dvipsnames]{xcolor}
\usepackage{booktabs}
\usepackage{array}

\newtheorem{theorem}{Theorem}
\newtheorem{lemma}[theorem]{Lemma}
\newtheorem{prop}[theorem]{Proposition}
\newtheorem{coro}[theorem]{Corollary}

\newtheorem{assumption}{Assumption}

\usepackage{tcolorbox}
\usepackage[toc,page]{appendix}

\title{Demystifying Group Relative Policy Optimization: Its Policy Gradient is a U-Statistic}
\author{
Hongyi Zhou$^{1,*}$ \and
Kai Ye$^{2,*}$ \and
Erhan Xu$^{2}$ \and
Jin Zhu$^{3}$ \and
Ying Yang$^{4}$ \and
Shijin Gong$^{5,\dagger}$ \and
Chengchun Shi$^{2,\dagger}$
}

\date{}

\begin{document}
\maketitle
\begingroup
\renewcommand{\thefootnote}{\fnsymbol{footnote}}
\footnotetext[1]{Equal contribution.}
\footnotetext[2]{Corresponding authors.}
\endgroup

% numeric affiliation footnotes
\begingroup
\renewcommand{\thefootnote}{\arabic{footnote}}
\footnotetext[1]{Department of Mathematics, Tsinghua University; Email: \texttt{zhou-hy21@mails.tsinghua.edu.cn}.}
\footnotetext[2]{Department of Statistics, London School of Economics and Political Science; Emails: \texttt{K.Ye1@lse.ac.uk}, \texttt{E.Xu2@lse.ac.uk}, \texttt{c.shi7@lse.ac.uk}.}
\footnotetext[3]{School of Mathematics, University of Birmingham; Email: \texttt{j.zhu.7@bham.ac.uk}.}
\footnotetext[4]{Department of Statistics and Data Science, Tsinghua University; Email: \texttt{yangying@tsinghua.edu.cn}.}
\footnotetext[5]{School of Management, University of Science and Technology of China; Email: \texttt{shijin49@mail.ustc.edu.cn}.}
\endgroup

\begin{abstract}
Group relative policy optimization (GRPO), a core methodological component of DeepSeekMath and DeepSeek-R1, has emerged as a cornerstone for scaling reasoning capabilities of large language models. Despite its widespread adoption and the proliferation of follow-up works, the theoretical properties of GRPO remain less studied. This paper provides a unified framework to understand GRPO through the lens of classical U-statistics. We demonstrate that the GRPO policy gradient is inherently a U-statistic, allowing us to characterize its mean squared error (MSE), derive the finite-sample error bound and asymptotic distribution of the suboptimality gap for its learned policy. Our findings reveal that GRPO is asymptotically equivalent to an oracle policy gradient algorithm -- one with access to a value function that quantifies the goodness of its learning policy at each training iteration -- and achieves asymptotically optimal performance within a broad class of policy gradient algorithms. Furthermore, we establish a universal scaling law that offers principled guidance for selecting the optimal group size. Empirical experiments further validate our theoretical findings, demonstrating that the optimal group size is universal, and verify the oracle property of GRPO.
\end{abstract}

\section{Introduction}
Since the birth of ChatGPT in 2022, large language models (LLMs) have been increasingly integrated into our work and everyday life. These models are evolving at an extremely fast pace and can now perform a broad range of tasks at or beyond human level, ushering in an era of rapid technological advancement across scientific, industrial, and societal domains. One of the underlying technologies that enhance the capabilities of these LLMs is \textit{reasoning}. The idea is straightforward: because LLMs generate human language, they can be guided to perform multi-step reasoning in a manner similar to humans. Initially, this is often achieved through some simple \textit{magical prompts} \citep{wei2022chain} -- for example, “Let us think step by step” -- which encourage the model to decompose complex questions into intermediate reasoning steps and produce more accurate solutions; see Figure \ref{fig:reasoning_pipeline} for an illustration.

\begin{figure}[t]
    \centering
    \includegraphics[width=\textwidth]{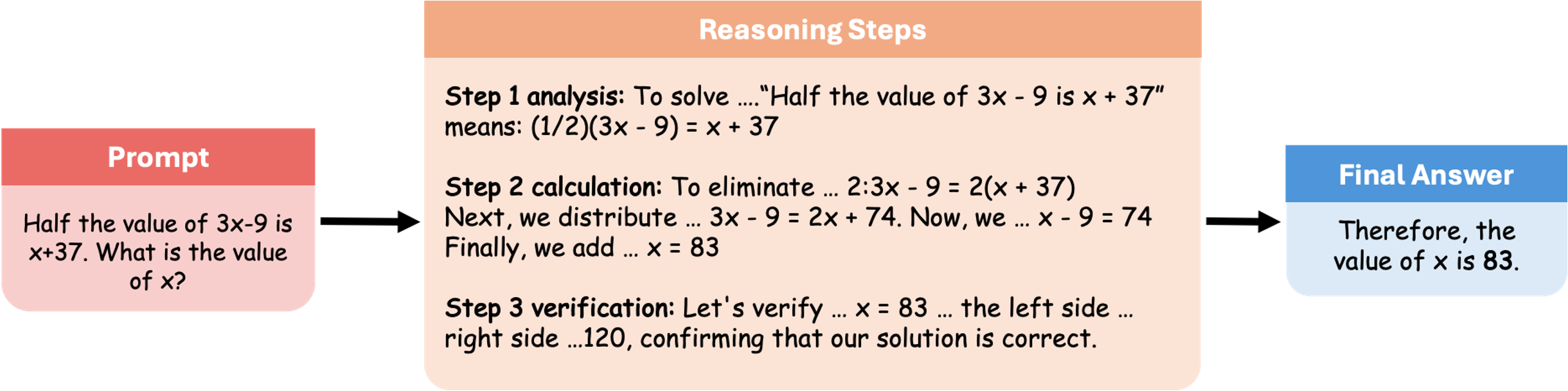}
    \vspace{-2em}
    \caption{An illustration of LLM reasoning. The example shows a prompt and the corresponding model output, consisting of a reasoning trace and a final solution.}
    \label{fig:reasoning_pipeline}
\end{figure} 

Early progress in LLM reasoning was driven largely by inference-time techniques that avoid retraining the language model, focusing on eliciting or searching for more effective intermediate reasoning steps. 
In parallel, \citet{ouyang2022training} established a scalable reinforcement learning with human feedback (RLHF) pipeline for retraining -- more precisely, post-training -- LLMs to align their outputs with human values. Its main idea is to (i) collect multiple responses for each prompt and obtain human preference over these responses; (ii) learn a reward model from the resulting preference data \citep{christiano2017deep}; and (iii) apply reinforcement learning (RL), specifically the proximal policy optimization algorithm \citep[PPO,][]{schulman2017proximal} to fine-tune the LLM parameters so as to maximize the cumulative estimated reward. This work provided a practical demonstration that RL, when combined with human preference feedback, can reliably steer LLM behavior at scale. 

The RLHF pipeline is applicable to LLM reasoning. However, it introduces two challenges. First, reasoning tasks often produce long intermediate trajectories, making human supervision for labeling responses time-consuming. This challenge can be addressed by \textit{reinforcement learning with verifiable rewards} \citep[RLVR,][]{lambert2025tulu}, which similarly employs RL for post-training but replaces the reward function learned from subjective human preference with objective verifiers in applications such as mathematics problems, where solutions are unique, or programming tasks, where the generated code can be executed to verify its correctness. The second challenge is more technical: PPO requires to learn a critic network that quantifies the quality of the model at each step to reduce the variance of the stochastic gradient. Nonetheless, estimating and storing such a network in reasoning tasks is computationally expensive. 

Group relative policy optimization \citep[GRPO,][]{shao2024deepseekmath} represents a major milestone for RL algorithms tailored to LLM reasoning. It provides a highly effective solution to the second challenge by eliminating the critic network entirely and instead sampling multiple outputs for each prompt, using their group average as a proxy for the critic; refer to Figure \ref{fig:grpo_pipeline} for an overview of the GRPO pipeline, and Section \ref{sec:preliminary} for details. This approach leads to the development of DeepSeek-R1, a prominent large reasoning model at the time, whose post-training requires only about 147K H800 GPU-hours -- an order of magnitude lower than many contemporary large reasoning models, including OpenAI’s o1 series \citep{jaech2024openai}. 

The methodology was further formalized in a paper published in Nature \citep{guo2025deepseek}, marking the first peer-reviewed LLM article in the journal, and was subsequently deployed by various open-source LLMs. These efforts established GRPO as a foundational RLVR algorithm for LLM reasoning, with numerous follow-up methods built upon it shortly after its introduction (see Section \ref{sec:RLVR} for a review). Beyond LLM reasoning, GRPO has also been applied as a general-purpose RL engine in other domains, including agent settings, where the goal extends beyond chatbots to enabling the model to use external tools or perform more complex tasks \citep[e.g.,][]{qian2025toolrl,ding2026treegrpo}. 
 
\begin{figure}[t]
    \centering
    \includegraphics[width=\textwidth]{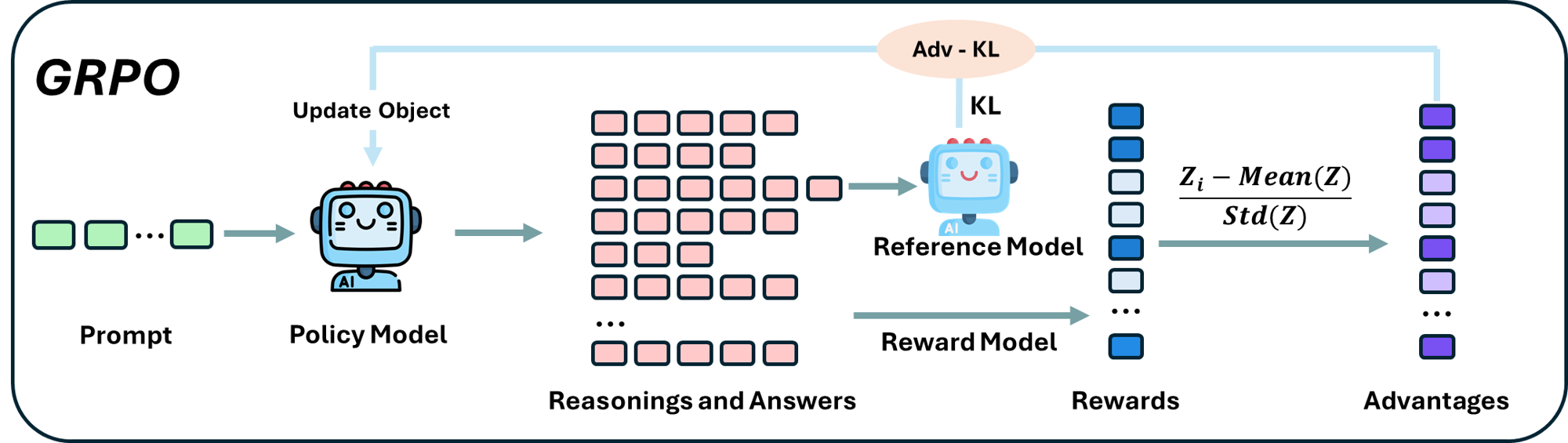}
    \caption{Overview of the GRPO pipeline. For each prompt, GRPO samples multiple reasoning traces to generate outputs, which are then evaluated by a reward model to measure their quality. These rewards are compared against the group mean and standardized to compute the advantage function, which is used to update the policy model, subject to KL regularization with respect to a reference model.}
    \label{fig:grpo_pipeline}
\end{figure}

All of the aforementioned developments highlight the practical appeal of GRPO, but they also reveal a theoretical gap, as formal analyses remain limited in the literature. In particular: 
\begin{enumerate}[leftmargin=*]
    \item[Q1.]\textit{Why is GRPO so effective?}
    \item[Q2.]\textit{What is the rationale for using the group mean to approximate the critic network?}
    \item[Q3.]\textit{Can we provide finite-sample or asymptotic analyses regarding its convergence?}
    \item[Q4.]\textit{How many outputs shall we sample per prompt?}
\end{enumerate}
 This paper addresses these questions by demystifying GRPO from a statistical perspective. Our key observation is that GRPO is deeply connected to \textit{U-statistics} \citep{hoeffding1948class}, a connection that has not been explicitly recognized in prior works. Building on this observation, we conduct a comprehensive analysis of GRPO, covering both \textit{finite-sample} and \textit{asymptotic} properties, examining both the original algorithm and its variants, and analyzing both the \textit{policy gradient estimator} used to update model parameters at each iteration and the \textit{suboptimality gap}, which measures the difference between the learned policy and the optimal policy; see Figure \ref{fig:theory_roadmap} for a visualization of our theoretical framework. 

Our theoretical contributions are as follows:
\begin{enumerate}[leftmargin=*]
    \item We establish the first connection between GRPO and U-statistics by showing that the GRPO policy gradient is inherently a U-statistic (Lemma \ref{lem:gradientUstat}). This addresses Q2 by providing a principled explanation for using the group mean to approximate the critic network through classical U-statistics theory.  
    \item To address Q3, we provide finite-sample analyses that characterize the mean squared error (MSE) of the GRPO policy gradient (Theorem \ref{thm:gradMSE} \& Proposition \ref{prop:MSE}) and derive error bounds on its suboptimality gap (Lemma \ref{lemma:policynonasympotic}). We further establish parameter consistency and derive the asymptotic distribution of the suboptimality gap without requiring parameter identifiability (Theorem \ref{thm:policyasympotic}). This result is novel in two respects: (i) existing literature primarily focuses on error bounds for the suboptimality gap, which characterize its order but are not as accurate as its distribution; and (ii) classical asymptotic analyses rely on parameter identifiability, an assumption clearly violated in overparameterized LLMs.

\item Our finite-sample and asymptotic analyses lead to two desirable properties of GRPO: (i) the \textit{oracle property} (Corollaries \ref{coro:gradoracle} \& \ref{coro:policyoracle}), whereby GRPO is asymptotically equivalent to an oracle algorithm with access to the true critic network; and (ii) \textit{optimality} (Corollaries \ref{coro:gradoptimal} \& \ref{coro:policyoptimal}), whereby GRPO asymptotically minimizes both the MSE and the suboptimality gap among a broad class of RL algorithms. These theoretical findings are further supported by empirical evidence (Figure \ref{fig:toy example}), which addresses Q1 and explains GRPO’s effectiveness.

\item Finally, to address Q4, we derive a \textit{scaling law} that delineates how GRPO’s performance depends on the number of sampled outputs per group, and identifies the optimal group size that maximizes its performance (Theorem \ref{thm:policynonasympotic}). Notably, this optimal choice depends only on the training data and model architecture, and is independent of other factors such as the training budget or number of iterations. This universality makes our scaling law particularly appealing, as it does not require retuning when these factors change. Its universality is further validated empirically (Table \ref{tab:batch_group_transposed_4dp} \& Figure \ref{fig:gsms8k_choice_g_steps_1024}).
\end{enumerate}

\begin{figure}[t]
    \centering
    \includegraphics[width=\textwidth]{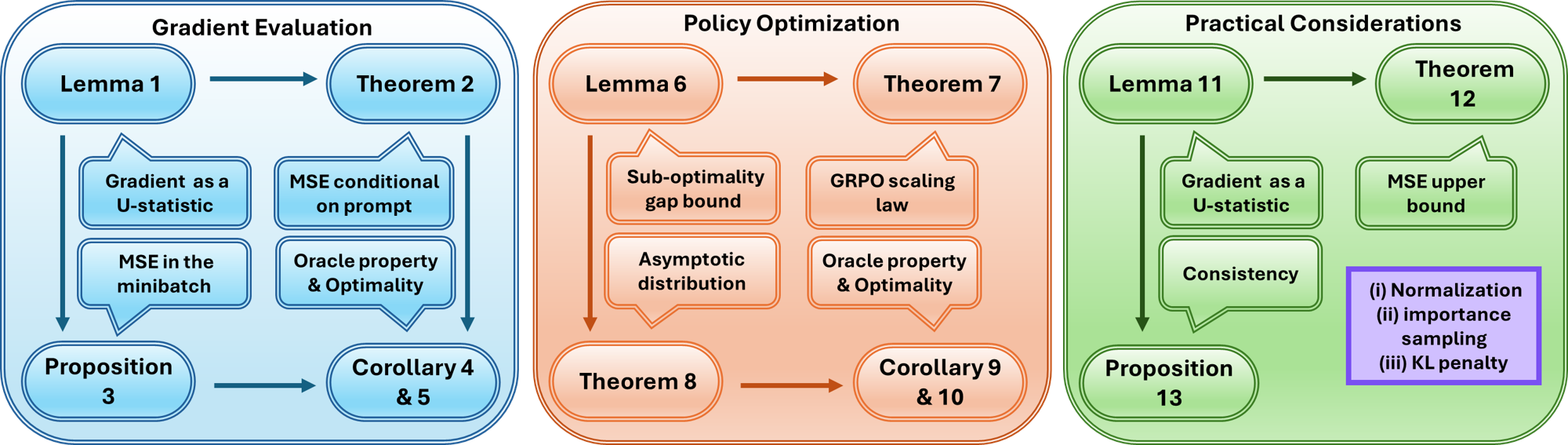}
    \vspace{-2em}
    \caption{Roadmap of our theoretical results.}
    \label{fig:theory_roadmap}
\end{figure}

\section{Related works} Our work connects two modern areas in artificial intelligence -- reinforcement learning (Section \ref{subsec:RL}), and its emerging application to LLM reasoning through RLVR (Section \ref{sec:RLVR}) -- with the classical theory of U-statistics (Section \ref{subsec:Ustat}).
\subsection{Reinforcement learning}\label{subsec:RL}
The literature on RL is vast. Existing algorithms are broadly distinguished as \textit{planning} or \textit{learning}, based on whether the data generating process is known \citep[e.g.,][Chapter 8]{sutton2018reinforcement}. Within learning, \textit{model-based} approaches explicitly estimate the MDP model \citep[e.g.,][]{jiang2024note}, while \textit{model-free} methods do not. The latter can be further divided into \textit{value-based} algorithms that learn a value function to measure the goodness of a policy and derive the optimal policy by maximizing this value, and \textit{policy-based} algorithms that directly searches the optimal policy over a restricted policy class. Over time, RL research has evolved across four phases, summarized below in chronological order:
\begin{enumerate}[leftmargin=*]
    \item \textbf{Classical RL}: Early works studied non-deep-learning algorithms, in the form of tabular methods that store estimates in lookup tables, or using classical ML models for function approximation. Two foundational examples are tabular Q-learning \citep{watkins1992q} and fitted Q-iteration \citep{ernst2005tree}. Among these algorithms, GRPO is closely related to policy-based algorithms such as REINFORCE \citep{williams1992simple} and actor-critic \citep{konda1999actor}. As demonstrated later, GRPO is in essence a policy-based algorithm adapted for LLM reasoning. 
    \item \textbf{Deep RL}: Spurred by the deep learning revolution of the 2010s, a line of advanced deep RL algorithms emerged. This era was largely catalyzed by the success of the deep Q-network in mastering video games \citep{mnih2015human} and the development of AlphaGo for the game of Go \citep{silver2016mastering}. Since then, deep RL has become a cornerstone of modern AI research both methodologically \citep[e.g.,][]{mnih2016asynchronous, van2016deep, dabney2018distributional, zhou2020non, chen2021decision} and theoretically \citep[e.g.,][]{fan2020theoretical, feng2023over, shen2025deep, sun2025intrinsic}. Of particular relevance to GRPO are trust region policy optimization \citep[TRPO,][]{schulman2015trust} and its successor, PPO. The latter has been a prominent policy-based algorithm widely applied across various domains, including robotics \citep{andrychowicz2020learning} and LLM fine-tuning \citep{ouyang2022training}.
    \item \textbf{Offline RL}: In high-stakes applications where safety concerns make online exploration prohibitive, it is more practical to employ offline RL that learns exclusively from static, historical datasets \citep{levine2020offline}. The core principle of offline RL is to apply the pessimistic principle \citep[e.g.,][]{jin2021pessimism, rashidinejad2021bridging} for conservative policy learning. While this principle traces back to the seminal works of \citet{swaminathan2015self, swaminathan2015batch}, the paradigm saw a resurgence in the early 2020s \citep[e.g.,][]{kumar2019stabilizing, wu2019behavior, yu2020mopo, xie2021bellman, ueharapessimistic2022}. Parallel to offline policy learning, off-policy evaluation seeks to evaluate the impact of adopting a target policy based on the historical data generated by a different policy \citep[e.g.,][]{thomas2015high,jiang2016doubly,thomas2016data,liu2018breaking,xie2019towards,kallus2022efficiently}; see \citet{uehara2022review} for a recent review.
    \item \textbf{RLHF}: Following the seminal work by \citet{ouyang2022training}, the field has seen an unprecedented proliferation of literature dedicated to RLHF \citep{christiano2017deep}, with the goal of aligning the output of LLMs with human preferences through RL \citep[e.g.,][]{munos2023nash,rafailov2023direct,wu2024pairwise,wu2024sppo,zeng2024tokenlevel}. Theoretically, existing works have explored both the asymptotic distribution of the parameter estimates \citep{liu2024dual} and finite-sample error bounds for the sub-optimality gap \citep[e.g.,][]{chowdhury2024provablyrobustdpoaligning,zhong2024provable,aminian2025kl,ye2025robust,xu2025doubly} and regret \citep{zhang2025iterativenashpolicyoptimization}. Despite these analyses, GRPO -- a key driver of recent breakthroughs in LLM reasoning -- remains largely under-theorized.
\end{enumerate}
Finally, RL is closely related to two branches of research in statistics, primarily motivated by healthcare applications; see \citet{chakraborty2013statistical, kosorok2019precision, tsiatis2019dynamic, li2023optimal, shi2025statisticalinferencereinforcementlearning,ge2025reviewcausaldecisionmaking,gazi2026statisticalreinforcementlearningreal} for reviews. (i) Early works develop RL algorithms to learn optimal dynamic treatment regimes (DTRs), which are individualized strategies that tailor medical interventions to a patient’s unique characteristics and evolving clinical state. Some representative algorithms include Q-learning \citep{qian2011performance,song2015penalized}, A-learning \citep{murphy2003optimal,robins2004optimal,shi2018high} and policy-based methods \citep{zhang2013robust,zhao2015new}. These studies focus on short-horizon settings, where treatment decisions are made at a single stage or over a small, finite number of stages. (ii) More recently, the literature has expanded to study MDPs under long-horizon settings \citep[e.g.,][]{ertefaie2018constructing,luckett2020estimating,liao2022batch,chen2024reinforcement,li2024settling,shi2024statistically,shi2024value,zhou2024estimating,ma2025sequential,jin2025policy,li2025reinforcement,zhong2025risksensitive} as well as RLHF \citep[e.g.,][]{lee2024lowrankcontextualreinforcementlearning,liu2025uncertaintyquantificationlargelanguage,liu2025statistical,lu2025contextual,xiao2025algorithmic}.

\subsection{Reinforcement learning from verifiable rewards}\label{sec:RLVR}
RLVR, an LLM post-training strategy introduced by \citet{lambert2025tulu}, directly optimizes LLMs against \emph{verifiable} outcomes. Unlike RLHF, which requires to learn a reward function from subjective human preferences across multiple candidate responses, RLVR leverages objective feedback signals -- for example, by verifying whether a model’s answer matches a ground-truth mathematical solution or by executing the model's generated code in programming tasks. The original algorithm in \citet{lambert2025tulu} relied on PPO for policy learning, which learns a separate critic model that evaluates the quality of the learning policy.

GRPO, a major breakthrough in RLVR, drastically scales the reasoning capacities of existing LLMs. Its key ingredient lies in completely eliminating the critic model, sampling multiple reasoning traces, and using their average reward as a proxy for the critic (see Section \ref{sec:preliminary} for details). 
This approach, originally introduced in DeepSeekMath \citep{shao2024deepseekmath}, was largely popularized by DeepSeek-R1 \citep{guo2025deepseek} and later validated by a series of pioneering open-source reasoning models like Qwen2.5 \citep{qwen2025qwen25technicalreport}. Together, these works have sparked a surge of follow-up RLVR algorithms \citep[see, e.g.,][for a recent review]{zhang2025system}, which can be broadly categorized into three types:
\begin{enumerate}[leftmargin=*]
    \item The first line of works refined the GRPO policy gradient estimator by (i) modifying \citep{hao2025policy,xiao2025bnpo,zeng2025shrinking} or replacing \citep{li2023remax,ahmadian2024backtobasics,hu2025reinforce++} the baseline term (the empirical group mean in GRPO); (ii) revising the importance sampling ratio used in GRPO \citep{zheng2025group,pang2025theory} or removing it altogether \citep{chu2025gpg}; and (iii) applying different normalizations to the reward  \citep{xiong2025minimalist,liu2025understanding,xiao2025bnpo}. In summary, these algorithms preserve the critic-free GRPO framework while modifying how policy gradient estimators are computed or normalized.
     
    \item The second line of works considered different optimization objectives, including entropy-regularized objectives that encourage exploration to prevent the model from entropy collapse \citep{zhang2025empo,cheng2025reasoning,chen2025seedgrpo}, length- or difficulty-aware objectives that balance correctness with concise reasoning \citep{zhang2025grpo_lead,dai2025grpolambda}, risk-sensitive targets \citep{ren2026riskpo} and objectives that guide the model to adaptively select reasoning formats \citep{wu2025arm}.
    
    \item The third line of works focused on improving GRPO’s training efficiency, either statistically or computationally \citep{dai2025cde,yu2025dapo,lin2025cppo,xu2025pods,zhang2025srpo}. For example, \citet{yan2025learning} leveraged off-policy data from stronger models to provide informative reasoning trajectories and enhance learning. \citet{li2025repo} and \citet{zhan2026exgrpo} reused previously generated trajectories instead of discarding all old samples after each policy update. \citet{zheng2025parallel} explored multiple reasoning paths concurrently to accelerate reasoning speed, while \citet{xu2026singlestream} generated only a single sampled output per prompt to reduce computational cost.  
\end{enumerate}

Despite the popularity in developing practical RLVR algorithms, their theoretical foundations -- and those of GRPO specifically -- remain largely unexplored. Among those available, \citet{liu2025understanding} and \citet{yang2026your} studied the biases of GRPO's policy gradient and advantage function (see Section \ref{sec:preliminary} for formal definitions). \citet{pang2025theory} upper bounded the expected squared $\ell_2$-norm of the gradient for GRPO's KL-regularized objective. \citet{davis2025objective} and \citet{vojnovic2025alignment} characterized GRPO's objective function. In particular, \citet{davis2025objective} proved GRPO optimizes an $\arcsin$ transformation of the expected reward, rather than the expected reward in its original form. More recently, \citet{yao2026grouprelative} discussed the off-policy nature of GRPO. However, these results do not deliver a unified finite-sample and asymptotic characterization of GRPO, nor do they establish a connection to U-statistics -- both of which are precisely what our theory establishes.

\subsection{U-statistics}\label{subsec:Ustat}
Our work builds upon the classical theory of U-statistics, a class of estimators introduced by \citet{hoeffding1948class} that generalize the sample mean to averages over functions of multiple random variables or vectors. Let $\{X_i\}_{i=1}^n$ denote a sequence of independent and identically distributed (i.i.d.) random vectors. A U-statistic of order $m$ is defined by a symmetric kernel function $h(X_1, \dots, X_m)$. For instance, a second-order U-statistic with a bivariate kernel $h$ is defined as:
\begin{equation}\label{eqn:Ustatistic}
    U=\frac{1}{n(n-1)}\sum_{i\neq j} h(X_i,X_j).
\end{equation}
Its statistical property is largely characterized through the Hoeffding decomposition, which expresses the U-statistic as a sum of several orthogonal components:
\begin{equation}\label{eqn:Udecompose}
    U=h_0+\underbrace{\frac{2}{n}\sum_{i=1}^n [h_1(X_i)-h_0]}_{\textrm{first-order term}}+\underbrace{\frac{1}{n(n-1)}\sum_{i\neq j} [h(X_i,X_j)-h_1(X_i)-h_1(X_j)+h_0]}_{\textrm{second-order term}},
\end{equation}
where $h_0 = \mathbb{E}[h(X_1, X_2)]$ is the expectation of the kernel and $h_1(x) = \mathbb{E}[h(x, X_2)]$ denotes the first-order projection. 

In the non-degenerate case where $\text{Var}(h_1(X_1)) > 0$, the first-order term dominates the variance and fluctuates at an order of $O_p(n^{-1/2})$, whereas the second-order term decays at a faster rate of $O_p(n^{-1})$. In essence, this decomposition establishes the asymptotic equivalence between a U-statistic and a simple average of i.i.d. random variables (captured by the first-order term), allowing classical limit theorems for sample averages to be directly applied to U-statistics. As we will demonstrate in Section \ref{sec:theory}, this decomposition is instrumental in characterizing the behavior of GRPO's policy gradient and its estimated optimal policy.

Owing to their favorable theoretical properties, U-statistics are widely employed across disciplines, including their extension to U-processes in probability theory \citep{Nolan1987}, their use in semiparametric statistics via higher-order influence functions \citep{liu2026semiparametricefficientempiricalhigher}, their role in econometrics through the maximum rank correlation estimator \citep{han1987, sherman1993}, and their application in estimating optimal individualized treatment regimes via concordance-assisted learning \citep{fan2017concordance, liang2018sparse}.

\section{Preliminaries}\label{sec:preliminary}
In this section, we first introduce the sequential decision making problem in RL and formulates LLM reasoning as a sequential decision making problem (Section \ref{sec:setup}). We next present a meta-algorithm (see Algorithm \ref{alg1} for the pseudocode) that unifies a range of policy-based approaches, ranging from the classical REINFORCE algorithm to more advanced advantage actor-critic methods and variants of GRPO (Section~\ref{sec:meta}).
\subsection{Problem setup}\label{sec:setup}
RL is a powerful machine learning framework for solving sequential decision making problems. In this framework, an agent repeatedly interacts with an environment. At each time step, the agent receives an \textit{observation} which represents the environment's current state, and selects an \textit{action} based on this information. In response, the environment provides feedback in the form of a \textit{reward} and transitions to a new state, yielding the \textit{next observation}. The agent's objective is to utilize the collected observation-action-reward tuples to learn an optimal \textit{policy} -- a mapping from its observed data history to the space of actions -- that maximizes the expected cumulative reward in the long run. 

In natural language processing, text is represented as a sequence of discrete units known as tokens. Let $\mathcal{V}$ denote the vocabulary consisting of all such tokens. LLMs produce text via autoregressive next token prediction. Specifically, given an input prompt $X$ (e.g., a user query) represented as a sequence of tokens, the model generates the first token $Y_1 \in \mathcal{V}$ sampled from a conditional distribution $\pi_{\theta}(\bullet|X) = \mathbb{P}(Y_1 = \bullet|X)$, parameterized by $\theta\in \Theta$. Next, the second token $Y_2$ is produced according to $\pi_{\theta}(\bullet|X, Y_1)$. At each time step $t$, the model generates $Y_t$ according to $\pi_{\theta}(\bullet|X, Y_{<t})$, where $Y_{<t} = (Y_1, \dots, Y_{t-1})$ represents the previously generated prefix. This iterative procedure continues until the model generates a complete output $Y = (Y_1, \dots, Y_T)$, terminating with an ``end-of-sequence'' token $Y_T$. To handle conditioning sets of varying lengths, a Transformer architecture \citep{vaswani2017attention} is commonly employed to parameterize $\pi_{\theta}$. For reasoning tasks, the output $Y$ contains both a reasoning trace and a final solution (denoted by $S$); see Figure~\ref{fig:reasoning_pipeline} for an illustration.

The above procedure can be naturally formulated as a sequential decision making problem. Specifically, the observation is fixed as the input prompt $X$, the action at time step $t$ corresponds to the generated token $Y_t$, and the policy is given by $\pi_{\theta}$. Unlike standard RL settings where rewards may be provided at each step, the reward here is sparse and observed only upon completion of the entire sequence. In GRPO, the terminal reward, denoted by $Z$, evaluates both the format consistency of the response (i.e., its adherence to a pre-specified template) and the accuracy of the final solution $S$. All intermediate rewards are set to zero. Our goal is to optimize the policy $\pi_{\theta}$, or equivalently, the parameter $\theta$, to maximize the expected reward of the generated output $\mathbb{E}^{\pi_{\theta}}(Z)$. 

Alternatively, this problem can be viewed from a different perspective by treating the entire output sequence $Y$ as a single action. This effectively collapses the time horizon $T$ to $1$. Consequently, the problem is recast as a bandit problem \citep[e.g.,][]{lai1985asymptotically}, where the data is summarized by the context-action-reward $(X, Y, Z)$ tuples. 

\subsection{A meta-algorithm}\label{sec:meta}
Existing RLVR algorithms are policy-based algorithms, grounded in the policy gradient theorem \citep{sutton1999policy}. This theorem is instrumental as it provides a closed-form expression for the gradient of the expected reward with respect to the policy parameters, which enables the application of gradient-based algorithms for parameter optimization. In our specific formulation, the gradient of the expected reward can be expressed as:
\begin{equation}\label{eqn:policygradient}
    g(\theta):=\nabla_{\theta} [\mathbb{E}^{\pi_{\theta}}(Z)]= \mathbb{E} \Big[\sum_t \nabla_{\theta} \log \pi_{\theta}(Y_t|X,Y_{<t})Z\Big]. 
\end{equation}
Equation \eqref{eqn:policygradient}  gives rise to the classical REINFORCE algorithm \citep{williams1992simple}. At the $i$th iteration, the algorithm (i) samples a prompt $X$ from a reference distribution $f(\bullet)$\footnote{While the true distribution $f$ is unknown, it is typically approximated by the empirical distribution of a finite set of prompts. In such cases, our established theoretical guarantees are understood to be conditional on this specific prompt set.}, (ii) utilizes the current parameter $\theta_i$ to generate an output $Y$, (iii) computes its reward $Z$, (iv) constructs a plug-in estimator $\widehat{g}(\theta_i)=\sum_t \nabla_{\theta} \log\pi_{\theta}(Y_t|X,Y_{<t})Z$ based on these samples, and (v) updates the parameters via stochastic gradient ascent $\theta_{i+1} \leftarrow \theta_i + \eta_i \widehat{g} (\theta_i)$ with learning rate $\eta_i>0$. 

\begin{algorithm}[t]
\caption{A meta-algorithm for LLM reasoning.}\label{alg1}
\begin{algorithmic}[1]
    \STATE \textbf{Input:}  Prompt distribution $f(X)$, initial parameters $\theta_0 \in \Theta$, sequence of learning rates $\{\eta_i\}_{i \in \mathbb{N}}$, batch size $B$, per-prompt group size $G$, and baseline functions $\{C_i^{(b,g)}\}_{\substack{i,b,g}}$.
    \FOR{$i = 0, 1, 2, \dots,n-1$}
        \STATE Sample a batch of prompts $\{X^{(b)}\}_{b=1}^B \stackrel{iid}{\sim} f(\bullet)$.
        \STATE For each $X^{(b)}$, generate a group of $G$ outputs $\{Y^{(b,g)}\}_{g=1}^G \sim \pi_{\theta_i}(\cdot|X^{(b)})$.
        \STATE For each $Y^{(b,g)}$, obtain its reward $Z^{(b,g)}$.
        \STATE Compute the gradient
        $$\widehat{g}(\theta_i)=\frac{1}{BG} \sum_{b=1}^B \sum_{g=1}^G \sum_{t} \nabla_\theta \log \pi_{\theta_i}(Y_{t}^{(b,g)} | X^{(b)}, Y_{<t}^{(b,g)})(Z^{(b,g)}-C_i^{(b,g)})$$
        \STATE Parameter update: $\theta_{i+1} \gets \theta_i + \eta_{i+1} \widehat{g}(\theta_i)$.
    \ENDFOR
    \STATE \textbf{Output:} $\pi_{\theta_n}$. 
\end{algorithmic}
\end{algorithm}

While REINFORCE's gradient estimator is unbiased, it is notoriously susceptible to high variance. A common trick in the literature to reduce its variance is to subtract a baseline term $C_i$ (which can vary across iterations $i$) from the reward $Z$. The rationale lies in that the score function $\sum_{t} \nabla_\theta \log \pi_{\theta}(Y_{t} | X, Y_{<t})$ has an expectation of zero. Consequently, provided that $C_i$ is a function exclusively of the random variable $X$ (being conditionally independent of $Y$ and $Z$ given $X$), the expectation in \eqref{eqn:policygradient} remains invariant when $Z$ is replaced by its ``centered'' version, $Z - C_i$. As such, the resulting gradient estimator remains unbiased. However, a proper choice of $C_i$ can effectively reduce the variance of the estimator. 

The most widely adopted $C_i$ is the value function $V^{\pi_{\theta_i}}(X) = \mathbb{E}^{\pi_{\theta_i}}(Z | X)$. This choice gives rise to advantage actor-critic (A2C) algorithms \citep[e.g.,][]{mnih2016asynchronous}, which, in addition to the policy (the actor), learn a value function (the critic) to estimate the baseline and calculate the advantage function (the difference between the return and the critic) for variance reduction. %As mentioned in the introduction, this critic network is also a core component of PPO. 
However, in reasoning tasks, maintaining and updating a separate critic network is extremely computationally intensive. GRPO addresses this inefficiency by eliminating the critic network entirely. It samples multiple outputs for each prompt and utilizes their group mean as the baseline term (see Equation \eqref{eqn:GRPObaseline}). Because sampling from the policy is drastically more efficient than training a critic model, GRPO offers a highly scalable solution to LLM reasoning. In Section \ref{sec:theory}, we show that this solution is also mathematically elegant. 

To unify these methods, we introduce the meta-algorithm in Algorithm~\ref{alg1}. It incorporates (i) the centering trick for variance reduction, (ii) minibatch sampling of $B$ prompts per iteration, and (iii) $G$ sampled outputs per prompt for evaluating the gradient. Several aforementioned algorithms arise as special cases: 
\begin{itemize}[leftmargin=*]
    \item \textbf{REINFORCE}: Recovered by setting $B=1, G=1$, and $C_i^{(1,1)}=0$.
    \item \textbf{A2C}: Recovered by setting $B=1, G=1$, and $C_i^{(1,1)}$ to the critic network.
    \item \textbf{GRPO-type}: Recovered by allowing $B, G > 1$ and setting $C_i^{(b,g)}$ to the \textit{leave-one-out group mean}\footnote{This is equivalent to using the full group mean $G^{-1}\sum_{k=1}^G Z^{(b,k)}$ as the baseline, up to a rescaling of the learning rate by a factor of $G/(G-1)$.}
    \begin{equation}\label{eqn:GRPObaseline}
        C_i^{(b,g)}= \bar{Z}^{(b,-g)} :=\frac{\sum_{k\neq g} Z^{(b,g)}}{G-1}.
    \end{equation}
\end{itemize}
We conclude this section by noting that while we focus on analyzing Algorithm \ref{alg1} for building intuition, it simplifies the standard production-level implementation in three aspects: (i) it omits the reward normalization used in GRPO; (ii) it does not incorporate the importance sampling strategy used in GRPO to enhance sample efficiency through multiple gradient updates per batch; (iii) it excludes the Kullback–Leibler (KL) divergence penalty, which prevents the policy from deviating excessively from the reference model. These simplifications were implemented in GRPO variants such as Dr-GRPO \citep{liu2025understanding} and GPG \citep{chu2025gpg}, which we analyze rigorously in Sections \ref{sec:grge} and \ref{sec:grpo}. In Section \ref{sec:ext} of the Supplementary Material, we will bridge these gaps to align our theoretical analyses with standard GRPO implementations.

\section{Main results}\label{sec:theory}
This section presents our main theoretical results, where we compare GRPO with two alternatives to demonstrate its effectiveness. To ensure a fair comparison, all algorithms use the same batch size $B$ and group size $G$. As summarized in Table~\ref{tab:4algorithms2compare}, their differences lie solely in the choice of the baseline term. Specifically, we analyze:

\noindent (i) a \textbf{vanilla} algorithm, a minibatch variant of REINFORCE that sets $C_i^{(b,g)} = 0$;

\noindent (ii) a \textbf{GRPO-type} algorithm, which adopts the group mean defined in \eqref{eqn:GRPObaseline};

\noindent (iii) an \textbf{oracle} algorithm, which uses the true value function $V^{\pi_{\theta_i}}(X)$ as the baseline; 

Notice that the oracle algorithm is not practically implementable: the true value function is unknown and estimating it via a separate network is computationally expensive, as discussed earlier. Nevertheless, this algorithm serve as a benchmark whose performance practical algorithms aim to match. We say that an algorithm achieves the \emph{oracle property} if it attains asymptotically equivalent performance to the oracle algorithm.

We next provide a high-level summary of our findings, covering both finite-sample and asymptotic results (see Figure~\ref{fig:theory_roadmap} for a roadmap of our results):  
\begin{itemize}[leftmargin=*]
    \item \textbf{Gradient evaluation}: Section \ref{sec:grge} investigates the properties of policy gradient estimators employed by GRPO-type algorithms. Specifically, Lemma \ref{lem:gradientUstat} formally establishes the connection between the estimator and the U-statistic. Theorem \ref{thm:gradMSE} and Proposition \ref{prop:MSE} derive bounds on its MSE, which in turn demonstrating the estimator's superiority over the vanilla algorithm, yielding its oracle property (Corollary \ref{coro:gradoracle}) and optimality (Corollary \ref{coro:gradoptimal}). 
    \item \textbf{Policy optimization}: Section \ref{sec:grpo} shifts the focus to the sub-optimality gap of the learned policy. In particular, Lemma \ref{lemma:policynonasympotic} presents a finite-sample upper bound on this gap, upon which Theorem \ref{thm:policynonasympotic} establishes a scaling law that offers insights into the optimal choice of the group size. Theorem \ref{thm:policyasympotic} further establishes the consistency of the parameter estimates and derives the asymptotic distribution of the sub-optimality gap as a mixture of $\chi^2$ random variables without assuming parameter identifiability. These results verify the oracle property of the GRPO policy (Corollary \ref{coro:policyoracle}) as well as its optimality (Corollary \ref{coro:policyoptimal}).
    \item \textbf{Practical considerations}: Section \ref{sec:ext} of the Supplementary Material extends the aforementioned analyses to accommodate (i) reward standardization; (ii) importance sampling and (iii) the KL divergence penalty. This section connects the gradient estimator with U-statistics (Lemma \ref{lem:pgobjective}), characterizes the gradient estimator's MSE (Theorem \ref{thm:practicalgradMSE}), and derives the consistency of the parameter estimates (Proposition~\ref{prop:target}).
\end{itemize}

\begin{table}
    \centering
    \caption{Three meta-algorithms compared in Section~\ref{sec:theory}, differing in their choice of baseline term.}\label{tab:4algorithms2compare}
    \begin{tabular}{c|c|c|c}\toprule
      Algorithm & Vanilla algorithm & GRPO-type algorithm & Oracle algorithm \\ \midrule
      Baseline term & $C_i^{(b,g)}=0$  & $C_i^{(b,g)}=\bar{Z}^{(b,-g)}$ & $C_i^{(b,g)}=V^{\pi_{\theta_i}}(X^{(b)})$ \\  \bottomrule
    \end{tabular}
\end{table}

\subsection{Group relative gradient evaluation}\label{sec:grge}
In this section, we  consider the task of evaluating the gradient $g(\theta)$ (see Equation \eqref{eqn:policygradient}). To build intuition, we begin with estimating $g(\theta)$ for a fixed prompt $x$. This corresponds to the setting where the prompt distribution $f$ in Algorithm \ref{alg1} is degenerate at $x$. The observed data consists of $G$ i.i.d. output-reward pairs $\{(Y^{(g)},Z^{(g)})\}_{g=1}^G$ sampled from the model's policy $\pi_{\theta}$. Here, we omit the superscript $(b)$ as the prompt is fixed. We denote the resulting gradient estimators for the vanilla, GRPO-type and oracle algorithms as $\widehat{g}_{\textrm{vanilla}}(x;\theta)$, $\widehat{g}_{\textrm{GRPO}}(x;\theta)$ and $\widehat{g}_{\textrm{oracle}}(x;\theta)$, respectively. These estimators follow the general form,
\begin{equation}\label{eqn:generalgradest}
   \widehat{g}(x;\theta)=\frac{1}{G}\sum_{g=1}^G \nabla_{\theta} \log \pi_{\theta}(Y^{(g)}\mid x) [Z^{(g)}-C^{(g)}],
\end{equation}
with different choices of the baseline term $C^{(g)}$ summarized in Table \ref{tab:4algorithms2compare}. 

We begin by demonstrating that the GRPO's group relative gradient estimator, which contrasts a realized reward $Z$ against its group average, admits a \emph{second-order U-statistic} representation.

\begin{lemma}[Gradient estimator as a U-statistic]\label{lem:gradientUstat}
$\widehat{g}_{\rm GRPO}(x;\theta)$ can be written as a second-order U-statistic:
\begin{equation}\label{eq:grpo_ustat}
\widehat{g}_{\rm GRPO}(x;\theta)
=\binom{G}{2}^{-1}\sum_{1\le i<j\le G} h\big((Y^{(i)},Z^{(i)}),(Y^{(j)},Z^{(j)})\big),
\end{equation}
with a symmetric kernel
\begin{equation*}
h\big((Y^{(i)},Z^{(i)}),(Y^{(j)},Z^{(j)})\big)
:=\frac{1}{2}[\nabla_\theta \log \pi_\theta(Y^{(i)}|x)-\nabla_\theta \log \pi_\theta(Y^{(j)}|x)](Z^{(i)}-Z^{(j)}).
\end{equation*}
\end{lemma}
The proof of Lemma \ref{lem:gradientUstat} is straightforward. By setting $C^{(g)}$ in \eqref{eqn:generalgradest} to the group mean baseline $\bar{Z}^{(-g)}=(G-1)^{-1} \sum_{k\neq g} Z^{(g)}$, each individual term in the sum satisfies 
\begin{equation*}
    \nabla_{\theta} \log \pi_{\theta}(Y^{(g)}\mid x) [Z^{(g)}-\bar{Z}^{(-g)}]=\frac{1}{G-1}\sum_{k\neq g} \nabla_{\theta} \log \pi_{\theta}(Y^{(g)}\mid x) [Z^{(g)}-Z^{(k)}]. 
\end{equation*}
Averaging over all $g \in \{1, \dots, G\}$ and applying a standard symmetrization argument allows $\widehat{g}_{\textrm{GRPO}}(x;\theta)$ to be reformulated as the average of the symmetric kernel $h$ over all pairs, yielding the $U$-statistic in \eqref{eq:grpo_ustat}. 

Next, applying the Hoeffding decomposition from \eqref{eqn:Udecompose} partitions the group relative gradient estimator into three orthogonal components:

\noindent (i) the expectation of the kernel, which equals the true gradient $$g(x;\theta)=\mathbb{E}^{\pi_{\theta}} [\nabla_{\theta} \log \pi_{\theta}(Y\mid x) \mathbb{E}(Z|X=x)];$$
\noindent (ii) a first-order term, which can be represented as the difference between the \textit{oracle} gradient estimator and $g(x;\theta)$, since the first-order projection $h_1$ satisfies 
\begin{equation*}
2h_1(y,z) = \mathbb{E} \Big[ h\big((y,z), (Y,Z)\big) \Big] = \nabla_\theta \log \pi_\theta(y \mid x) \big[ z - \mathbb{E}^{\pi_{\theta}}(Z \mid X=x) \big],
\end{equation*} 
which is precisely each summand in \eqref{eqn:generalgradest} under the oracle baseline $C^{(g)}=V^{\pi_{\theta}}(x)$; 

\noindent (iii) a second-order degenerate term. 

Since these components are uncorrelated, this decomposition leads to the following MSE bound, as summarized in Theorem \ref{thm:gradMSE}. \begin{assumption}[Bounded reward]\label{assump:boundedreward}
    $Z$ is almost surely bounded.
\end{assumption}

\begin{theorem}[MSE conditional on the prompt]\label{thm:gradMSE}
    Under Assumption \ref{assump:boundedreward}, we have
    \begin{eqnarray}\label{eqn:MSEgradest}
    \begin{split}
        \text{MSE}(\widehat{g}_{\text{GRPO}}(x;\theta)):=\mathbb{E}\|\widehat{g}_{\text{GRPO}}(x;\theta)-g(x;\theta) \|^2 \\=\frac{\text{trace}[\Sigma_{\text{oracle}}(x;\theta)]}{G}+O\Big(\frac{\mathbb{E} \|\nabla_\theta \log \pi_{\theta}(Y|x) \|^2 }{G^2}\Big),
    \end{split}
    \end{eqnarray}
    where $\text{trace}[\bullet]$ denotes the trace of a matrix, $\Sigma_{\text{oracle}}(x;\theta)$ denotes the asymptotic covariance matrix of the oracle gradient estimator $G\text{Cov}(\widehat{g}_{\text{oracle}}(x;\theta))$ and $\|\bullet\|$ denotes the $\ell_2$-norm of a vector. 
\end{theorem}
We make two remarks. First, Assumption \ref{assump:boundedreward} is generally satisfied. In GRPO, the reward $Z$ is typically the sum of a format reward and an accuracy reward, which evaluate the format consistency and the correctness of the output, respectively. Since both components are bounded within $[0, 1]$, the final reward $Z$ remains bounded within $[0, 2]$. 

Second, the two terms on the second line of \eqref{eqn:MSEgradest} correspond to the expected squared $\ell_2$-norm of the first- and second-order projections from the Hoeffding decomposition. Crucially, the first-order projection serves as the \textit{leading} term that scales at a rate of $G^{-1}$ and coincides with the MSE of the oracle gradient estimator. In contrast, the second-order term decays at a faster rate of $G^{-2}$, confirming its role as a higher-order \textit{residual}. This observation establishes the oracle property of $\widehat{g}_{\text{GRPO}}(\theta)$, which we formalize in Corollary \ref{coro:gradoracle}: as the group size $G \to \infty$, the estimator's MSE becomes equivalent to that of the oracle estimator with access to the true value function. In the official DeepSeekMath implementation, $G$ is set to 64 \citep{shao2024deepseekmath}, which is sufficiently large for the residual term to be negligible relative to the leading term. 

Next, we extend Theorem \ref{thm:gradMSE} to the minibatch setting, where the prompt is no longer fixed at $x$. Instead, we sample $B$ i.i.d. prompts $\{X^{(b)}\}_{b=1}^B$ and $G$ output-reward pairs per prompt to estimate 
\begin{eqnarray}\label{eqn:gtheta}
g(\theta) = \mathbb{E}_{X \sim f(\bullet)} [g(X; \theta)].
\end{eqnarray}
Accordingly, we define the vanilla, oracle, and GRPO-type minibatch gradient estimators as the empirical average of the prompt-specific estimators: 
\begin{equation*}
\widehat{g}(\theta) = \frac{1}{B} \sum_{b=1}^B \widehat{g}(X^{(b)}; \theta).
\end{equation*}
\begin{prop}[MSE in the minibatch setting]\label{prop:MSE}
    Under Assumption \ref{assump:boundedreward}, we have
    \begin{eqnarray}\label{eqn:MSEgradbatch}
    \begin{split}
        & \text{MSE}(\widehat{g}_{\text{GRPO}}(\theta)):=\mathbb{E}\|\widehat{g}_{\text{GRPO}}(\theta)-g(\theta) \|^2 \\=&\frac{\mathbb{E} \| g(X;\theta) - g(\theta)\|^2}{B}+\frac{\text{trace}[\Sigma_{\text{oracle}}(\theta)]}{BG}+O\Big(\frac{\mathbb{E} \|\nabla_\theta \log \pi_{\theta}(Y|X) \|^2 }{BG^2}\Big),
    \end{split}
    \end{eqnarray}
    where $\Sigma_{\text{oracle}}(\theta)=\mathbb{E}[\Sigma_{\text{oracle}}(X;\theta)]$.
\end{prop}
In the minibatch setting, the MSE is further reduced by a factor of $B^{-1}$. The first term in the second line of \eqref{eqn:MSEgradbatch} represents the inherent variance arising from prompt sampling, which is independent of $G$ and shared by all gradient estimators (e.g., vanilla, oracle). Equation \eqref{eqn:MSEgradbatch} is instrumental in deriving our scaling law detailed in the next section. To elaborate, suppose we fix the sampling budget per prompt as $N = B \times G$. Then $B = N/G$, so the first term scales linearly with $G$, the second term is constant when $N$ is fixed, and the third term scales inversely with $G$. Consequently, there exists an optimal choice of $G$ that is neither too small nor too large, balancing the first and third terms to minimize the MSE. As we will show later, the MSE is closely related to the suboptimality gap of the learned policy, and thus there exists a corresponding $G$ that minimizes this suboptimality gap.

Additionally, as $G \to \infty$, the MSE of the GRPO-type gradient again becomes equivalent to that of the oracle gradient, and we formalize such oracle property below.

\begin{coro}[Oracle property of gradient estimator]\label{coro:gradoracle}
Let $\text{MSE}_A(\bullet)$ denote the asymptotic MSE of a gradient estimator, obtained by removing errors that are high-order in the group size $G$, we have $\text{MSE}_{A}(\widehat{g}_{\rm GRPO}(x;\theta)) = \text{MSE}(\widehat{g}_\text{oracle}(x;\theta))$ and  
    $\text{MSE}_{A}(\widehat{g}_{\rm GRPO}(\theta)) = \text{MSE}(\widehat{g}_\text{oracle}(\theta)).$
\end{coro}

Finally, we establish the optimality of the GRPO estimator in Corollary \ref{coro:gradoptimal} by demonstrating the following properties: (i) it asymptotically minimizes the MSE within the class of gradient estimators of the form \eqref{eqn:generalgradest} where the baseline term is a function of the prompt $x$ only; and (ii) its asymptotic MSE is strictly smaller than that of the vanilla algorithm. 

\begin{assumption}[Conditional uncorrelation]\label{assump:CU}
    $\|\nabla_{\theta}\log \pi_{\theta}(Y|X)\|^2$ is conditionally uncorrelated of $Z$ given $X$. 
\end{assumption}

\begin{coro}[Optimality of gradient estimator]\label{coro:gradoptimal}
    Suppose Assumptions \ref{assump:boundedreward} and \ref{assump:CU} hold. For any gradient estimator $\widehat{g}(x;\theta)$ of the form \eqref{eqn:generalgradest} whose baseline term is a function of the prompt $x$ only, we have
    \begin{equation*}
        \text{MSE}_{A}(\widehat{g}_{\rm GRPO}(x;\theta))\leq\text{MSE}(\widehat{g}(x;\theta))\quad \text{and} \quad\text{MSE}_{A}(\widehat{g}_{\rm GRPO}(\theta))\leq\text{MSE}(\widehat{g}(\theta)).
    \end{equation*}
    In particular, provided that the score functions $\nabla_{\theta} \log \pi_{\theta}(Y \mid x)$, $\nabla_{\theta} \log \pi_{\theta}(Y \mid X)$ and the value function $V^{\pi_{\theta}}(X)$ are not almost surely zero, and $V^{\pi_{\theta}}(x)\neq 0$,
    \begin{equation*}
        \text{MSE}_{A}(\widehat{g}_{\rm GRPO}(x;\theta))< \text{MSE}(\widehat{g}_{\text{vanilla}}(x;\theta)) \quad \text{and} \quad \text{MSE}_{A}(\widehat{g}_{\rm GRPO}(\theta))< \text{MSE}(\widehat{g}_{\text{vanilla}}(\theta)).
    \end{equation*}
\end{coro}
To conclude this section, we remark that Assumption \ref{assump:CU} ensures that the oracle gradient estimator minimizes the MSE within the class of unbiased policy gradient estimators \citep{greensmith2004variance}. This assumption is well-supported by the empirical success of modern RL algorithms such as advantage actor-critic and PPO, which prioritize the value function as the optimal baseline for variance reduction.

\subsection{Group relative policy optimization}\label{sec:grpo}
This section turns to policy optimization and investigates the properties of the policy learned by Algorithm~\ref{alg1}. For a given policy $\pi$, we evaluate its performance via the suboptimality gap, defined as the difference in the expected return between $\pi$ and the optimal policy:
\begin{equation*}
    \Delta(\pi)=\max_{\theta\in \Theta}\mathbb{E}^{\pi_{\theta}} (Z) -  \mathbb{E}^{\pi} (Z)
\end{equation*}
By definition, a smaller gap indicates a policy closer to the optimal one. Additionally, the suboptimality gap measures only the quality of the final policy obtained at the last iteration. In the RL literature, an alternative metric is \emph{regret}, which quantifies the cumulative difference between the optimal policy and the sequence of intermediate policies across iterations. However, we adopt the suboptimality gap because it better reflects practical LLM applications: once training is complete, only the final policy is deployed in practice. This criterion also aligns with the evaluation metric in deployment-efficient RL \citep[e.g.,][]{huangtowards2022}.

We begin with the following set of conditions under which we establish a finite-sample error bound for the suboptimality gap of the meta-algorithm in Lemma \ref{lemma:policynonasympotic}. 
\begin{assumption}[$L$-smoothness]\label{assump:lipcon}
The expected return $\mathbb{E}^{\pi_\theta}(Z)$ is $L$-smooth with respect to $\theta$. That is, its gradient $ g(\bullet)$ is differentiable and Lipschitz continuous with some constant $L > 0$ such that:
\begin{equation*}
\|g(\theta_1) - g(\theta_2)\| \le L \|\theta_1 - \theta_2\|, \quad \forall \theta_1, \theta_2 \in \Theta.
\end{equation*}
Additionally, the score function $\nabla_{\theta}\log \pi_{\theta}$ is uniformly bounded and continuous in $\theta$.
\end{assumption}
\begin{assumption}[Polyak-Lojasiewicz (PL) condition]\label{assump:PL}
There exists some constant $\mu>0$ such that 
\begin{equation}\label{eqn:PLinequality}
    \|g(\theta)\|^2\ge 2\mu \Delta(\pi_{\theta}), \quad \forall \theta\in \Theta. 
\end{equation}
\end{assumption}
\begin{assumption}[Learning rate]\label{assump:LR}
The sequence of learning rates $\{\eta_i\}_{i\ge 1}$ in Algorithm \ref{alg1} satisfies either (a) a \textit{constant} schedule where $\eta_i = \beta$ for some $0< \beta < (2L)^{-1}$, or (b) an $1/i$ schedule where $\eta_i = i^{-1}\beta$ for some constant $\beta > (2\mu)^{-1}$.   
\end{assumption}
The $L$-smoothness condition (Assumption \ref{assump:lipcon}) is standard in the optimization literature \citep[e.g.,][]{nesterov2013introductory}. It requires the gradient of the expected return to be a Lipschitz continuous function of $\theta$. The PL condition (Assumption \ref{assump:PL}) is substantially weaker than the strong concavity condition commonly imposed in the literature \citep[e.g.,][]{boyd2004convex}. Specifically, strong concavity requires a unique global optimizer and a strictly negative definite Hessian matrix $\nabla^2_{\theta} \mathbb{E}^{\pi_\theta}(Z)$ -- both requirements are likely violated in over-parameterized models such as LLMs. 

The PL condition, to the contrary, accommodates landscapes with multiple global optimizers and singular Hessians. 
To see this, consider a hypothetical two-dimensional example where $\theta=(\theta_1, \theta_2)$, $\Theta=\mathbb{R}^2$, and the expected return is $\mathbb{E}^{\pi_\theta}(Z) = -\theta_1^2$. Its Hessian matrix $\begin{pmatrix} -2 & 0 \\ 0 & 0 \end{pmatrix}$ 
has an eigenvalue of $0$ and is therefore not negative definite, violating strong concavity. Nonetheless, its gradient is $g(\theta) = (-2\theta_1, 0)^\top$, which satisfies \eqref{eqn:PLinequality} for any $\mu \le 2$. 

Finally, Assumption \ref{assump:LR} considers two standard learning rate schedules. The constant schedule (a) is common in practical applications \citep{sheng2025hybridflow}. To the contrary, the $1/i$ schedule (b) is motivated by stochastic approximation theory: it ensures that the suboptimality gap theoretically achieves the optimal convergence rate \citep{zhang2016central}.

\begin{lemma}[Finite-sample sub-optimality gap]\label{lemma:policynonasympotic}
     Suppose Assumptions \ref{assump:boundedreward} and \ref{assump:lipcon}--\ref{assump:LR} hold. Suppose each $C_i^{(b,g)}$ in Algorithm \ref{alg1} is either a function of $X^{(b)}$ or a GRPO-type baseline in \eqref{eqn:GRPObaseline}. Let $M = \sup_{\theta} \textrm{MSE}(\widehat{g}(\theta))$ denote the uniform upper bound on the MSE of this algorithm's gradient estimator. 
     Then for any $n\ge 1$, under the constant  schedule (a), the output policy of $\pi_{\theta_n}$ satisfies
     \begin{equation*}
        \mathbb{E}[\Delta(\pi_{\theta_n})]\le (1-2\mu \beta+L\mu \beta^2)^n \mathbb{E}[\Delta(\pi_{\theta_0})]+\frac{L\beta^2 M}{4\mu\beta-2L\mu \beta^2},
    \end{equation*}
    where the expectation is taken with respect to the randomness in the estimator $\theta_n$ and any randomness in the initial parameter $\theta_0$. 
    Under the $1/i$ schedule (b), we have
    \begin{equation*}
        \mathbb{E}[\Delta(\pi_{\theta_n})]\le \max\Big(\frac{(1+\epsilon)L\beta^2 M}{(4\mu\beta-2)n},\frac{c}{n}\Big),
    \end{equation*}
    for any sufficiently small $\epsilon > 0$, where $c$ is a positive constant depending only on $(\beta, \mu, L, \epsilon)$. 
\end{lemma}
Lemma \ref{lemma:policynonasympotic} provides a unified analysis for the meta-algorithm, covering vanilla, oracle, and GRPO-type estimators under both learning rate schedules. The bounds reveal that the sub-optimality gap: (i) scales with the smoothness constant $L$; (ii) decreases with the PL constant $\mu$ and the number of iterations $n$; and (iii) is crucially governed by the MSE of the gradient estimator $M$. As this MSE is the dominant factor (for the constant schedule) and the convergence rate constant (for the $1/i$ schedule), the oracle property and optimality we established for the GRPO gradient estimator directly translates into the learned policy's performance. By substituting the MSE bounds derived in Proposition \ref{prop:MSE} into Lemma \ref{lemma:policynonasympotic}, we can explicitly characterize how the group size $G$ and batch size $B$ affects the suboptimality gap, leading to the following scaling law for GRPO:

\begin{theorem}[Scaling law for GRPO]\label{thm:policynonasympotic}
    For GRPO-type algorithms, the sub-optimality upper bounds established in Lemma \ref{lemma:policynonasympotic} depend on $B$ and $G$ through the following quantity
    \begin{equation}\label{eqn:quantity}
        \frac{c_1}{B}+\frac{c_2}{BG}+\frac{c_3}{BG^2},
    \end{equation}
    where the constants are given by $c_1=\sup_{\theta} \mathbb{E} \|g(X;\theta)-g(\theta)\|_2^2$, $c_2=\sup_{\theta}\text{trace}[\Sigma_{\text{oracle}}(\theta)]$ and $c_3=O(\sup_{\theta} \mathbb{E} \|\nabla_\theta \log\pi_{\theta}(Y|X) \|^2)$. 
    
    Under a fixed sampling budget $N = BG$ per iteration or a total sampling budget $\bm{N}=nBG$, the optimal group size $G^*$ that minimizes the sub-optimality upper bounds is:
    \begin{equation}\label{eqn:scalinglaw}
        G^* = \sqrt{\frac{c_3}{c_1}}.
    \end{equation}
\end{theorem}
In GRPO, sampling reasoning traces from the learning policy (Line 4 of Algorithm~\ref{alg1}) is computationally expensive.
With a fixed sampling budget, Equation \eqref{eqn:quantity} reveals a trade-off in the allocation of computational resources. A \textit{small} $G$ allows for a larger batch size $B$ and more iterations $n$, which reduces the first variance term in \eqref{eqn:quantity} and accelerates the convergence rates established in Lemma \ref{lemma:policynonasympotic}. Conversely, a \textit{large} $G$ reduces the higher-order residual term (the third term in \eqref{eqn:quantity}). 

Our scaling law proposed in \eqref{eqn:scalinglaw} explicitly balances this trade-off. Crucially, the optimal group size $G^*$ is \textit{universal}: it is independent of the budget $N$, the number of iterations $n$ and the learning rate schedule. Instead, $G^*$ depends solely on the underlying data generating process and the geometry of the policy space. This makes our result highly practical for implementation, as the optimal group size remains constant for a given task regardless of the total compute available. We will verify this observation empirically in Section \ref{sec: exp}. While the constants $c_i$ are formally defined as suprema over the parameter space, they may yield conservative theoretical bounds. In practice, these suprema can be relaxed by estimating the constants at a few representative parameter values (e.g., those of the reference model).

We remark that the aforementioned results apply only to upper bounds on the suboptimality gap. Although GRPO attains a sharper upper bound, this does not necessarily imply that it achieves a strictly smaller suboptimality gap. To  rigorously establish its superior performance, a more refined characterization of its suboptimality gap -- beyond bounds -- is required. This motivates our study of its asymptotic distribution in Theorem \ref{thm:policyasympotic} below.

Nonetheless, the asymptotic analysis presents substantial theoretical challenges. Most classical asymptotic results rely on the uniqueness of the population-level optimizer (i.e., a unique $\theta^*$ that maximizes $\mathbb{E}^{\pi_{\theta}}(Z)$) and the strict negative definiteness of the Hessian matrix $H(\theta^*)$ at that optimum. In overparameterized models such as LLMs, both assumptions are inherently violated. This leads to two complications: (i) Parameter convergence is not well-defined in the classical sense: the estimators may oscillate within an optimal manifold of population-level maximizers and approach different maximizers across iterations. This lack of identifiability prevents the sequence of estimators from converging to a single fixed point. (ii) The parameter's asymptotic distribution becomes analytically intractable, since the estimator fails to converge to a fixed point in the first place.

To address the first challenge, we shift our focus from point-wise convergence to set-wise convergence. We define $\Theta^*$ as the set of all population-level maximizers,
\begin{equation*}
\Theta^* = \{ \theta \in \Theta : \theta \in \arg\max_{\theta^* \in \Theta} \mathbb{E}^{\pi_{\theta^*}}(Z) \},
\end{equation*}
and characterize parameter convergence by the distance of the estimator $\theta_n$ to the set $\Theta^*$, 
\begin{equation*}
d(\theta_n, \Theta^*) = \inf_{\theta^* \in \Theta^*} \|\theta_n - \theta^*\|.
\end{equation*}
We say the estimator $\theta_n$ is consistent if $d(\theta_n, \Theta^*) \stackrel{P}\to 0$ as $n \to \infty$. To address the second challenge, we recognize that while the parameter estimates are non-identifiable, the suboptimality gap remains identifiable. We thus analyze the asymptotic distribution of the suboptimality gap rather than the parameters themselves. 

Consider the asymptotic regime where the number of iterations $n \to \infty$ and other parameters such as the batch size $B$ and group size $G$ are constant. For this analysis, we adopt the $1/i$ schedule for the learning rate. This is motivated by Lemma \ref{lemma:policynonasympotic}: under the constant learning rate schedule, the sub-optimality gap might not converge to zero when $B$ and $G$ are fixed. While one could theoretically allow $B$ and $G$ to grow with $n$ or explore alternative decay schedules, these configurations introduce substantial technical complexity into the proofs without altering our major findings. We focus on this specific regime to more clearly elucidate the performance of GRPO, and impose the following conditions. 

\begin{assumption}[Compact support]\label{assump:boundedsupport}
    The parameter space $\Theta$ is compact. 
\end{assumption}
\begin{assumption}[Hessian matrices]\label{assump:Hessian}
    $\Theta^*$ is closed and convex, and all Hessian matrices at $\Theta^*$ share the same rank $r$. Moreover, there exists an orthogonal projection matrix $Q$ and a strictly negative definite matrix $H^*$ such that $Q^\top Q = I_r$ and  $Q^\top H(\theta^*)Q = H^*,\forall \theta^*\in \Theta^*$.
\end{assumption}

\begin{assumption}[Convergence of covariance matrix]\label{assum:convergence}
    There exists a (possibly random) matrix $\Gamma$ such that the following holds almost surely as $n\to \infty$,
    \begin{equation*}
        \mathbb{E}\left\{(\widehat{g}(\theta_n) - g(\theta_n))(\widehat{g}(\theta_n) - g(\theta_n))^\top\Big| \theta_n \right\} \to \Gamma.
    \end{equation*} 
\end{assumption}

\begin{assumption}[Weak strong concavity]\label{ass:WSC}
    Let $\Pi_{\Theta^*}$ denote the projection operator such that $\Pi_{\Theta^*}(\theta):= \arg\min_{x\in\Theta^*} d(\theta, x)$, we have
    \begin{equation*}
    \langle g(\theta), \Pi_{\Theta^*}(\theta)-\theta \rangle \geq \mu d^2(\theta, \Theta^*), \quad\forall\theta\in\Theta.
    \end{equation*}
\end{assumption}

Assumption \ref{assump:boundedsupport} is mild and commonly imposed in the literature \citep[e.g.,][]{schmidt2020nonparametric}. Assumption \ref{assump:Hessian} relaxes the standard identifiability condition by allowing the population parameter to form a connected manifold rather than a unique point. It assumes that, after projection onto a suitable $r$-dimensional subspace, the parameter becomes identifiable with a strictly negative definite Hessian $H^*$. This effectively decomposes the parameter space into an identifiable component defined through the transformation matrix $Q$ and a non-identifiable component along its orthogonal complement. Assumption \ref{assum:convergence} is not overly restrictive, as the limiting matrix $\Gamma$ is allowed to be random. The resulting asymptotic distribution can be interpreted as a mixture distribution: first, $\Gamma$ is realized, and then the limiting distribution is obtained conditionally on $\Gamma$. Finally, Assumption \ref{ass:WSC} imposes a weak strong concavity (WSC) condition. It is weaker than strong concavity (which entails a unique optimizer) but stronger than PL (Assumption \ref{assump:PL}). In contrast to strong concavity, WSC does not require the objective function to be globally concave in all directions, but only along directions that point toward the optimal set $\Theta^*$. Thus, the objective function $\mathbb{E}^{\pi_{\theta}} (Z)$ may remain flat along other directions, allowing multiple optimizers to exist. Under the $L$-smoothness condition (Assumption \ref{assump:lipcon}), it can be shown that WSC implies the PL condition.

\begin{theorem}[Consistency \& asymptotic distribution]\label{thm:policyasympotic}
Suppose Assumptions \ref{assump:boundedreward}, \ref{assump:lipcon}, \ref{assump:PL}, \ref{assump:LR}(b) and \ref{assump:boundedsupport} hold. Suppose each $C_i^{(b,g)}$ in Algorithm \ref{alg1} is either a function of $X^{(b)}$ or a GRPO-type baseline in \eqref{eqn:GRPObaseline}. Then the estimator $\theta_n$ is consistent in the sense that $d(\theta_n, \Theta^*)\stackrel{P}\to 0$. 

Suppose Assumptions \ref{assump:Hessian} -- \ref{ass:WSC} hold additionally. Then the output policy satisfies
\begin{equation*}
    n\Delta(\pi_{\theta_n}) \stackrel{d}\to \sum_{k=1}^r w_k \chi^2_{1,k}, 
\end{equation*}
where $r$ is the rank of the Hessian matrix, $\{\chi_{1,k}^2\}_{k=1}^r$ are i.i.d. $\chi_1^2$ random variables, and $\{w_k\}_{k=1}^r$ are positive weights arranged in non-increasing order, determined by the asymptotic covariance matrix of the gradient estimators (see step (iv) of the proof of Theorem \ref{thm:policyasympotic} in the Supplementary Material).
\end{theorem}

Theorem~\ref{thm:policyasympotic} is novel in two respects: (i) It establishes the asymptotic distribution of the suboptimality gap for policy gradient estimators as a weighted sum of independent $\chi^2$ random variables, whereas most existing literature focuses primarily on finite-sample error bounds. (ii) It derives this limiting distribution in the overparameterized regime, which goes beyond the classical assumption of a unique optimizer with a non-singular Hessian. 

Nonetheless, establishing such a result in the overparameterized regime is non-trivial. While classical results have established CLTs for parameter estimates in stochastic gradient algorithms \citep[e.g.,][]{zhang2016central}, these results do not directly apply due to overparameterization. Our key idea is to first establish a CLT for the subset of parameters corresponding to directions in which the Hessian is negative definite. A second-order Taylor expansion then shows that the suboptimality gap is driven entirely by the asymptotic behavior of these directions, which facilitates the derivation of its limit. 

As mentioned earlier, the weights $w_k$s are determined by the covariance matrices of the gradient estimators. The oracle and optimality properties of GRPO gradient estimators (Corollaries~\ref{coro:gradoracle} and \ref{coro:gradoptimal}) directly affect these weights, which in turn establishes the oracle and optimality properties of the GRPO policy. We summarize these results in Corollaries~\ref{coro:policyoracle} and \ref{coro:policyoptimal} below.

\begin{coro}[Oracle property of the policy]\label{coro:policyoracle}
    Suppose the assumptions in Theorem \ref{thm:policyasympotic} hold. Suppose the projected covariance matrix $\Omega(\theta^*)=\textrm{Cov}(Q^\top \widehat{g}(\theta^*))$ depends on $\theta^*\in \Theta^*$ only through $Q^\top \theta^*$ and is non-singular. Then GRPO's weights and oracle weights satisfy $w_{k,\rm GRPO}- w_{k,\rm oracle}=O(G^{-2})$. Consequently, the suboptimality gap of GRPO is asymptotically equivalent to that of the oracle algorithm as $G\to \infty$. 
\end{coro}

\begin{coro}[Optimality of the policy]\label{coro:policyoptimal}
    Suppose the assumptions in Corollary \ref{coro:policyoracle} hold. Suppose $Q^\top \nabla_{\theta} \log\pi_{\theta^*}(Y|X) [\nabla_{\theta} \log\pi_{\theta^*}(Y|X)]^\top Q$ is conditionally uncorrelated with $Z$ given $X$, for some $\theta^*\in \Theta^*$. Then for any meta-algorithm whose baseline term is a function of the prompt only, its weights satisfy $w_k\ge w_{k,\rm oracle}+O(G^{-2})$. Consequently, as  $G\to \infty$, its suboptimality gap is asymptotically no smaller than that of GRPO. 
\end{coro}

\section{Experiments}\label{sec: exp}
We conduct two sets of experiments in this section to validate our theoretical findings. In Section~\ref{sec:exp-variance}, we empirically compare GRPO with the vanilla and oracle algorithms in terms of gradient evaluation, to verify the oracle property of the GRPO gradient estimator (Corollary~\ref{coro:gradoracle}) and its superiority over the vanilla estimator  (Corollary~\ref{coro:gradoptimal}). In Section~\ref{sec:choice_of_G}, we investigate the optimal group size for policy optimization, to verify the universality of the optimal group size $G^*$ established by our scaling law (Theorem~\ref{thm:policynonasympotic}). 

\subsection{Oracle property in gradient evaluation}\label{sec:exp-variance}
\begin{figure}[t]
    \centering
    \begin{minipage}[t]{0.3\textwidth}
        \centering
        \includegraphics[width=\textwidth]{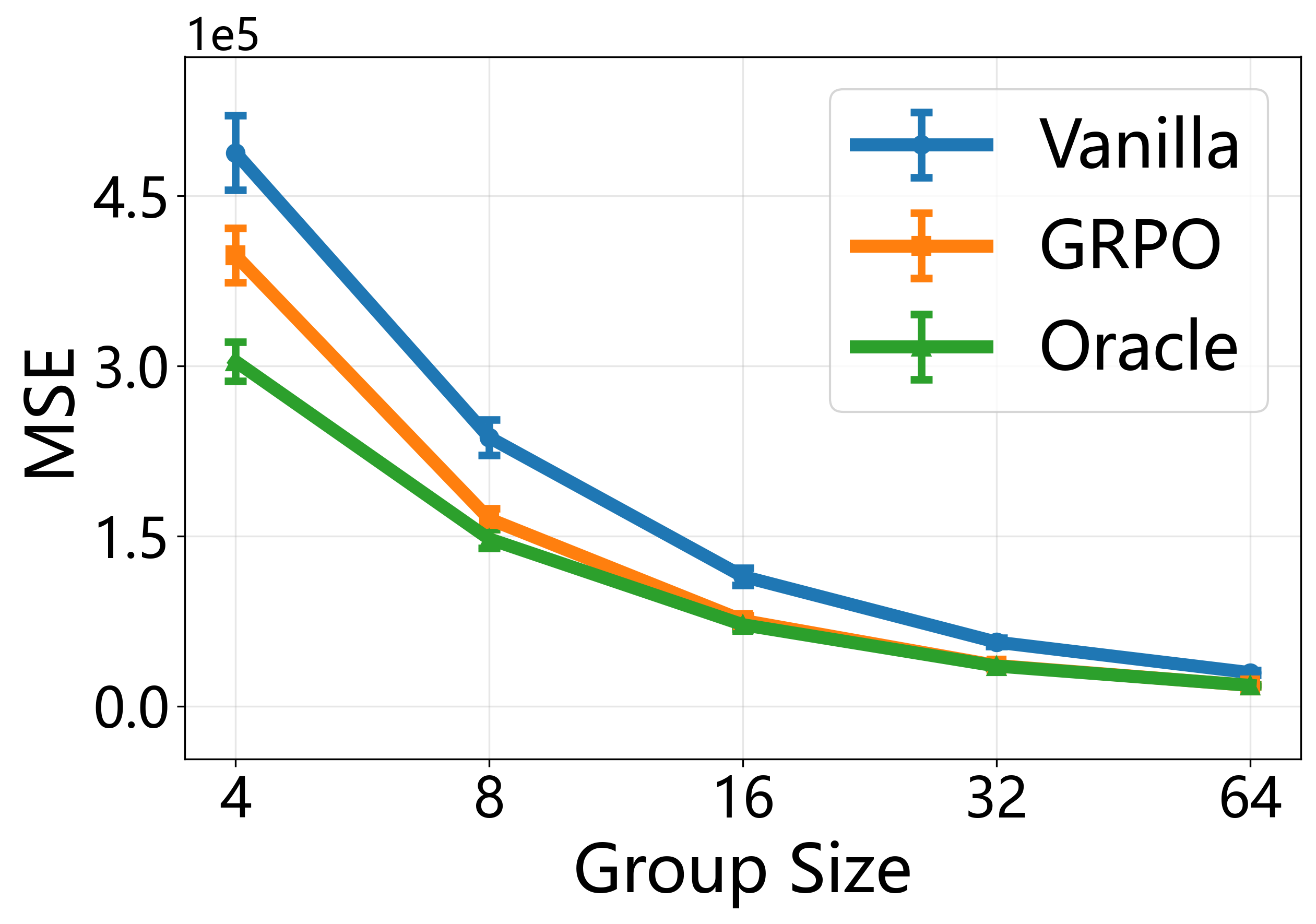}
        
        \centerline{\textit{(a) Base model}}
    \end{minipage}\hfill
    \begin{minipage}[t]{0.3\textwidth}
        \centering
        \includegraphics[width=\textwidth]{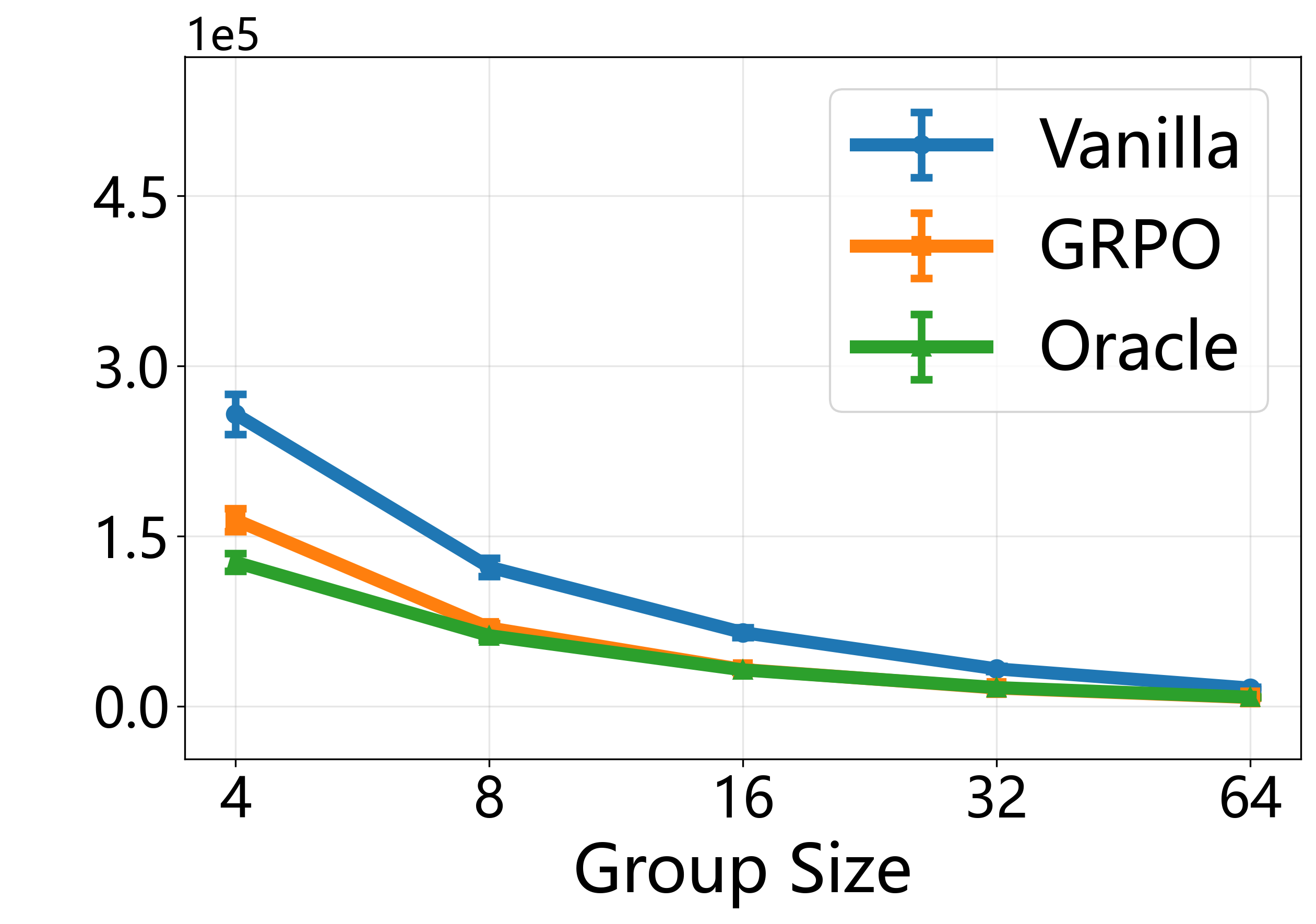}
        
        \centerline{\textit{(b) Instruct model}}
    \end{minipage}\hfill
    \begin{minipage}[t]{0.3\textwidth}
        \centering
        \includegraphics[width=\textwidth]{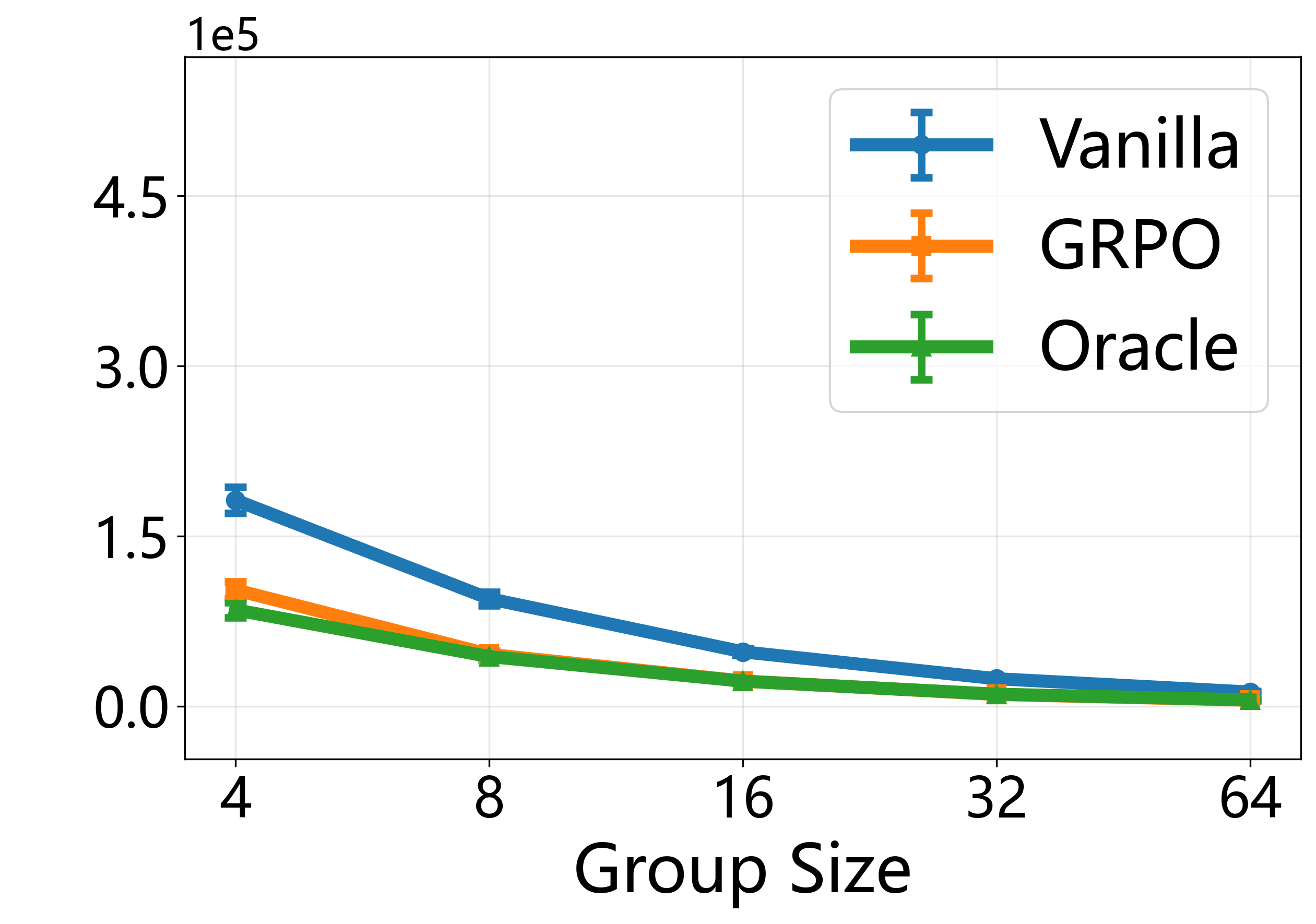}
        
        \centerline{\textit{(c) ICL model}}
    \end{minipage}
    \caption{MSEs of three policy gradient estimators (vanilla, GRPO-type, and oracle) under three model configurations (base, instruct and in-context learning (ICL)) for different group sizes. Error bars represent 95\% confidence intervals of the empirically estimated MSEs.}
    \label{fig:toy example}
\end{figure}

We first evaluate the MSE of the three gradient estimators introduced in Section~\ref{sec:grge}: vanilla, GRPO-type, and oracle. To conduct this comparison, we construct a synthetic arithmetic dataset consisting of 500 questions.  Each question is a medium-difficulty integer arithmetic problem, sampled uniformly from five categories: (i) two-step addition or subtraction, (ii) three-step addition or subtraction, (iii) single-step multiplication, (iv) integer division, and (v) addition or subtraction with parentheses. For each generated problem, we record the ground-truth solution. We then query an LLM (the prompt is provided in Section~\ref{app:prompt} of the Supplementary Material), 
extract the solution from its output, and compute a binary reward indicating whether the solution equals the ground truth. 

We evaluate gradients aggregated over the 500 questions (see Equation \eqref{eqn:gtheta}) with respect to three target policies, derived from three models with progressively stronger reasoning capabilities: (i) \texttt{Qwen/Qwen2.5-0.5B Base} model, (ii) \texttt{Qwen/Qwen2.5-0.5B Instruct} model, and (iii) \texttt{Qwen/Qwen2.5-0.5B Instruct} model with in-context learning \citep[ICL,][]{brown2020language}. For the third model, in addition to the prompt, we provide a few-shot in-context demonstration containing several arithmetic questions with solutions (see Section~\ref{app:prompt} of the Supplementary Material). These models produce target policies of different quality: the base model is expected to achieve the lowest accuracy, while the ICL model attains the highest. We also consider multiple groups sizes $G\in\{4,8,16,32,64\}$ to estimate the gradient and use Monte Carlo simulations to evaluate each gradient estimator's MSE and report their associated $95\%$ confidence interval in Figure \ref{fig:toy example}. 

We make the following observations:
\begin{enumerate}[leftmargin=*]
    \item Across all combinations of group size and target policy, the {\bf \color{blue}{vanilla}} estimator (blue line) exhibits the largest MSE, reflecting the well-known high-variance nature of REINFORCE. In contrast, the {\bf \color{orange}{GRPO-type}} estimator (orange line) achieves significantly smaller MSEs in all cases, which verifies the second assertion of Corollary \ref{coro:gradoptimal} and demonstrates its superiority over the vanilla estimator. 
    \item The {\bf \color{OliveGreen}{oracle}} estimator (green line) achieves the smallest MSE, as expected. Nevertheless, with a moderately large group size ($G=8$), the MSE of the GRPO-type estimator is already close to that of the oracle estimator. As $G$ increases further (e.g., $G=32$ or $64$), the two estimators become nearly indistinguishable, confirming the oracle property in Corollary~\ref{coro:gradoracle}. 
    \item Finally, the MSE of all estimators decreases as the group size $G$ increases. It also decreases with the model’s reasoning capability. This behavior is expected as well: as the model becomes stronger, its outputs are more likely to be correct and thus more deterministic. Conversely, weaker models generate more random outputs, which enlarges the variance of the gradient estimator.
\end{enumerate}

\subsection{Optimal group size for policy optimization}\label{sec:choice_of_G}
As shown in our theoretical analysis and empirical results (Section~\ref{sec:choice_of_G}), increasing the group size $G$ reduces the MSE of the GRPO-type gradient estimator. However, a larger $G$ also increases computational cost, as more outputs must be sampled per prompt. Under a fixed sampling budget (i.e., a fixed total number of sampled outputs), our scaling law in Theorem~\ref{thm:policynonasympotic} shows that the optimal group size $G^*$ depends only on the data and the policy model, and is universal with respect to other parameters such as the number of iterations $n$ and the total budget per prompt $N$. In this section, we verify this universality using two widely adopted math reasoning benchmark datasets, \textit{GSM8K} \citep{cobbe2021training} and \textit{MATH} \citep{hendrycks2measuring}. 

We first utilize GSM8K to examine the universality of $G^*$ across different training iterations $n$. Specifically, we fix the total sampling budget per prompt at $N = BG = 1024$ and evaluate six candidate group sizes: $G \in \{4, 8, 16, 32, 64, 128\}$. For each choice of $G$, we apply GRPO to fine-tune the \texttt{Qwen2.5-1.5B Instruct} model and calculate its test accuracy, defined as the percentage of correctly solved problems in the test dataset. In addition to reporting the final accuracy at the last iteration, we also record accuracy at intermediate checkpoints where $n \in \{200, 300, 400, 600, 800\}$. Given the inherent stochasticity of RLVR -- arising, for example, from randomness in sampled outputs -- which often leads to high variance across runs, we repeat the training five times for each value of $G$. We report the mean accuracy along with its associated confidence interval in Figure~\ref{fig:gsms8k_choice_g_steps_1024}. We note that much of the existing literature reports results from a single training run due to the substantial computational cost of training, which can lead to results that are difficult to reproduce.

\begin{figure}[t]
    \centering
    \begin{minipage}[t]{0.30\textwidth}
        \centering
        \includegraphics[width=\textwidth]{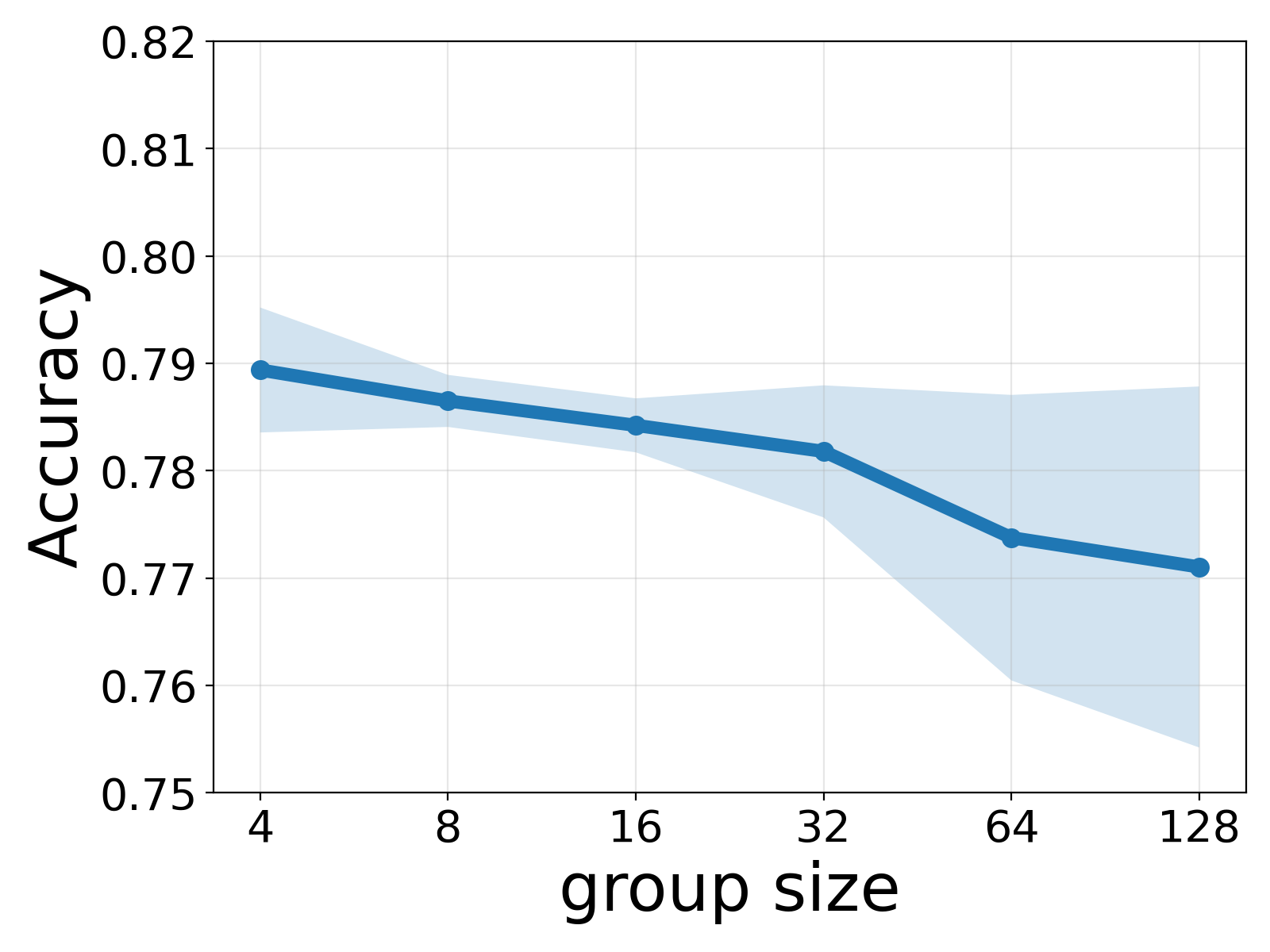}
        \centerline{\scalebox{0.9}{\textit{Step 200}}}
    \end{minipage}\hspace{0.6em}%
    \begin{minipage}[t]{0.30\textwidth}
        \centering
        \includegraphics[width=\textwidth]{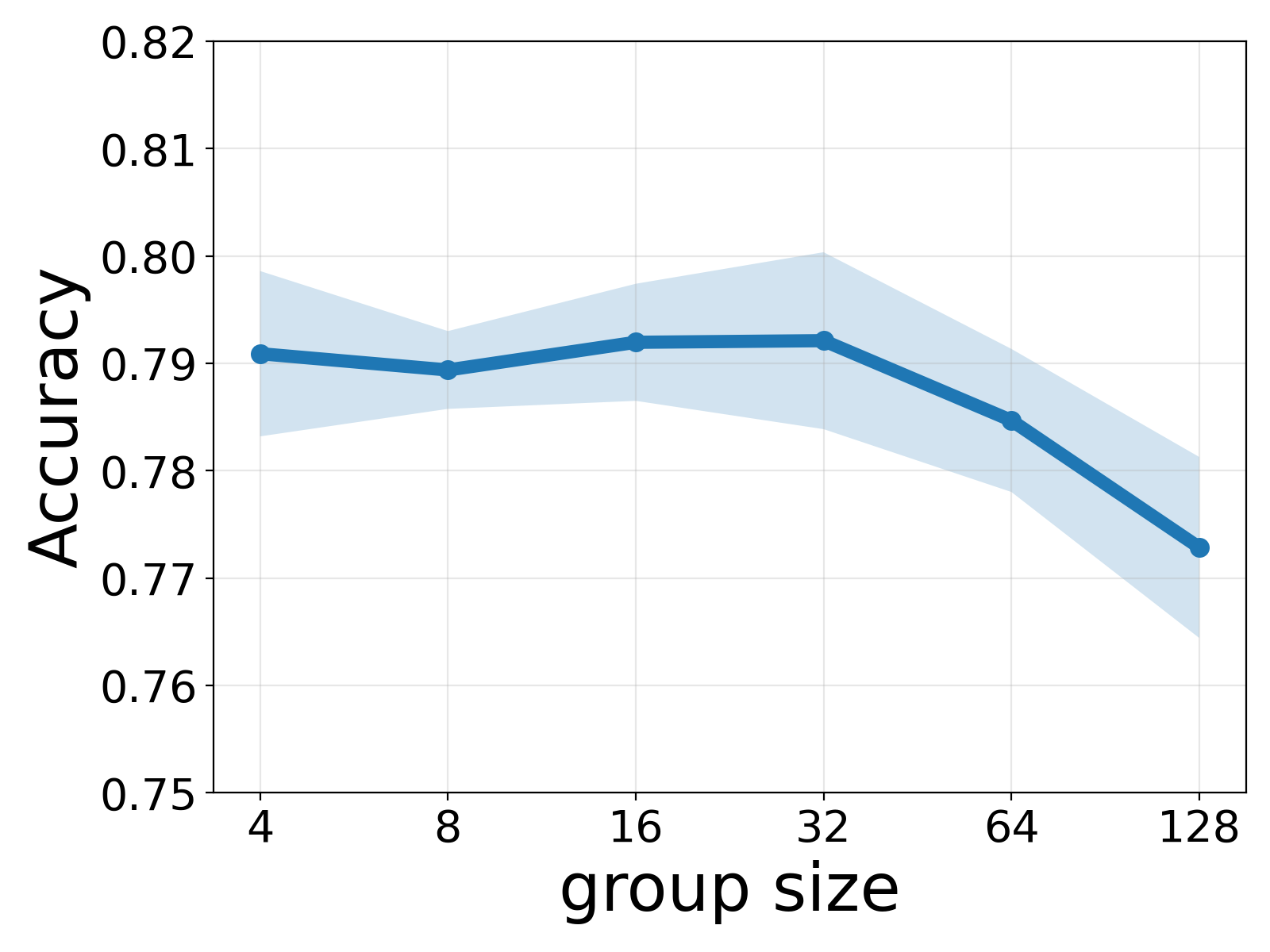}
        \centerline{\scalebox{0.9}{\textit{Step 300}}}
    \end{minipage}\hspace{0.6em}%
    \begin{minipage}[t]{0.30\textwidth}
        \centering
        \includegraphics[width=\textwidth]{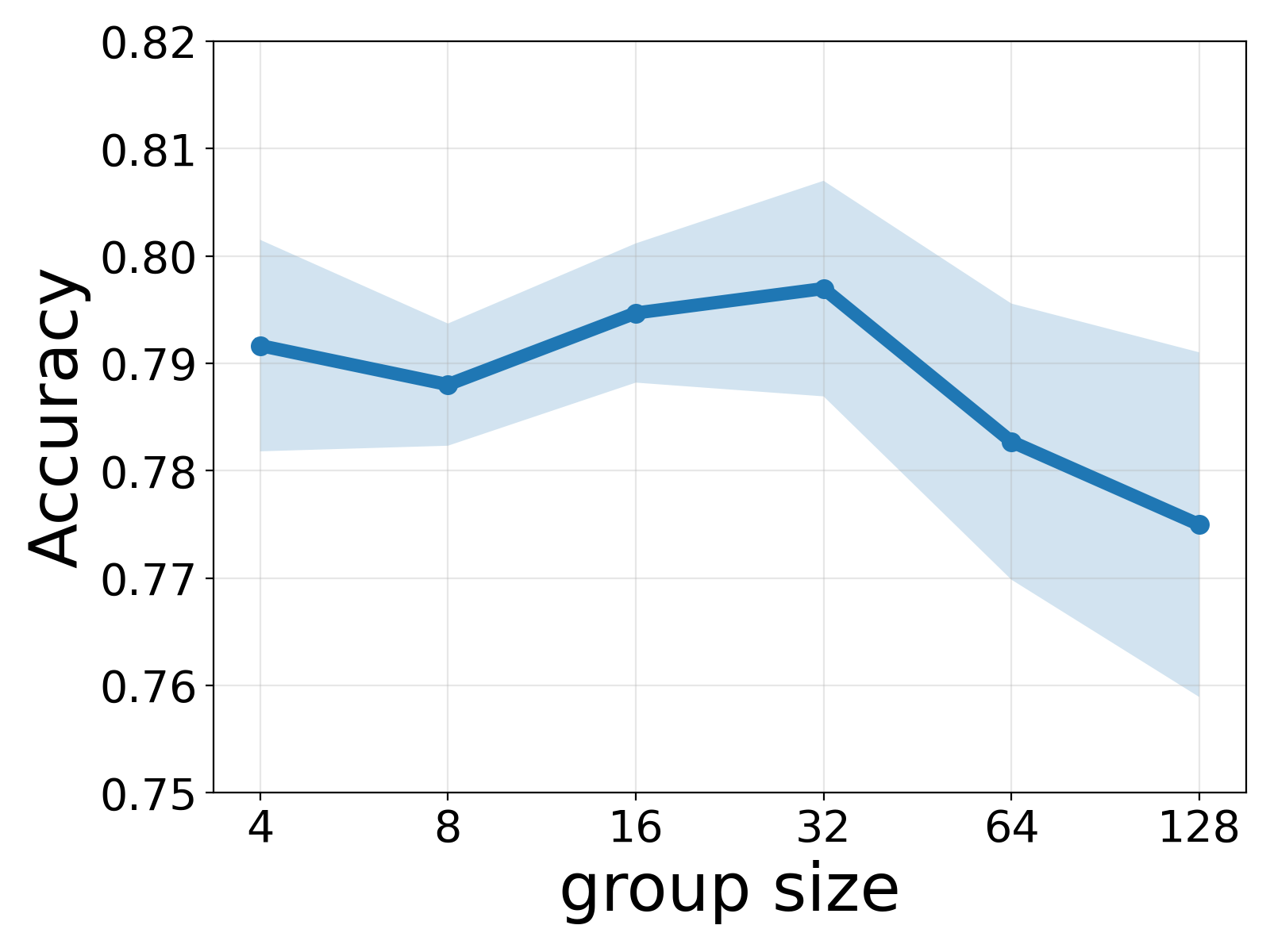}
        \centerline{\scalebox{0.9}{\textit{Step 400}}}
    \end{minipage}

    \vspace{0.6em}

    \begin{minipage}[t]{0.30\textwidth}
        \centering
        \includegraphics[width=\textwidth]{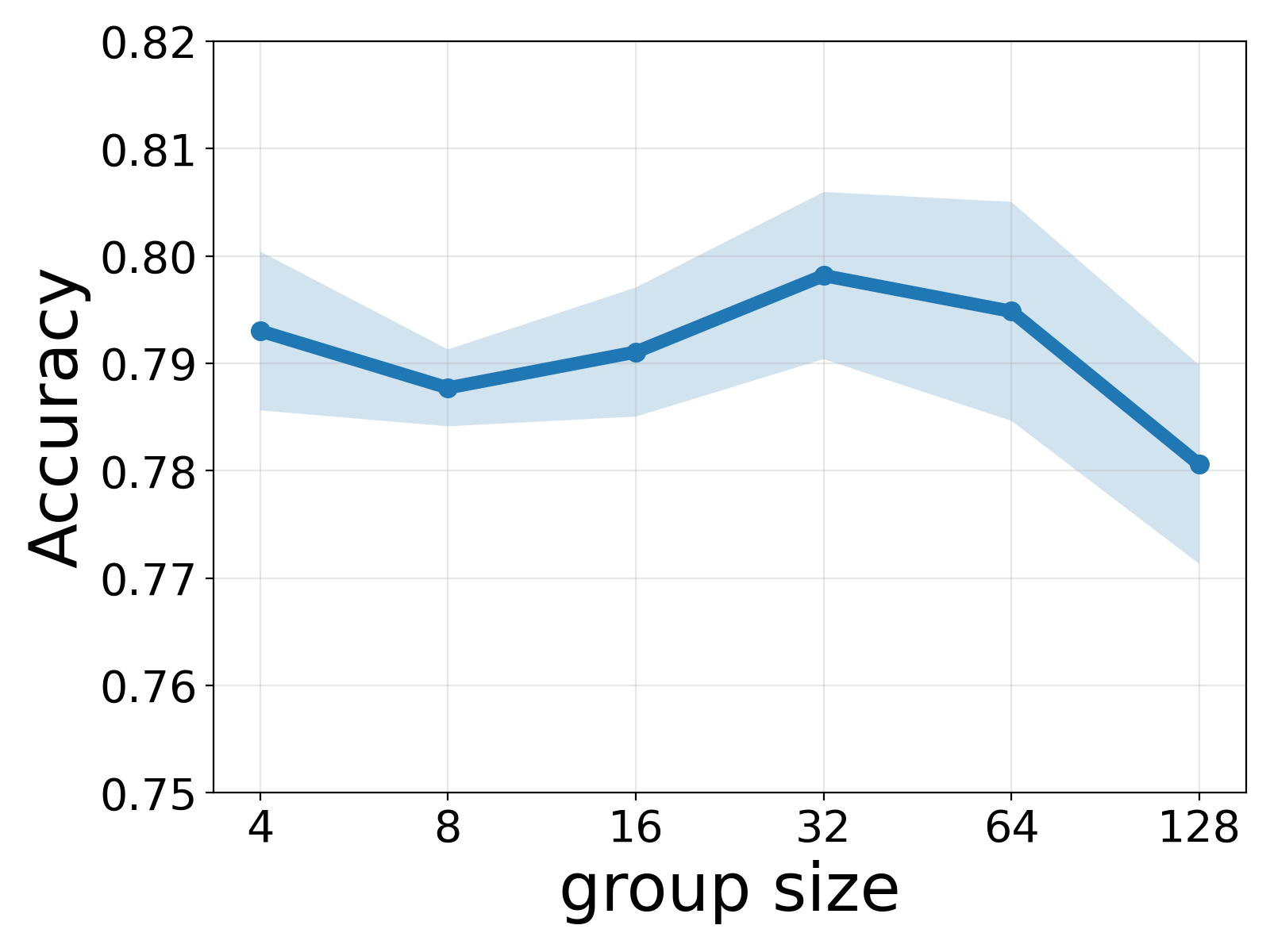}
        \centerline{\scalebox{0.9}{\textit{Step 600}}}
    \end{minipage}\hspace{0.6em}%
    \begin{minipage}[t]{0.30\textwidth}
        \centering
        \includegraphics[width=\textwidth]{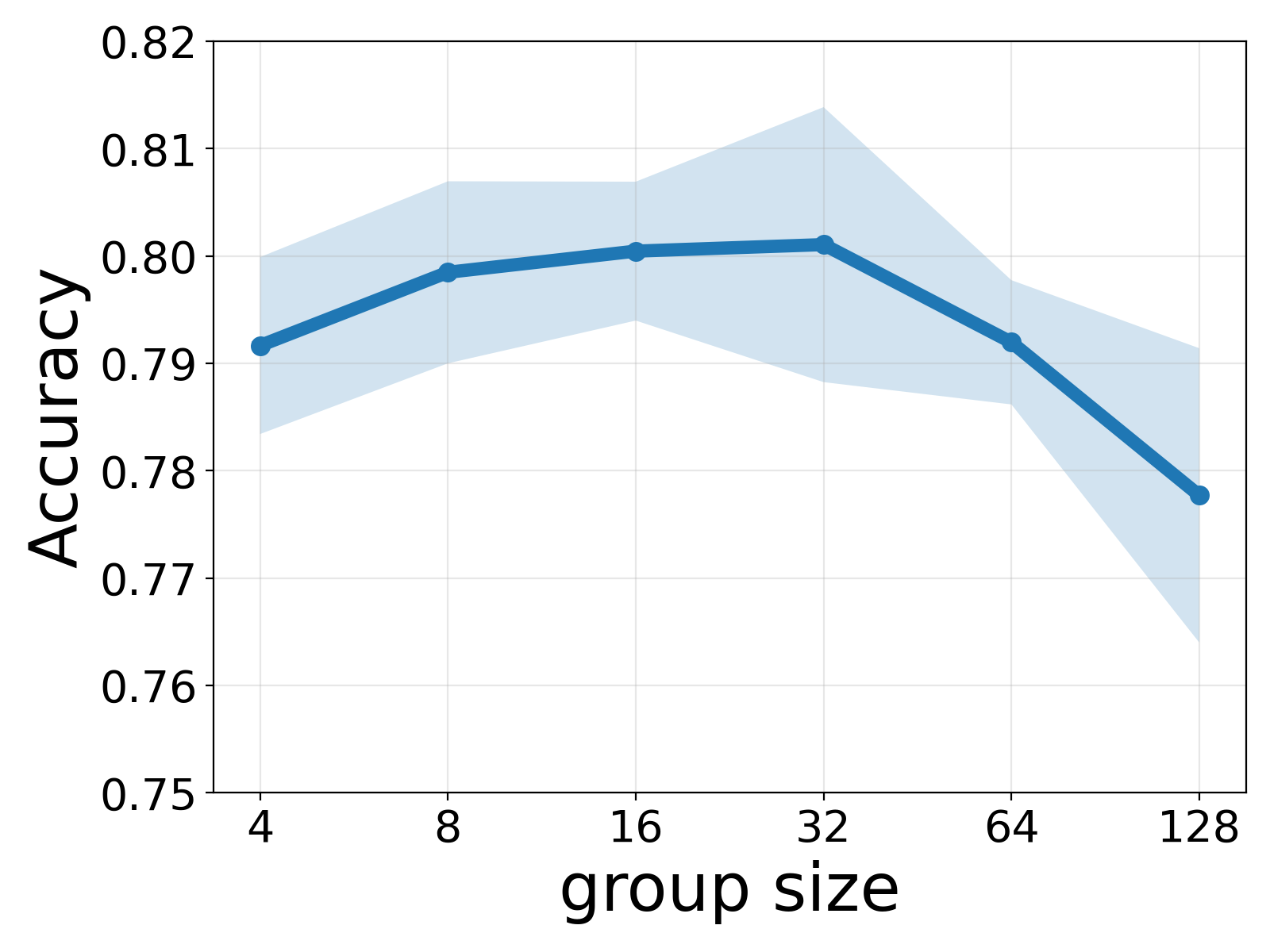}
        \centerline{\scalebox{0.9}{\textit{Step 800}}}
    \end{minipage}\hspace{0.6em}%
    \begin{minipage}[t]{0.30\textwidth}
        \centering
        \includegraphics[width=\textwidth]{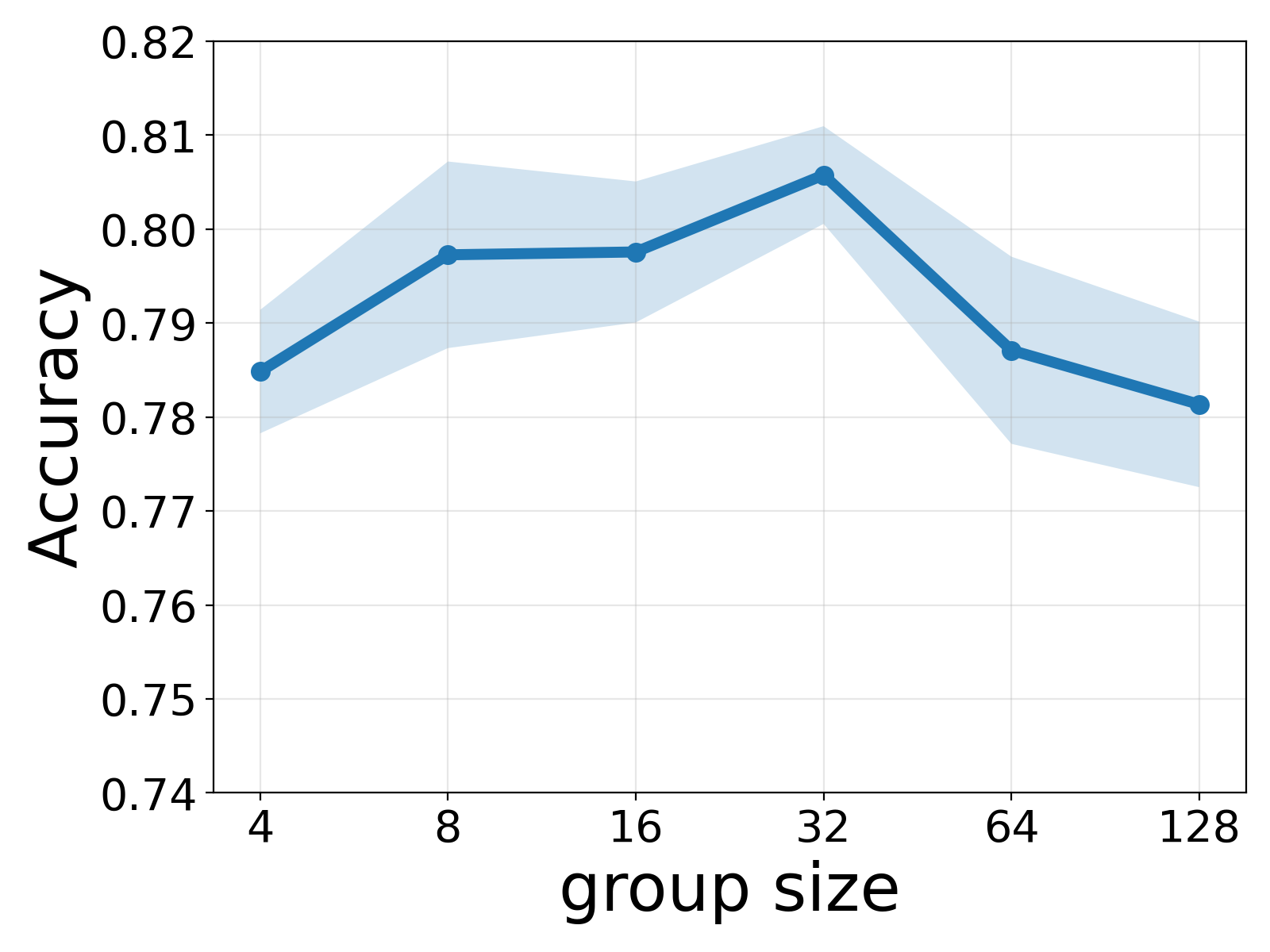}
        \centerline{\scalebox{0.9}{\textit{Final Step}}}
    \end{minipage}

    \caption{Test accuracy of GRPO-fine-tuned models at different training steps with a fixed sampling budget of $N=B \times G = 1024$ per prompt. Both training and evaluation are conducted on GSM8K. Each curve shows accuracy as a function of the group size $G \in \{4, 8, 16, 32, 64, 128\}$. Results are averaged over five independent runs, with shaded regions visualizing 95\% confidence bands.}
    \label{fig:gsms8k_choice_g_steps_1024}
\end{figure}

We make the following observations from Figure \ref{fig:gsms8k_choice_g_steps_1024}:
\begin{enumerate}[leftmargin=*]
    \item Except at training step $n = 200$, the test accuracy generally increases with the group size $G$ and then decreases as $G$ becomes larger. This trend is consistent with the scaling law in Theorem~\ref{thm:policynonasympotic}. When $G$ is small, the second-order residual term can be large, inflating the variance of the gradient estimator. Conversely, when $G$ is large, fixing the total sampling budget forces the batch size $B$ to be small, which enlarges the first variance term in \eqref{eqn:quantity}. As a result, the optimal group size $G^*$ lies between these two extremes.
    \item With the exception of $n = 200$, the optimal group size is consistently $G^* = 32$ across all training steps, demonstrating the universality of $G^*$ with respect to $n$. For $n = 300$ and $n = 800$, the performance with $G = 16$ is very close to that with $G = 32$. However, its accuracy deteriorates more noticeably for other values of $n$.
    \item Finally, the differences in test accuracy across different values of $G$ are mostly statistically insignificant, as only five independent runs are conducted for each setting. This limitation is due to the high computational cost of training. While increasing the number of runs could yield more statistically meaningful conclusions, doing so is not computationally feasible at this stage.
\end{enumerate}

We next use the MATH dataset to assess the universality of the optimal group size $G^*$ across different sampling budgets $N$. The procedure closely follows that used for GSM8K: we consider the same candidate group sizes $G \in \{4, 8, 16, 32, 64, 128\}$ and train the model separately for each choice of $G$. The differences are that, for MATH, we fine-tune a larger and more powerful \texttt{Qwen2.5-Math-7B} model, and we evaluate three different sampling budgets, $N \in \{1024, 2048, 4096\}$. Test accuracies of the resulting trained models are reported in Table~\ref{tab:batch_group_transposed_4dp}. 

It can be seen that the optimal group size is mostly 64, but increases to 128 as the sampling budget grows. We suspect this shift is due to the finite number of prompts, which remains constant with respect to the batch size $B$ and group size $G$, and is not accounted for in our theoretical analysis for simplicity. Meanwhile, when the sampling budget is fixed at 1024, the optimal $G^*$ is larger than that for GSM8K. This shift is expected, as the constants in \eqref{eqn:quantity} depend on the data and the model, and thus the optimal $G^*$ varies accordingly. These results suggest that larger models may benefit from a larger group size during training.

\begin{table}[t]
\centering
\caption{Test accuracy of GRPO-fine-tuned models at the final training step. Each row reports accuracy as a function of the group size $G \in {4, 8, 16, 32, 64, 128}$ with a fixed sampling budget per prompt; the sampling budget varies across rows. Both training and evaluation are conducted on MATH. Due to the high computational cost of training a 7B model, results are reported from a single run, with the highest accuracy highlighted in bold.}
\begin{adjustbox}{width=\textwidth}
\setlength{\tabcolsep}{15pt}
\renewcommand{\arraystretch}{1.15}
\begin{tabular}{lcccccc}
\toprule
\multirow{2}{*}{Sampling budget} & \multicolumn{6}{c}{Group size $G$} \\
 \cmidrule(lr){2-7}
 & 4 & 8 & 16 & 32 & 64 & 128 \\
\midrule
1024 & 0.7677 & 0.7491 & 0.7753 & 0.7627 & \textbf{0.7817} & 0.7743 \\
2048 & 0.7703 & 0.7679 & 0.7697 & 0.7793 & \textbf{0.7819} & 0.7665 \\
4096 & 0.7691 & 0.7713 & 0.7701 & 0.7673 & 0.7703 & \textbf{0.7757} \\
\bottomrule
\end{tabular}
\label{tab:batch_group_transposed_4dp}
\end{adjustbox}
\end{table}

\section{Conclusion}
This paper provides a rigorous theoretical analysis of GRPO, a cornerstone algorithm for enhancing the reasoning capabilities of LLMs. We show that GRPO is a statistically principled policy gradient algorithm whose gradient estimator naturally forms a U-statistic (Lemma \ref{lem:gradientUstat}). Leveraging Hoeffding's decomposition, we characterize its MSE (Theorem \ref{thm:gradMSE}, Proposition \ref{prop:MSE}) and establish both oracle (Corollary \ref{coro:gradoracle}) and optimality (Corollary \ref{coro:gradoptimal}) properties. For policy optimization, we derive explicit bounds on GRPO’s suboptimality gap (Lemma \ref{lemma:policynonasympotic}), leading to a scaling law that informs the optimal choice of group size (Theorem \ref{thm:policynonasympotic}). We characterize the asymptotic distribution of the suboptimality gap (Theorem \ref{thm:policyasympotic}), confirming GRPO’s oracle (Corollary \ref{coro:policyoracle}) and optimality (Corollary \ref{coro:policyoptimal}) properties in policy learning. We further extend these results to practical settings by accommodating reward normalization, importance sampling, and KL-divergence penalties (Lemma \ref{lem:pgobjective}, Theorem \ref{thm:practicalgradMSE}, Proposition \ref{prop:target}). Finally, we empirically validate GRPO’s gradient estimator's oracle and optimality properties (Figure \ref{fig:toy example}), and verify our scaling law with respect to group size $G$, demonstrating its universality across training iterations (Figure \ref{fig:gsms8k_choice_g_steps_1024}) and sampling budgets (Table \ref{tab:batch_group_transposed_4dp}). 

\bibliographystyle{plainnat}
\bibliography{reference}

\begin{appendices}

\section{Practical considerations}\label{sec:ext}
As commented in the main paper, the GRPO-type algorithm analyzed in Section \ref{sec:theory} differs from the implementation proposed in the original DeepSeekMath paper. To bridge this gap, we present a more consistent formulation in Algorithm \ref{alg2}. 
There are three major differences compared to Algorithm \ref{alg1}: (i) reward normalization, (ii) importance sampling (IS), and (iii) KL penalty. Specifically:
\begin{enumerate}[leftmargin=*]
    \item[(i)] \textbf{Reward normalization.} According to \eqref{eqn:advantage}, the advantage function $A^{(b,g)}$ is not simply the difference between the reward $Z^{(b,g)}$ and the baseline $\bar{Z}^{(b,-g)}$. Instead, this difference is normalized by the standard error of the within-group rewards (see \eqref{eqn:se}). Consequently, the algorithm assigns larger weight to prompts whose rewards are more uncertain. Intuitively, when a prompt is either too easy or too difficult, most responses receive a score of zero or achieve the maximum score, and the normalized advantage is small. In contrast, when the level of difficulty lies in the middle — so that some responses are correct while others are wrong — the prompt receives large weight.

    \item[(ii)] \textbf{Importance sampling.} The practical implementation includes a minibatch parameter. Whenever the minibatch size is smaller than the full batch size $B$, the sampled prompts are divided into $m$ minibatches, and the parameter is updated multiple times within each iteration (see Lines 8--10 of Algorithm \ref{alg2}). This is intended to improve learning efficiency: smaller minibatches lead to more frequent parameter updates, although each update is based on less data and is therefore noisier. Since the outputs are generated under the initial parameter $\theta_{\mathrm{old}}$ (see Line~4), whereas the gradient is evaluated at the current parameter $\theta$, a distributional shift arises. To correct for this mismatch,  
    the advantage function is multiplied by an IS ratio $\mu^{(b,g)}_t=\pi_{\theta,t}^{(b,g)}/\pi_{\theta_{\mathrm{old}},t}^{(b,g)}$ in \eqref{eqn:gradient}. Ideally, one would instead use the sequential IS ratio up to time $t$, $\Pi_{k=1}^t \mu^{(b,g)}_k$, 
    which yields an unbiased correction. However, it is well known that such sequential IS ratio suffers from the curse of horizon \citep{liu2018breaking}, in the sense that their variance grows exponentially fast with $t$. As such, the practical algorithm uses only the token-level ratio, which substantially reduces variance at the cost of introducing bias, as we show formally in Theorem~\ref{thm:practicalgradMSE}. The use of the token-level ratio is also in line with the PPO algorithm. %Meanwhile, the use of minibatches ; by contrast, using the full batch produces a more accurate gradient estimate, but only one update is performed per iteration, which may slow down learning.

    \item[(iii)] \textbf{KL penalty.} The second and third terms inside the square brackets in \eqref{eqn:gradient} arise from imposing a KL penalty 
    \begin{align*}
    \kappa \textrm{KL}(\pi_{\theta}\|\pi_{\theta_{\text{ref}}}) =\kappa\,\mathbb{E}_{X\sim f, Y\sim \pi_{\theta}(\bullet|X)}\log \frac{\pi_{\theta}(Y|X)}{\pi_{\theta_{\textrm{ref}}}(Y|X)}
    \end{align*}
    on the objective function, for some regularization parameter $\kappa>0$. The purpose is to discourage the learning policy to departure too far away from a reference policy, an idea that shares similar spirits with the pessimistic principle in offline RL.  
    In particular, GRPO uses Schulman's K3 estimator\footnote{\url{http://joschu.net/blog/kl-approx.html}} 
    \begin{eqnarray}\label{eqn:K3}
        \frac{1}{G}\sum_{g=1}^G\Big[\frac{\pi_{\theta_{\text{ref}}}(Y^{(b,g)}|X)}{\pi_{\theta}(Y^{(b,g)}|X)}-\log \frac{\pi_{\theta_{\text{ref}}}(Y^{(b,g)}|X)}{\pi_{\theta}(Y^{(b,g)}|X)}-1\Big]
    \end{eqnarray}
    for the KL divergence, which leads to the gradient estimator $G^{-1}\sum_t \nabla_\theta \log \pi_{\theta,t}^{(b,g)} (1-\pi_{\theta,t}^{(b,g)}/\pi_{\theta_{\text{ref}},t}^{(b,g)})$ in \eqref{eqn:gradient}. When $Y^{(g)}$ is sampled from $\pi_{\theta}(\bullet|X)$, it is immediate that the first term inside the square brackets of \eqref{eqn:K3} has mean one, so the expectation of \eqref{eqn:K3} reduces to the second term alone and is thus unbiased for the KL divergence. The advantage of using the K3 estimator over solely the second term (known as the K1 estimator) is that K3 achieves lower variance and is guaranteed to be non-negative, both of which help stabilize training \citep{shao2024deepseekmath}. 
\end{enumerate}

\begin{algorithm}[t]
\caption{Group relative policy optimization}\label{alg2}
\begin{algorithmic}[1]
    \STATE \textbf{Input:} prompt distribution $f(X)$, initial parameter $\theta \in \Theta$, learning rates $\{\eta_i\}_{i\in\mathbb{N}}$, batch size $B$, per-prompt group size $G$, number of minibatches $m$, KL regularization parameter $\kappa>0$ and reference model $\theta_{\textrm{ref}}$. 
    \FOR{$i=0,1,2,\dots,n-1$}
        \STATE Set $\theta_{\text{old}}\gets \theta$.
        
        \STATE Sample prompts $\{X^{(b)}\}_{b=1}^B \stackrel{\mathrm{iid}}{\sim} f(\cdot)$.
        
        \STATE For each prompt $X^{(b)}$, generate a group of $G$ outputs
        \[
        \{Y^{(b,g)}\}_{g=1}^G \sim \pi_{\theta_{\mathrm{old}}}(\cdot \mid X^{(b)}).
        \]
        \STATE For each output $Y^{(b,g)}$, obtain reward $Z^{(b,g)}$.
        \STATE Partition $\{1,\dots,B\}$ into $m$ disjoint minibatches $\mathcal{B}_1,\dots,\mathcal{B}_m$ of equal size.
        \FOR{$j=1,\dots,m$}
            \STATE Compute the minibatch gradient estimator $\widehat{g}(\theta)=|\mathcal{B}_j|^{-1}\sum_{b\in \mathcal{B}_j}\widehat{g}(X^{(b)};\theta)$ where
            \begin{align}\label{eqn:gradient}
            \widehat g(X^{(b)};\theta)
            :=
            \frac{1}{G}
            \sum_{g=1}^G\sum_{t}
            \nabla_\theta \log \pi_{\theta,t}^{(b,g)}
            \Bigg[
            \frac{\pi_{\theta,t}^{(b,g)}}{\pi_{\theta_{\mathrm{old}},t}^{(b,g)}}
            A^{(b,g)}
            +\kappa\frac{\pi_{\theta_{\mathrm{ref}},t}^{(b,g)}}{\pi_{\theta,t}^{(b,g)}}
            -\kappa
            \Bigg],
            \end{align}
            where
            \[
            \pi_{\theta,t}^{(b,g)}
            :=
            \pi_\theta\!\left(
            Y_t^{(b,g)} \mid X^{(b)}, Y_{<t}^{(b,g)}
            \right),
            \]
            and
            \begin{eqnarray}\label{eqn:advantage}
                A^{(b,g)}:=
            \frac{Z^{(b,g)}-\bar Z^{(b,-g)}}{\operatorname{se}(Z^{(b,\bullet)})},
            \end{eqnarray}
            where $Z^{(b,\bullet)}=\{Z^{(b,g)}\}_{g=1}^G$ and
            \begin{eqnarray}\label{eqn:se}
                \operatorname{se}(Z^{(b,\bullet)})=\sqrt{\frac{1}{G-1}\sum_g [Z^{(b,g)}-\bar{Z}^{(b)}]^2},
            \end{eqnarray}
            denotes its standard error and $\bar{Z}^{(b)}$ denotes the empirical group mean $G^{-1}\sum_{g=1}^G Z^{(b,g)}$.  
            \STATE Update the parameter:
            \[
            \theta \gets \theta + \eta_{i+1}\,\widehat g(\theta).
            \]
        \ENDFOR
    \ENDFOR
    \STATE \textbf{Output:} policy $\pi_\theta$.
\end{algorithmic}
\end{algorithm}
We begin by analyzing the gradient estimator $\widehat{g}(x;\theta)$ in \eqref{eqn:gradient}. The following lemma shows that, due to reward normalization, this gradient estimator is not exactly a U-statistic -- as the kernel function depends on $\operatorname{se}(Z^{(b,\bullet)})$, which is data-dependent -- but admits a U-statistic representation asymptotically as $G\to\infty$.
\begin{assumption}[Coverage]\label{assump:coverage}
    $\pi_{\theta}$ is lower bounded by $\epsilon$ for some $\epsilon>0$ and any $\theta\in \Theta$.
\end{assumption}
Conditions similar to \ref{assump:coverage} are frequently imposed in offline RL \citep[e.g.,][]{uehara2022review}. 
\begin{lemma}[Gradient estimator asymptotically equivalent to a U-statistic]\label{lem:pgobjective}
    Suppose Assumptions \ref{assump:boundedreward}, \ref{assump:lipcon} and \ref{assump:coverage} hold. Suppose $\text{std}_X(Z):=\sqrt{\text{Var}_{Z\sim \pi_{\theta_{\text{old}}(\bullet|X)}}(Z|X)}$ is bounded away from zero. Then $\widehat{g}(x;\theta) $ converges in probability to a second-order U-statistic
    \begin{equation*}
        \widehat{g}(x;\theta)
        =\binom{G}{2}^{-1}\sum_{1\le i<j\le G} h\big((Y^{(i)},Z^{(i)}),(Y^{(j)},Z^{(j)})\big),
    \end{equation*}
    as $G\to \infty$, with symmetric kernel
    \begin{align*}
    &h\big((Y^{(i)}, Z^{(i)}), (Y^{(j)}, Z^{(j)})\big)\\
    =& \frac{1}{2}
    \Big[
    \sum_t \mu_t^{(i)} W_t^{(i)}
    -
    \sum_t \mu_t^{(j)} W_t^{(j)}
    \Big] 
    \frac{Z^{(i)}-Z^{(j)}}{\mathrm{std}_x(Z)}
    +\frac{\kappa}{2}
    \Big[
    \sum_t (\mu_t^{(i)}-1) W_t^{(i)}
    +
    \sum_t (\mu_t^{(j)}-1) W_t^{(j)}
    \Big],
    \end{align*}
    where $W^{(g)}_t=\nabla_{\theta}\log \pi_{\theta}(Y^{(g)}_t|x,Y^{(g)}_{<t})$ and the superscript $(b)$ on $Y$ and $Z$ is removed for simplicity. 
\end{lemma}

We next analyze the MSE of $\widehat{g}(x,\theta)$, defined as $\text{MSE}(\widehat{g}(x,\theta)) = \mathbb{E}[\Vert \widehat{g}(x;\theta) - g^\dagger(x;\theta)\Vert^2]$ where
\begin{eqnarray*}
    g^\dagger(x;\theta)=\mathbb{E}\Big[\mu^{(g)}\nabla_\theta\log\pi_\theta(Y^{(g)}|x)\frac{Z^{(g)}- \bar{Z}^{(-g)}}{\text{std}_x(Z)}+ \kappa \mu^{(g)} \sum_t (\mu_t^{(g)}-1) W_t^{(g)} \mid X=x\Big],
\end{eqnarray*} 
and $\mu^{(g)}$ denotes the sequential IS (SIS) ratio $\prod_t \mu_t^{(g)}$. This ratio appears in the ground-truth gradient to account for the distributional shift: the gradient is evaluated under the current policy $\pi_{\theta}$, whereas the outcomes were generated under $\pi_{\theta_{\text{old}}}$. As discussed earlier, the practical implementation does not use the SIS ratio for gradient estimation due to its extraordinarily large variance; instead, only the token-level IS ratio $\mu_t^{(b,g)}$ is used to multiply the score and the advantage function. Meanwhile, the KL divergence term does not include any IS ratio and is therefore biased as well. Thus, bias arises from both sources. In Theorem \ref{thm:practicalgradMSE} below, we characterize the MSE of the gradient estimator, upper bounding both its squared bias and variance.

\begin{theorem}[MSE upper bounds]\label{thm:practicalgradMSE}
    Under Assumptions \ref{assump:boundedreward}, \ref{assump:lipcon}, \ref{assump:coverage} and the same conditions in lemma \ref{lem:pgobjective}, we have
    \begin{eqnarray}\label{eqn:MSEbiasvriance}
        \text{MSE}(\widehat{g}(x;\theta))= O\Big(\frac{m^2\eta_{i+1}^2}{\epsilon}\Big) + O\Big(\frac{\kappa^2 m^2\eta_{i+1}^2}{\epsilon^3}\Big)+ \frac{\text{trace}[\Sigma_{1}(x;\theta)]}{G}+O\Big(\frac{1 }{\epsilon G^2}\Big),
    \end{eqnarray}
    where 
    \begin{eqnarray}\label{eqn:Sigma1}
        \Sigma_1(x;\theta)= \text{Cov}\Big[\sum_t \mu_t^{(g)} W_t^{(g)}\frac{Z^{(g)}-V^{\pi_{\theta_{\text{old}}}}(x)}{\text{std}_x(Z)}+\kappa \sum_t (\mu_t^{(g)}-1) W_t^{(g)}|X=x\Big].
    \end{eqnarray}
\end{theorem}
Theorem~\ref{thm:practicalgradMSE} gives a bias-variance decomposition regarding the gradient estimator's MSE:
\begin{enumerate}[leftmargin=*]
    \item The first two terms on the right hand side (RHS) of \eqref{eqn:MSEbiasvriance} are bias terms that upper bound the squared biases arising from the two sources mentioned above, respectively. Notably, these bias terms are proportional to the learning rate $\eta_{i+1}$ and vanish as $\eta_{i+1}$ approaches zero.
    \item The last two terms are variance terms. In particular, $\mathrm{trace}[\Sigma_1(x;\theta)]/G$ is the leading term, which, by its definition in \eqref{eqn:Sigma1}, corresponds to the variance of an oracle estimator that knows the value function $V^{\pi_{\theta_{\text{old}}}}(x)$ a priori.
    \item The last term is a higher-order variance term that measures the error of the U-statistic in approximating the oracle estimator, and is of the same order $O(G^{-2})$ as in Theorem \ref{thm:gradMSE}. As $G$ grows to infinity, this term decays to zero at a faster rate, so that the estimator achieves the same asymptotic variance as the oracle estimator.
\end{enumerate}

Finally, we establish the consistency of the estimated policy parameter $\widehat{\theta}_n$.
\begin{prop}\label{prop:target}
    Suppose the reward signal $Z$ is binary and the value function $V^{\pi_{\theta}}$ is bounded away from 0 and 1. Then under Assumptions \ref{assump:boundedreward}, \ref{assump:lipcon}, \ref{assump:PL}, \ref{assump:LR}(b) and \ref{assump:boundedsupport}, we have $$d(\widehat{\theta}_n, \widetilde{\Theta}^*)\stackrel{P}{\to}0,$$ where $\widetilde{\Theta}^* = \arg\max_\theta \mathcal{J}(\theta)$ and
    \begin{equation}\label{eqn:objbinary}
        \mathcal{J}(\theta) = \mathbb{E}\Big[2\arcsin \big(\sqrt{V^{\pi_{\theta}}(X)}\big)\Big] -\kappa \,\textrm{KL}(\pi_{\theta}\Vert\pi_{\theta_\textrm{ref}}).
    \end{equation}
\end{prop}
We make a few remarks. First, we adopt the same asymptotic framework as in Section \ref{sec:theory} for establishing consistency, by allowing multiple optimizers to exist and by letting $n\to \infty$ while fixing all other parameters such as $B$, $G$, $\kappa$, $m$ and $\epsilon$. Second, consider the case where $\kappa=0$, so that the KL penalty is not imposed. Interestingly, according to \eqref{eqn:objbinary}, the policy parameter does not directly maximize the value function but rather an arcsin transformation of it. To see why, note that when $Z$ is binary, $\textrm{std}_x(Z)=\sqrt{V^{\pi_{\theta_{\textrm{old}}}}(x)[1-V^{\pi_{\theta_{\textrm{old}}}}(x)]}$. Under the stated conditions, we can show that the ground-truth gradient
\begin{eqnarray*}
    g^\dagger(\theta)=\mathbb{E} [g^\dagger(X;\theta)]=\mathbb{E}^{\pi_{\theta}}\Big[ W_t^{(g)} \frac{Z}{\textrm{std}_x(Z)} \Big]\to \mathbb{E}^{\pi_{\theta}}\Big[ \nabla_{\theta}\log \pi_{\theta}(Y^{(g)}_t|X,Y^{(g)}_{<t}) \frac{\sqrt{V^{\pi_{\theta}}(X)}}{\sqrt{1-V^{\pi_{\theta_{\textrm{old}}}}(X)}} \Big],
\end{eqnarray*}
which is precisely the derivative of $\mathcal{J}(\theta)$ in \eqref{eqn:objbinary}. This arises from reward normalization. This relationship was identified by \citet{davis2025objective}, who did not consider the use of minibatches or KL penalties. Nor did they establish parameter consistency. We formally verify their result by establishing parameter consistency under more practical settings.

To conclude this section, we remark that although we aim to match practical settings as closely as possible, two gaps remain: (i) the original GRPO uses length normalization, which further divides each advantage function $A^{(b,g)}$ by the length of $Y^{(g)}$, although this is not used in later variants such as Dr.\ GRPO and GPG; (ii) simple stochastic gradient descent algorithm is no longer used in practice, and has been replaced by more sophisticated optimizers such as Muon \citep{jordan2024muon}. We leave these gaps for future research.

\section{Proofs}
\subsection{Group relative gradient evaluation}
In this section, we provide the proofs of Lemma \ref{lem:gradientUstat}, Theorem \ref{thm:gradMSE}, Proposition \ref{prop:MSE}, and Corollaries \ref{coro:gradoracle} and \ref{coro:gradoptimal} stated in Section \ref{sec:grge}. 
\begin{proof}[Proof of Lemma \ref{lem:gradientUstat}:]
    To ease notation, we denote $W^{(g)} = \nabla_\theta\log \pi_\theta(Y^{(g)}|x)$. As discussed below Lemma \ref{lem:gradientUstat}, by setting $C^{(g)}$ in \eqref{eqn:generalgradest} to the leave-one-out group mean baseline, each individual term in the sum satisfies 
    \begin{equation*}
        W^{(g)} [Z^{(g)}-\bar{Z}^{(-g)}]=\frac{1}{G-1}\sum_{k\neq g} W^{(g)} [Z^{(g)}-Z^{(k)}]. 
    \end{equation*}
    By definition, we obtain
    \begin{equation*}
        \widehat{g}_{\rm GRPO}(x;\theta)=\frac{1}{G(G-1)}\sum_{k\neq g} W^{(g)} [Z^{(g)}-Z^{(k)}].
    \end{equation*}
    Applying a standard symmetrization argument, it follows that
    \begin{eqnarray*}
        \widehat{g}_{\rm GRPO}(x;\theta)&=&\frac{1}{2G(G-1)}\sum_{k\neq g} W^{(g)} [Z^{(g)}-Z^{(k)}]+\frac{1}{2G(G-1)}\sum_{k\neq g} W^{(k)} [Z^{(k)}-Z^{(g)}]\\
        &=&\frac{1}{2G(G-1)}\sum_{k\neq g} [W^{(g)} Z^{(g)}- W^{(g)}Z^{(k)} -  W^{(k)} Z^{(g)}+W^{(k)} Z^{(k)}]\\
        &=&\frac{1}{2G(G-1)}\sum_{k\neq g} [(W^{(g)}-W^{(k)})(Z^{(g)}-Z^{(k)})]\\
        &=&\frac{1}{2G(G-1)}\sum_{k\neq g} h\big((Y^{(k)},Z^{(k)}),(Y^{(g)},Z^{(g)})\big).
    \end{eqnarray*}
    The conclusion of Lemma \ref{lem:gradientUstat} thus follows.
\end{proof}

\begin{proof}[Proof of Theorem \ref{thm:gradMSE}:]
   Recall that the Hoeffding decomposition \eqref{eqn:Udecompose} expresses a U-statistic as the sum of the expectation of the kernel $h_0$, a first-order component, 
    \begin{eqnarray*}
        \frac{1}{n}\sum_{i=1}^n \zeta_1(X_i)=\frac{2}{n}\sum_{i=1}^n\bigl[h_1(X_i)-h_0\bigr],
    \end{eqnarray*}
    and a second-order component 
    \begin{eqnarray*}
        \frac{1}{n(n-1)}\sum_{i\neq j}\zeta_2(X_i,X_j)=\frac{1}{n(n-1)}\sum_{i\neq j} \bigl[h(X_i,X_j)-h_1(X_i)-h_1(X_j)+h_0\bigr].
    \end{eqnarray*}
With some calculations, it is straightforward to show that
\begin{itemize}
    \item[(i)] Each of $\zeta_1(X_i)$ and $\zeta_2(X_i,X_j)$ has mean zero and is therefore orthogonal (i.e., uncorrelated) to $h_0$;
    
    \item[(ii)] The first-order term $\zeta_1(X_i)$ is orthogonal to the second-order term $\zeta_2(X_k,X_g)$ for any $i,k,g$;
    
    \item[(iii)] Different first-order and second-order terms are mutually orthogonal, i.e., $\zeta_1(X_i)\perp \zeta_1(X_j)$ for $i\neq j$, and $\zeta_2(X_i,X_j)\perp \zeta_2(X_k,X_g)$ whenever $\{i,j\}\neq\{k,g\}$.
\end{itemize}
These orthogonality relations lead to the following MSE decomposition for the $U$-statistic:
\begin{eqnarray*}
    \mathrm{MSE}(U)&=&
\frac{1}{n^2}\sum_{i=1}^n \textrm{trace}\big[\mathrm{Var}\bigl(\zeta_1(X_i)\bigr)\big]
+
\frac{4}{n^2(n-1)^2}\sum_{1\le i< j\le n}\textrm{trace}\big[\mathrm{Var}\bigl(\zeta_2(X_i,X_j)\bigr)\big]\\
&=&\frac{1}{n} \textrm{trace}\big[\mathrm{Var}\bigl(\zeta_1(X_1)\bigr)\big]
+
\frac{2}{n(n-1)}\textrm{trace}\big[\mathrm{Var}\bigl(\zeta_2(X_1,X_2)\bigr)\big].
\end{eqnarray*}
Next, by Jensen's inequality,
\begin{align*}
\textrm{trace}\big[\mathrm{Var}\bigl(\zeta_2(X_1,X_2)\bigr)\big]
&=
\mathbb{E}\bigl\|
h(X_1,X_2)
-\mathbb{E}[h(X_1,X_2)\mid X_1]
-\mathbb{E}[h(X_2,X_1)\mid X_2]
+\mathbb{E}h(X_1,X_2)
\bigr\|^2 \\
&\le
4\,\mathbb{E}\|h(X_1,X_2)\|^2
+4\,\mathbb{E}\bigl\|\mathbb{E}[h(X_1,X_2)\mid X_1]\bigr\|^2 \\
&\qquad
+4\,\mathbb{E}\bigl\|\mathbb{E}[h(X_2,X_1)\mid X_2]\bigr\|^2
+4\,\|\mathbb{E}h(X_1,X_2)\|^2 \\
&\le
16\,\mathbb{E}\|h(X_1,X_2)\|^2.
\end{align*}
It follows that
\begin{eqnarray}\label{someeqn0}
    \mathrm{MSE}(U)=\frac{1}{n} \textrm{trace}\big[\mathrm{Var}\bigl(\zeta_1(X_1)\bigr)\big]
+
\frac{O\bigl(\mathbb{E}\|h(X_1,X_2)\|^2\bigr)}{n(n-1)}.
\end{eqnarray}
By setting $U$ to the GRPO gradient estimator $\widehat{g}_{\text{GRPO}}(\theta;x)$, it is immediate to see that the first-order component becomes the oracle gradient estimator so that 
\begin{eqnarray}\label{someeqn1}
    \textrm{trace}\big[\mathrm{Var}\bigl(\zeta_1(X_1)\bigr)\big]=\textrm{trace}\big[\Sigma_{\text{oracle}}(x;\theta)\big].
\end{eqnarray} 
Additionally, under the boundedness assumption \ref{assump:boundedreward}, we obtain
 \begin{eqnarray}\label{someeqn2}
     \mathbb{E}\Vert W^{(1)} - W^{(2)}\Vert^2(Z^{(1)} - Z^{(2)} )^2 
     = O\bigl(\mathbb{E}\Vert W^{(1)} - W^{(2)}\Vert^2\bigr)=O\bigl(\mathbb{E}\Vert W\Vert^2\bigr),
 \end{eqnarray}
 where the last equality follows from the fact that for any random vectors $X$ and $Y$, $\mathbb{E}\Vert X + Y\Vert^2\leq 2(\mathbb{E}\Vert X\Vert^2 +\mathbb{E}\Vert Y\Vert^2)$. Combining \eqref{someeqn2} with \eqref{someeqn0} and \eqref{someeqn1}, we have
 \begin{eqnarray*}
     \text{MSE}(\widehat{g}_{\text{GRPO}}(x;\theta)) 
     = \frac{1}{G}\text{trace}[\Sigma_{\text{oracle}}(x;\theta)] + O\left(\frac{\mathbb{E}\Vert\nabla_\theta\log\pi_\theta(Y|x)\Vert^2}{G^2}\right).
 \end{eqnarray*}
 This completes the proof of Theorem \ref{thm:gradMSE}.
\end{proof}

\begin{proof}[Proof of Proposition \ref{prop:MSE}]
    By definition, $\text{MSE}(\widehat{g}_{\text{GRPO}}(\theta))$ equals
    \begin{eqnarray}
        &&\mathbb{E}\|\widehat{g}_{\text{GRPO}}(\theta)-g(\theta) \|^2\nonumber\\
        &=& \mathbb{E}\left\Vert \frac{1}{B}\sum_{b=1}^B [\widehat{g}_{\text{GRPO}}(X^{(b)};\theta) - g(X^{(b)};\theta)]+\frac{1}{B}\sum_{b=1}^B [g(X^{(b)};\theta) - g(\theta)]\right\Vert^2\nonumber.
    \end{eqnarray}
    Since $g(X^{(b)};\theta)$ equals the conditional mean of $\widehat{g}_{\text{GRPO}}(X^{(b)};\theta)$ given $X^{(b)}$, $\widehat{g}_{\text{GRPO}}(X^{(b)};\theta) - g(X^{(b)};\theta)$ is orthogonal to $g(X^{(b)};\theta) - g(\theta)$. It follows that
    \begin{eqnarray*}
        \text{MSE}(\widehat{g}_{\text{GRPO}}(\theta))=\mathbb{E}\left\Vert \frac{1}{B}\sum_{b=1}^B [\widehat{g}_{\text{GRPO}}(X^{(b)};\theta) - g(X^{(b)};\theta)]\right\Vert^2+\mathbb{E}\left\Vert\frac{1}{B}\sum_{b=1}^B [g(X^{(b)};\theta) - g(\theta)]\right\Vert^2\\
        =\frac{1}{B}\mathbb{E}\left\Vert \widehat{g}_{\text{GRPO}}(X;\theta) - g(X;\theta)\right\Vert^2 + \frac{1}{B}\mathbb{E}\left\Vert g(X;\theta) - g(\theta)\right\Vert^2.
    \end{eqnarray*}
    Applying the conclusion of Theorem \ref{thm:gradMSE}, we obtain
    \begin{eqnarray*}
        \mathbb{E}\|\widehat{g}_{\text{GRPO}}(\theta)-g(\theta) \|^2 &=& \frac{\mathbb{E}\left\Vert g(X;\theta) - g(\theta)\right\Vert^2}{B}\\
        &&\qquad+\frac{ \text{trace}[\Sigma_{\text{oracle}}(\theta)]}{BG} + O\left(\frac{\mathbb{E}\Vert\nabla_\theta\log\pi_\theta(Y|X)\Vert^2}{BG^2}\right).
    \end{eqnarray*}
    This completes the proof of Proposition \ref{prop:MSE}.
\end{proof}

\begin{proof}[Proof of Corollary \ref{coro:gradoracle}]
    The conclusion directly follows from Theorem \ref{thm:gradMSE} and Proposition \ref{prop:MSE} by letting $G\to \infty$.
\end{proof}

\begin{proof}[Proof of Corollary \ref{coro:gradoptimal}]
    Any gradient estimator $\widehat{g}(x;\theta)$ whose baseline term $b(x)$ is a function of the prompt $x$ only is an unbiased estimator of $g(x;\theta)$. Therefore, its MSE can be represented as:
    \begin{eqnarray}\label{eqn:mse-general-decomp}
        \text{MSE}(\widehat{g}(\theta;x))=\text{trace}\big[\text{Var}(\widehat{g}(\theta;x))\big] = \mathbb{E}\Vert\widehat{g}(\theta;x)\Vert^2 - \Vert g(\theta;x)\Vert^2.
    \end{eqnarray}
    The first term on the RHS can be further decomposed as
    \begin{eqnarray}\label{eqn:second-moment-decomp}
    \begin{split}
        \mathbb{E}\Vert\widehat{g}(\theta;x)\Vert^2=&\frac{1}{G}\mathbb{E}\left\{\Vert\nabla_\theta\log\pi_\theta(Y|x)\Vert^2[Z-V^{\pi_\theta}(x) + V^{\pi_\theta}(x) - b(x)]^2\right\}\\
        =&\frac{1}{G}\mathbb{E}\left\{\Vert\nabla_\theta\log\pi_\theta(Y|x)\Vert^2[Z-V^{\pi_\theta}(x)]^2\right\} \\
        +&  \frac{1}{G}\mathbb{E}\left\{\Vert\nabla_\theta\log\pi_\theta(Y|x)\Vert^2 [V^{\pi_\theta}(x) - b(x)]^2\right\}\\
        +&\frac{1}{G}\mathbb{E}\left\{\Vert\nabla_\theta\log\pi_\theta(Y|x)\Vert^2 [V^{\pi_\theta}(x) - b(x)][Z-V^{\pi_\theta}(x)]\right\}.
    \end{split}
    \end{eqnarray}
    Notice that the second line of \eqref{eqn:second-moment-decomp} equals $\text{trace}[\Sigma_{\text{oracle}}(x;\theta)]$. 
    Under Assumption \ref{assump:CU} that $Z$ and $\Vert\nabla_\theta\log\pi_\theta(Y|x)\Vert$ are conditionally independent given $x$, the interaction term (i.e., the last line of \eqref{eqn:second-moment-decomp}) equals zero. This together with \eqref{eqn:mse-general-decomp} and \eqref{eqn:second-moment-decomp} yields that
    \begin{eqnarray}\label{eqn:cor1-statement1}
    \begin{split}
        \text{MSE}(\widehat{g}(x;\theta)) = &\frac{1}{G}\text{trace}[\Sigma_{\text{oracle}}(x;\theta)] +\frac{1}{G}\mathbb{E}\left\{\Vert\nabla_\theta\log\pi_\theta(Y|x)\Vert^2 [V^{\pi_\theta}(x) - b(x)]^2\right\}\\
        \geq& \frac{1}{G}\text{trace}[\Sigma_{\text{oracle}}(x;\theta)].
    \end{split}
    \end{eqnarray}
    Since Corollary \ref{coro:gradoracle} implies $\text{MSE}_A(\widehat{g}_{\text{GRPO}}(x;\theta)) = G^{-1}\text{trace}[\Sigma_{\text{oracle}}(x;\theta)]$, we obtain
    \begin{equation}
        \text{MSE}_{A}(\widehat{g}_{\rm GRPO}(x;\theta))\leq\text{MSE}(\widehat{g}(x;\theta)).\nonumber%\quad \text{and} \quad\text{MSE}_{A}(\widehat{g}_{\rm GRPO}(\theta))\leq\text{MSE}(\widehat{g}(\theta)).
    \end{equation}
    Moreover, when the baseline term $b(x)\equiv 0$, the above inequality becomes strict given that $V^{\pi_\theta}(x)\neq 0$ and $\nabla_\theta\log\pi_\theta(Y|x)$ is not almost surely zero. 
    
    Next, consider $\widehat{g}(\theta)=\sum_b \widehat{g}(X^{(b)};\theta)/B$ in the minibatch setting. Following a similar argument to the proof of Proposition \ref{prop:MSE}, we obtain 
    \begin{eqnarray*}
        \mathbb{E}\|\widehat{g}(\theta)-g(\theta) \|^2 &=& \frac{\mathbb{E}\left\Vert g(X;\theta) - g(\theta)\right\Vert^2}{B} + \frac{1}{B}\mathbb{E}\Vert \widehat{g}(X;\theta) - g(X;\theta)\Vert^2.
    \end{eqnarray*}
    Applying inequality \eqref{eqn:cor1-statement1}, we obtain
    \begin{equation*}
        \text{MSE}(\widehat{g}(\theta))\geq \frac{\mathbb{E}\left\Vert g(X;\theta) - g(\theta)\right\Vert^2}{B} + \frac{1}{BG}\text{trace}[\Sigma_{\text{oracle}}(\theta)] = \text{MSE}_A(\widehat{g}_{\rm GRPO}(\theta)).
    \end{equation*}
    For a zero baseline function, the above inequality becomes strict given that $V^{\pi_\theta}(X)$ and $\nabla_\theta\log\pi_\theta(Y|X)$ are not almost surely zero. This completes the proof of Corollary \ref{coro:gradoptimal}.
\end{proof}

\subsection{Group relative policy optimization}
This section presents the proofs of Lemma \ref{lemma:policynonasympotic}, Theorems \ref{thm:policynonasympotic} and \ref{thm:policyasympotic}, and Corollaries \ref{coro:policyoracle} and \ref{coro:policyoptimal}. 
\begin{proof}[Proof of Lemma \ref{lemma:policynonasympotic}]
    Let $J(\theta) = \mathbb{E}^{\pi_\theta}(Z)$ denote our objective function defined on $\theta \in \Theta$. Its gradient is given by $\nabla_{\theta} J(\theta)=g(\theta)$ Recall that the parameter is updated according to $\theta_{n+1} = \theta_n+\eta_n\widehat{g}(\theta_n)$. 
    Under the $L$-smoothness condition (Assumption \ref{assump:lipcon}), we have
    \begin{equation*}
        J(\theta_{n+1}) = J(\theta_n+\eta_n\widehat{g}(\theta_n))\geq J(\theta_n) +\eta_n g^\top(\theta_n)\widehat{g}(\theta_n) - \frac{1}{2}L\eta_n^2\Vert\widehat{g}(\theta_n)\Vert^2.
    \end{equation*}
    Notice that $\widehat{g}(\theta)$ is unbiased to $g(\theta)$. Conditioned on $\theta_n$ and take conditional expectation on both sides, we obtain
    \begin{eqnarray}\label{eqn:someeqn3}
        \mathbb{E}[J(\theta_{n+1})|\theta_n] &\geq& J(\theta_n) +\eta_n\Vert g(\theta_n)\Vert^2 - \frac{1}{2}L\eta_n^2\mathbb{E} \bigr[\Vert\widehat{g}(\theta_n)\Vert^2|\theta_n\bigr]. %\frac{1}{2}L\eta_n^2\left(\Vert \nabla_\theta J(\theta_n)\Vert^2 + \text{MSE}(\widehat{g}_n(\theta_n))\right).\nonumber
    \end{eqnarray}
    The last term on the RHS can be represented by $-L\eta_n^2 \big[\textrm{MSE}(\widehat{g}(\theta_n))+\|g(\theta_n)\|^2\big] /2$. Let $\theta^*$ be a maximizer of $J(\theta)$. Then, rearranging the terms in \eqref{eqn:someeqn3} yields:
    \begin{eqnarray}
        J(\theta^*) - \mathbb{E}[J(\theta_{n+1})|\theta_n] &\leq& J(\theta^*) -J(\theta_n) - \eta_n\Vert g(\theta_n)\Vert^2 \nonumber\\
        &&\qquad +\frac{1}{2}L\eta_n^2\Vert g(\theta_n)\Vert^2 +\frac{1}{2}L\eta_n^2\text{MSE}(\widehat{g}_n(\theta_n)).\nonumber
    \end{eqnarray}
    It follows that
    \begin{equation*}
        \mathbb{E}[\Delta(\pi_{\theta_{n+1}})|\theta_n]\leq \Delta(\pi_{\theta_n}) - \left(\eta_n - \frac{L\eta_n^2}{2}\right)\Vert  g(\theta_n)\Vert^2 + \frac{L\eta_n^2}{2}\text{MSE}(\widehat{g}(\theta_n))
    \end{equation*}
    Under the PL condition (Assumption \ref{assump:PL}), we obtain
    \begin{equation}\label{eqn:nonasymp-iter}
        \mathbb{E}[\Delta(\pi_{\theta_{n+1}})|\theta_n]\leq \left(1 - 2\mu\eta_n +\mu L\eta_n^2\right)\Delta(\pi_{\theta_{n}}) +\frac{1}{2}L\eta_n^2M,
    \end{equation}
    where $M$ denotes the uniform upper bound of the MSEs. 
    
    Based on this inequality, we derive suboptimality gap bounds for the two learning-rate schedules separately, as follows.

    \smallskip 

    \noindent \textbf{(i) The constant schedule}. Under the constant schedule where $\eta_i = \beta$, \eqref{eqn:nonasymp-iter} becomes
    \begin{equation}\label{eqn:constanteta}
        \mathbb{E}[\Delta(\pi_{\theta_{n+1}})|\theta_n]\leq \left(1 - 2\mu\beta +\mu L\beta^2\right)\Delta(\pi_{\theta_{n}}) +\frac{1}{2}L\beta^2M.
    \end{equation}
    By Lemma \ref{lem:PL-condition} in Section \ref{sec:lemmas}, we have $\mu \le L$. Consider the quadratic equation $1 - 2\mu\beta +\mu L\beta^2=0$ with $\beta$ treated as the variable. Its discriminant equals $\mu^2-4\mu L\le 0$ so that the equation has at most one real solution. As such, we have $\rho:=1 - 2\mu\beta +\mu L\beta^2\ge 0$. Additionally, since $\beta<(2L)^{-1}$ (Assumption \ref{assump:LR}(a)), $\rho$ is strictly smaller than $1$. 
    
    Taking expectations on both sides of \eqref{eqn:constanteta} with respect to $\theta_n$ and unrolling the recursion yields
    \begin{eqnarray*}
         \mathbb{E}[\Delta(\pi_{\theta_{n+1}})] \leq \rho^n \mathbb{E}[\Delta(\pi_{\theta_0})] + \frac{1}{2}L\beta^2M\sum_{k=1}^{n} \rho^{k-1}=\rho^n\mathbb{E}[\Delta(\pi_{\theta_0})] + \frac{L\beta^2M}{2(1-\rho)}.
    \end{eqnarray*}
    This completes the proof under the constant schedule. 

    \smallskip

    \noindent \textbf{(ii) The $1/i$ schedule}. When $\eta_i = \beta/i$ for some constant $\beta$, \eqref{eqn:nonasymp-iter} becomes
    \begin{equation*}
        \mathbb{E}[\Delta(\pi_{\theta_{n+1}})|\theta_n]\leq \left(1 - \frac{2\mu\beta}{n} + \frac{\mu L\beta^2}{n^2}\right)\Delta(\pi_{\theta_{n}}) +\frac{L\beta^2M}{2n^2}.
    \end{equation*}
    Taking expectation on both sides yields
    \begin{equation*}
        \mathbb{E}[\Delta(\pi_{\theta_{n+1}})]\leq \left(1 - \frac{2\mu\beta}{n} + \frac{\mu L\beta^2}{n^2}\right)\mathbb{E}[\Delta(\pi_{\theta_{n}})] +\frac{L\beta^2M}{2n^2}.
    \end{equation*}
    To this end, we apply Lemma \ref{lem:nonasymptotic-iteration} (Section \ref{sec:lemmas}) by letting $a_n = \mathbb{E}[\Delta(\pi_{\theta_{n}})]$, $A = 2\mu\beta$, $B = \mu L\beta^2$ and $C = \frac{1}{2}L\beta^2 M$. Since $\beta > (2\mu)^{-1}$, we have $A>1$. Notice that the sequence $a_n$ is bounded under Assumption \ref{assump:boundedreward}. Therefore, lemma \ref{lem:nonasymptotic-iteration} implies that for any constant $\varepsilon>0$, 
    \begin{equation*}
        \mathbb{E}[\Delta(\pi_{\theta_{n}})] \leq \max\left\{\frac{(1+\varepsilon)L\beta^2M}{(4\mu\beta -2)n}, \frac{c}{n}\right\}
    \end{equation*}
    holds for all $n \geq 1$, where $c$ is a constant only depends on $\mu,\beta, L$ and $\varepsilon$. This finishes the proof of Lemma \ref{lemma:policynonasympotic}.
\end{proof}

\begin{proof}[Proof of Theorem \ref{thm:policynonasympotic}]
    According to Lemma \ref{lemma:policynonasympotic}, the sub-optimality upper bounds depend on $\sup_{\theta}\text{MSE}(\widehat{g}(\theta))/n$. By Proposition \ref{prop:MSE}, we know that for GRPO-type algorithm,
    \begin{equation}\label{eqn:thm7-quantity}
        \sup_{\theta}\text{MSE}(\widehat{g}(\theta)) \leq  \frac{c_1}{B} + \frac{c_2}{BG}+\frac{c_3}{BG^2},
    \end{equation}
    where $c_1=\sup_{\theta} \mathbb{E} \|g(X;\theta)-g(\theta)\|_2^2$, $c_2=\sup_{\theta}\text{trace}[\Sigma_{\text{oracle}}(\theta)]$ and $c_3=O(\sup_{\theta} \mathbb{E} \|\nabla_\theta \log\pi_{\theta}(Y|X) \|^2)$. 
    
    Under a fixed sampling budget $N= BG$ per iteration, the RHS becomes
    \begin{equation*}
        \frac{c_1G}{N} +\frac{c_2}{N} +\frac{c_3}{NG}\geq \frac{c_2}{N}+\frac{2}{N}\sqrt{c_1c_3},
    \end{equation*}
    where the inequality follows from that $a^2+b^2\geq 2ab$, % Cauchy-Schwarz inequality, 
    and the equality holds if and only if $G = \sqrt{c_3/c_1}$. Similarly, under a fixed total sampling budget $\mathbf{N} = nBG$, the MSE is upper bounded by
    \begin{equation*}
        \frac{c_1 n G}{\mathbf{N}} + \frac{c_2 n}{\mathbf{N}} + \frac{c_3 n}{\mathbf{N}G} \geq \frac{c_2 n}{\mathbf{N}} + \frac{2}{\mathbf{N}}\sqrt{c_1c_3 n}.
    \end{equation*}
    Again, the the equality holds if and only if $G = \sqrt{c_3/c_1}$.
\end{proof}

\begin{proof}[Proof of Theorem \ref{thm:policyasympotic}]
    We first prove consistency. For any $\theta_1,\theta_2\in\Theta$, the distance function satisfies
\[
\bigl|d(\theta_1,\Theta^*)-d(\theta_2,\Theta^*)\bigr|
\le \|\theta_1-\theta_2\|,
\]
and is therefore continuous in $\theta$. Hence, for any $\varepsilon>0$, the level set
\[
\mathcal{N}_{\varepsilon}
:=
\left\{\theta\in\Theta:\, d(\theta,\Theta^*)\ge \varepsilon\right\}
\]
is closed. Under the compactness assumption of $\Theta$ in Assumption \ref{assump:boundedsupport}, $\mathcal{N}_{\varepsilon}$ is compact as well. Combined with the continuity of the policy learning objective $J$ (implied by Assumption \ref{assump:lipcon}), we obtain that the supremum $\sup_{\theta\in\mathcal{N}_{\varepsilon}} J(\theta)$ 
is attainable at some $\theta_{\varepsilon}\in\mathcal{N}_{\varepsilon}$.

Since $\theta_{\varepsilon}\notin\Theta^*$, we have
$\delta_{\varepsilon}:=J^*-J(\theta_{\varepsilon})>0$. Therefore,
\[
\mathbb{P}\bigl(d(\theta_n,\Theta^*)\ge \varepsilon\bigr)
\le
\mathbb{P}\bigl(\Delta(\pi_{\theta_n})\ge \delta_{\varepsilon}\bigr).
\]
Applying Markov's inequality and invoking Lemma \ref{lemma:policynonasympotic}, we obtain
\[
\mathbb{P}\bigl(d(\theta_n,\Theta^*)\ge \varepsilon\bigr)
\le
\frac{\mathbb{E}\bigl[\Delta(\pi_{\theta_n})\bigr]}{\delta_{\varepsilon}}
\le
\frac{\bar c}{n\delta_{\varepsilon}}
\to 0
\qquad \text{as } n\to\infty,
\]
where $\bar c$ is a positive constant depending on $(\beta,\mu,L,M,\varepsilon)$. Thus,
\[
d(\theta_n,\Theta^*) \xrightarrow{P} 0,
\]
which establishes the consistency of $\theta_n$.

Next, we derive the asymptotic distribution of $\Delta(\pi_{\theta_n})$. The proof proceeds in four steps:
\begin{itemize}
    \item Step (i): We strengthen consistency to almost sure convergence, i.e., 
    $d(\theta_n,\Theta^*) \stackrel{\mathrm{a.s.}}{\to} 0$.
    
    \item Step (ii): We derive the convergence rate of $\theta_n$ by showing that 
    $\mathbb{E}[d^2(\theta_n,\Theta^*)] = O(n^{-1})$.
    
    \item Step (iii): We establish the asymptotic normality of the identifiable part of $\theta_n$, $Q^\top \theta_n$.
    
    \item Step (iv): We derive the asymptotic distribution of $\Delta(\pi_{\theta_n})$.
\end{itemize}
    \textbf{Step (i): Almost sure convergence.}
    To prove this result, recall the inequality established in equation \eqref{eqn:nonasymp-iter}:
    \begin{equation}
        \mathbb{E}[\Delta(\pi_{\theta_{n+1}})|\theta_n]\leq \left(1 - 2\mu\eta_n +\mu L\eta_n^2\right)\Delta(\pi_{\theta_{n}}) +\frac{1}{2}L\eta_n^2M,\nonumber
    \end{equation}
    where $\eta_n = \beta/n$. We apply the Robbins-Siegmund theorem (see Lemma \ref{lem:Robins-Siemund}) by setting $V_n = \Delta(\pi_{\theta_n})$, $A_n = \mu L\eta_n^2$, $B_n = \frac{1}{2}L\eta_n^2M$ and $C_n = 2\mu\eta_nV_n$. It is immediate to see that
    $$\sum_{n=1}^\infty A_n = \sum_{n=1}^\infty \frac{\mu L\beta^2}{n^2} <\infty$$
    and
    $$\sum_{n=1}^\infty B_n = \sum_{n=1}^\infty \frac{LM\beta^2}{2n^2} <\infty.$$
    By the Robbins-Siegmund theorem, it follows that $V_n = \Delta(\pi_{\theta_n})$ converges to some random variable $V_\infty$ almost surely.  

    Additionally, by the suboptimality gap bound established in Lemma \ref{lemma:policynonasympotic}, an application of Fatou's lemma yields
    $$\mathbb{E}V_\infty \le\liminf_n \mathbb{E}\Delta(\pi_{\theta_n}) \leq \liminf_n c^*/n= 0,$$
    for some constant $c^*>0$. Since $V_\infty$ is non-negative, it follows that $V_\infty = 0$ almost surely. Consequently, $\Delta(\pi_{\theta_n}) \to 0$ almost surely. This implies for any $\delta> 0$,
    \begin{equation*}
        \mathbb{P}\left( \{\Delta(\pi_{\theta_n}) \ge \delta,  \text{ i.o.}\} \right) =0,
    \end{equation*}
    where $\{A_n, \text{ i.o.}\}$ represents the event that $A_n$  occurs infinitely often. Using the same argument as in the proof of consistency of $\theta_n$, we obtain for any $\varepsilon>0$ that, 
    \begin{equation*}
        \mathbb{P}\left( \{d(\theta_n,\Theta^*) \ge \varepsilon,  \text{ i.o.}\}\right) \leq \mathbb{P}\left( \{\Delta(\pi_{\theta_n}) \ge \delta_\varepsilon,  \text{ i.o.}\} \right) =0.
    \end{equation*}
    This establishes the almost sure convergence of $\theta_n$.

    \smallskip

    \noindent \textbf{Step (ii): Convergence rate of $\theta_n$:} Recall that the projection operator $\Pi_{\Theta^*}$ is defined as:
    \begin{equation*}
        \Pi_{\Theta^*}(\theta):= \arg\min_{x\in\Theta^*} d(\theta, x).
    \end{equation*}
    Since $\Theta^*$ is closed and convex (Assumption \ref{assump:Hessian}), the projection operator is well-defined. It follows from its definition that
    \begin{equation*}
        d^2(\theta_{n+1},\Theta^*) \leq \Vert \theta_{n+1} - \Pi_{\Theta^*}(\theta_n)\Vert^2.
    \end{equation*}
    According to the update rule $\theta_{n+1} = \theta_n +\eta_n \widehat{g}(\theta_n)$, we have
    \begin{eqnarray}
        d^2(\theta_{n+1},\Theta^*)&\leq& \Vert\theta_n +\eta_n \widehat{g}(\theta_n) - \Pi_{\Theta^*}(\theta_n)\Vert^2\nonumber\\
        &=& d^2(\theta_n, \Theta^*) + 2\eta_n\widehat{g}(\theta_n)^\top(\theta_n - \Pi_{\Theta^*}(\theta_n)) +\eta_n^2\Vert\widehat{g}(\theta_n)\Vert^2.\nonumber
    \end{eqnarray}
    Taking conditional expectation with respect to $\theta_n$ on both sides yields
    \begin{eqnarray}\label{eqn:rateofconvergence}
    \begin{split}
        \mathbb{E}[d^2(\theta_{n+1},\Theta^*)|\theta_n] \leq d^2(\theta_n, \Theta^*) + 2\eta_n g^\top(\theta_n)(\theta_n - \Pi_{\Theta^*}(\theta_n)) +\eta_n^2 \mathbb{E}\bigl[\Vert\widehat{g}(\theta_n)\Vert^2|\theta_n\bigr]\\
        =d^2(\theta_n, \Theta^*) + 2\eta_n g^\top(\theta_n)(\theta_n - \Pi_{\Theta^*}(\theta_n)) +\eta_n^2 \bigl[\textrm{MSE}(\widehat{g}(\theta_n))+\Vert g(\theta_n)\Vert^2\bigr]. 
    \end{split}
    \end{eqnarray}
    Under $L$-smoothness (Assumption \ref{assump:lipcon}) and weak strong concavity (Assumption \ref{ass:WSC}), we have 
    \begin{eqnarray*}
        \Vert g(\theta_n)\Vert^2=\Vert g(\theta_n)-g(\Pi_{\Theta^*}(\theta_n))\Vert^2\le L^2 \Vert\theta_n-\Pi_{\Theta^*}(\theta_n)\Vert^2=L^2d^2(\theta_n,\Theta^*),
    \end{eqnarray*}
    and 
    \begin{eqnarray*}
        g^\top(\theta_n)(\theta_n - \Pi_{\Theta^*}(\theta_n))\le -\mu d^2(\theta_n,\Theta^*).
    \end{eqnarray*}
    Setting $\eta_n = \beta/n$, and combining these inequalities with \eqref{eqn:rateofconvergence} yields
    \begin{equation*}
        \mathbb{E}[d^2(\theta_{n+1},\Theta^*)]\leq \left(1-\frac{2\mu\beta}{n} +\frac{L^2\beta^2}{n^2}\right)\mathbb{E}[d^2(\theta_n,\Theta^*)] + \frac{\beta^2M}{n^2}.
    \end{equation*}
    Under the assumption that $\beta>(2\mu)^{-1}$, we can apply Lemma \ref{lem:nonasymptotic-iteration} again to obtain
    \begin{equation*}
        \mathbb{E}[d^2(\theta_{n},\Theta^*)] = O(n^{-1}).
    \end{equation*}
    This establishes the desired convergence rate of $\theta_n$.

    \smallskip
    
    \noindent \textbf{Step (iii) Asymptotic normality of $Q^\top\theta_n:$} In overparameterized models such as LLMs, the entire parameter vector is not asymptotically normal because it is not fully identifiable. To address this, we decompose the parameter into its identifiable and non-identifiable components. Specifically, the matrix $Q$ introduced in Assumption \ref{assump:Hessian} forms an orthonormal basis for the identifiable subspace.

    Without loss of generality, we assume the Hessian matrix $H^*$ in Assumption \ref{assump:Hessian} is a diagonal matrix $\text{diag}(-\lambda_1,\ldots,-\lambda_r)$, where $\lambda_1$, $\cdots$, $\lambda_r$ are the $r$ positive eigenvalues of $-H^*$. Indeed, since $H(\theta^*)$ is symmetric, so is $H^*$. As such, there exists an orthogonal matrix $P$ such that $P^\top H^* P  = \text{diag}(-\lambda_1,\ldots,-\lambda_r)$. Consequently, we can replace $H^*$ and the projection matrix $Q$ by $P^\top H^* P$ (which is diagonal) and $QP^\top$, respectively. Let $\bar{Q} = (Q, \widetilde{Q})$ denote an orthogonal completion of $Q$ so that $\bar{Q}$ is a $d\times d$ orthogonal matrix and 
    \begin{eqnarray}\label{eqn:step3normal}
        \bar{Q}^\top H(\theta^*)\bar{Q} = \text{diag}(-\lambda_1,\ldots,-\lambda_r,0\ldots,0).
    \end{eqnarray}
    By Lemma \ref{lem:singlepoint}, the set $Q^\top\Theta^* = \{Q^\top\theta:\theta\in\Theta^* \}$ consists of a single point, which we denote by $v^*$. We aim to establish the asymptotic normality of $Q^\top\theta_n-v^*$ in this step. 
    
   Let $\gamma_n = \bar{Q}^\top \theta_n$. Then $\gamma_n$ satisfies the recursion
\begin{align}
\gamma_{n+1}
&=
\gamma_n
+ \eta_n \bar{Q}^\top g(\bar{Q}\gamma_n)
+ \eta_n \bar{Q}^\top\bigl(\widehat{g}(\theta_n) - g(\theta_n)\bigr).
\end{align}
Let $v_n = Q^\top \theta_n$ denote the first $r$ components of $\gamma_n$ (the identifiable component), and let $u_n$ denote the remaining $d-r$ components (the non-identifiable component), where $d$ is the dimension of $\theta$. Then
\begin{align}\label{eqn:vn-update}
\begin{split}
v_{n+1}
&=
v_n
+ \eta_n Q^\top g(Qv_n+\widetilde{Q}u_n)
+ \eta_n Q^\top\bigl(\widehat{g}(\theta_n) - g(\theta_n)\bigr) \\
&=
v_n
+ \frac{h(v_n,u_n)}{n}
+ \frac{\Delta M_{n+1}}{n},
\end{split}
\end{align}
where
\[
h(v,u)=\beta Q^\top g(Qv+\widetilde{Q}u),
\qquad
\Delta M_{n+1}
=
\beta Q^\top\bigl(\widehat{g}(\theta_n)-g(\theta_n)\bigr)
\]
is a martingale difference satisfying $\mathbb{E}(\Delta M_{n+1}\mid\theta_n)=0$.

We aim to apply central limit theorems (CLTs) for stochastic approximation \citep[e.g.,][Proposition B.2]{zhang2016central} to establish the asymptotic normality of $v_n$. 
To invoke this CLT (see Lemma \ref{lem:CLT}), we need to represent $v_n$ in the form
\[
v_{n+1}-v^*
=
\Big(I_r-\frac{\Sigma+o(1)}{n}\Big)(v_n-v^*)
+
\frac{e_{n+1}}{n},
\]
for some martingale difference sequence $\{e_n\}$. 
Comparing with \eqref{eqn:vn-update}, we can set $e_{n+1}$ to $\Delta M_{n+1}$. 
It remains to show that
\begin{equation}\label{eqn:hlinearization}
h(v_n,u_n)
=
-H (v_n-v^*)
+
o(\|v_n-v^*\|),
\end{equation}
for some deterministic matrix $H$.

Nonetheless, it remains challenging to prove \eqref{eqn:hlinearization}. Although the orthogonal transformation isolates the identifiable component, the update still depends on the full parameter vector. 
Consequently, the recursion for $v_n$ involves not only $v_n$ but also the non-identifiable component $u_n$, as reflected in \eqref{eqn:vn-update}. 
However, \eqref{eqn:hlinearization} requires to approximate $h(v_n,u_n)$ using only $v_n$. 
While this may seem impossible at first glance, we show below that it can be achieved under the WSC condition (Assumption \ref{ass:WSC}) by setting $H=-\beta H^*$.
    
Let $u_n^*$ denote the projection of the non-identifiable component $u_n$ onto $\Theta^*$, defined through the following sequence of mappings:
\[
u_n \;\mapsto\; \theta_n \;\mapsto\; \theta_n^* \;\mapsto\; \gamma_n^* \;\mapsto\; u_n^*.
\]
Specifically, recall that $\theta_n = \bar{Q}\gamma_n = \bar{Q}(v_n^\top, u_n^\top)^\top$. 
Let $\theta_n^* = \Pi_{\Theta^*}(\theta_n)$ 
denote the projection of $\theta_n$ onto $\Theta^*$, and define $\gamma_n^* = \bar{Q}^\top \theta_n^*$.
We decompose $\gamma_n^*$ as $\gamma_n^* = (v^{*\top}, u_n^{*\top})^\top$, 
where $v^*$ and $u_n^*$ correspond to the first $r$ and last $d-r$ components of $\gamma_n^*$, respectively.

Note that by Lemma \ref{lem:singlepoint}, $v^*$ is unique and does not depend on $n$, whereas $u_n^*$ may vary with $n$. By definition, we have $h(v^*,u_n^*)=0$. It follows from Taylor's theorem that
    \begin{eqnarray*}
        && h(v_n,u_n)=h(v_n,u_n)-h(v^*,u_n^*)=\beta Q^\top [g(Qv_n+\widetilde{Q}u_n) - g(Qv^*+\widetilde{Q}u_n^*)]\\
        &=&\beta Q^\top H(\theta^*) [Q (v_n-v^*) + \widetilde{Q}(u_n - u_n^*)]+ o(\|Q (v_n-v^*)\|+ \|\widetilde{Q}(u_n - u_n^*)\|), 
    \end{eqnarray*}
    where the first term on the second line equals $\beta Q^\top H(\theta^*) Q (v_n-v^*)=\beta H^*(v_n-v^*)$ by \eqref{eqn:step3normal}.  Moreover, the second term $=o\!\left(\|v_n-v^*\|+\|u_n-u_n^*\|\right)$, 
and somewhat surprisingly, this term further simplifies to $o(\|v_n-v^*\|)$ under the WSC condition (Assumption \ref{ass:WSC}), as we show below.
    
    Under Assumption \ref{ass:WSC}, we have
    \begin{eqnarray*}
        \mu\big[\Vert v_n-v^*\Vert^2 + \Vert u_n-u^*\Vert^2\big]
        \leq g^\top(Q\gamma_n) \bar{Q} (\gamma_n^*-\gamma_n).
    \end{eqnarray*}
    Applying Taylor's theorem again, the RHS becomes
    \begin{eqnarray*}
        -(\gamma_n-\gamma_n^*)^\top Q^\top HQ(\gamma_n-\gamma_n^*) + o(\Vert \gamma_n - \gamma_n^*\Vert^2)= -\sum_{k=1}^r \lambda_k(v_{nk}-v_{k}^*)^2+ o(\Vert \gamma_n - \gamma_n^*\Vert^2).
    \end{eqnarray*}
    Consequently, we have for sufficiently large $n$ that 
    \begin{eqnarray*}
        \sum_{k=1}^r (\epsilon -\lambda_k-\mu) (v_{nk}-v_{k}^*)^2\ge (\mu-\epsilon) \|u_n-u^*\|^2,
    \end{eqnarray*}
    for some $\epsilon<\mu$. This leads to $\|u_n-u^*\|=O(\|v_n-v^*\|)$.  

%Additionally, we have $\nabla_v h(v^*, u) = -\beta H = \text{diag}(\kappa_1,\ldots,\kappa_r)$, where $\kappa_1\geq\ldots\geq\kappa_r>0$ are the positive eigenvalues of $-\beta H(\theta^*)$. 

    To summarize, we have shown that 
    \begin{eqnarray}\label{eqn:keyconvergenceequality}
        h(v_n,u_n)=\beta H^*(v_n-v^*)+o(\|v_n-v^*\|). 
    \end{eqnarray}
    This proves \eqref{eqn:hlinearization} with $\Sigma=-\beta H^*$. Define 
    \[
    H_n = 
    \begin{cases}
    -\beta H^*-\displaystyle \frac{[h(v_n,u_n)  -  \beta H^* (v_n - v^*)](v_n - v^*)^\top}{\Vert v_n-v^*\Vert^2}, & \text{if } (v_n,u_n)\notin \bar{Q}^\top\Theta^*, \\
    -\beta H^*, & \text{if } (v_n,u_n)\in \bar{Q}^\top\Theta^*.
    \end{cases}
    \]
    It follows from \eqref{eqn:keyconvergenceequality} that $H_n$ is well-defined and satisfies $H_n\to -\beta H^*$ almost surely, given the almost sure convergence of the estimated parameters. 
    
    Additionally, we have $h(v_n,u_n) = H_n (v_n-v^*)$ by definition. This together with \eqref{eqn:vn-update} leads to
    \begin{equation*}
        v_{n+1}-v^* = \left(I_r - \frac{H_n}{n}\right)(v_n-v^*) +\frac{\Delta M_{n+1}}{n}.
    \end{equation*}
    Based on this equation, according to Lemma \ref{lem:CLT}, it remains to verify the following conditions to establish the CLT:
    \begin{enumerate}
        \item[(i)] The smallest eigenvalue of the limit of $H_n$, i.e., $-\beta H^*$, should be larger than $1/2$.
        \item[(ii)] The Lindeberg condition for martingale differences: for any $\epsilon>0$,
        \begin{eqnarray*}
        \frac{1}{n}\sum_{i=1}^n\mathbb{E}[\Vert\Delta M_{i}\Vert^2\mathbb{I}\{\Vert\Delta M_i\Vert\geq\epsilon\sqrt{n}|\mathcal{F}_{i-1}]\to0 \text{ a.s.,}
        \end{eqnarray*}
        where $\mathcal{F}_{i}$ denotes the filtration generated by the martingale process.
        \item[(iii)] Convergence of the quadratic variation: 
        \begin{equation*}
        \frac{1}{n}\sum_{i=1}^n\mathbb{E}\left[(\Delta M_i) (\Delta  M_i)^\top|\mathcal{F}_{i-1}\right] \to \Omega  \quad\text{  a.s.},
    \end{equation*}
    for some symmetric positive semidefinite (random) matrix $\Omega$. 
    \end{enumerate}
    To verify Condition (i), notice that by Lemma \ref{lem:PL-condition}, the smallest eigenvalue of $-H^*$ is lower bounded by $\mu$. Condition (i) thus holds under the assumption that $\beta>(2\mu)^{-1}$ (Assumption \ref{assump:LR}). Condition (ii) follows from the bounded reward assumption (Assumption \ref{assump:boundedreward}) and the bounded score assumption (Assumption \ref{assump:lipcon}), which together imply the boundedness of $\nabla M_i$. Condition (iii) is directly implied by Assumption \ref{assum:convergence}, with $\Omega = Q^\top \Gamma Q$. Applying Lemma \ref{lem:CLT}, we obtain
    \begin{equation*}
         v_n \stackrel{d}{\rightarrow} N(0,\Sigma),
    \end{equation*}
    where
    \begin{equation*}
        \Sigma = \int_0^\infty(e^{-(-\beta H^*-\frac{1}{2}I_r)u})^\top Q^\top\Gamma Q e^{-(-\beta H^*-\frac{1}{2}I_r)u} du.
    \end{equation*}
    This completes the proof of this step. 
    
    \smallskip

    \noindent \textbf{Step (iv) Asymptotic distribution of $\Delta(\pi_{\theta_n})$:}
    Finally, we derive the asymptotic distribution of $\Delta(\pi_{\theta_n})$.
    Recall that $\theta_n^* = \arg\min_{\theta\in\Theta^*} d(\theta_n, \theta)$. It follows from Taylor's theorem that
    \begin{eqnarray}
        \Delta(\pi_{\theta_n}) &=&  - [J(\theta_n) - J(\theta_n^*) ]\nonumber\\
        &=& -\frac{1}{2} (\theta_n - \theta_n^*)^\top H(\theta_n^*)(\theta_n - \theta_n^*) + o(\Vert\theta_n - \theta_n^*\Vert^2)\nonumber\\
        &=&-\frac{1}{2} (\theta_n - \theta_n^*)^\top H(\theta_n^*)(\theta_n - \theta_n^*) +o_p(n^{-1}),\nonumber
    \end{eqnarray}
    where the last equation follows from the rate of convergence derived in Step (ii).

    Consider the leading term on the RHS. Notice that
    \begin{eqnarray}
        &&(\theta_n - \theta_n^*)^\top H(\theta_n^*)(\theta_n - \theta_n^*)\nonumber\\
        &=&(\theta_n - \theta_n^*)^\top \bar{Q} H^* \bar{Q}^\top(\theta_n - \theta_n^*)\nonumber\\
        &=& -(v_n-v^*)^\top\text{diag}(\lambda_1,\ldots,\lambda_r)(v_n-v^*)\nonumber
        %&\to&\sum_{k=1}^r w_k \chi^2_{1,k},
    \end{eqnarray}
    Combining this with the asymptotic normality $\sqrt{n}(v_n - v^*) \stackrel{d}{\rightarrow} \mathcal{N}(0,\Sigma)$ established in Step (iii), we obtain
    \begin{equation*}
        n\Delta(\pi_{\theta_n}) \stackrel{d}{\rightarrow}\sum_{k=1}^r w_k \chi^2_{1,k}
    \end{equation*}
    where $w_k$ are eigenvalues of $\frac{1}{2}\Sigma^{1/2}\text{diag}(\lambda_1,\ldots,\lambda_r)\Sigma^{1/2}$. This completes the proof of Theorem \ref{thm:policyasympotic}.
\end{proof}

\begin{proof}[Proof of Corollary \ref{coro:policyoracle}:]
Under Assumption \ref{assump:lipcon}, the policy score $\nabla_{\theta}\log \pi_{\theta}$ is both continuous (as a function of $\theta$) and uniformly bounded. This together with the bounded reward assumption (Assumption \ref{assump:boundedreward}) enables us to invoke the dominated convergence theorem to show that $|\Omega(\theta_n)-\Omega(\theta_n^*)|\to 0$ whenever $d(\theta_n,\Theta^*)\to 0$, where we recall $\theta_n^*=\Pi_{\Theta^*}(\theta_n)$. By assumption, $\Omega(\theta^*)$ depends on $\theta^*$ only through $Q^\top \theta^* = v^*$, which is unique and invariant over $\Theta^*$ by Lemma \ref{lem:singlepoint}. Fixing any $\theta^* \in \Theta^*$, we thus obtain $\Omega(\theta_n) \to \Omega(\theta^*)$, which under Assumption \ref{assum:convergence} yields $Q^\top \Gamma Q = \Omega(\theta^*)$. 

Using the same calculations as in the proof of Proposition~\ref{prop:MSE}, we have
\[
\textrm{Cov}(\widehat{g}_{\textrm{GRPO}}(\theta^*))-\textrm{Cov}(\widehat{g}_{\text{oracle}}(\theta^*))=O(G^{-2}).
\]
Consequently, the same order of magnitude holds for $\Omega_{\mathrm{GRPO}}(\theta^*) - \Omega_{\mathrm{oracle}}(\theta^*)$, and hence for $Q^\top (\Gamma_{\mathrm{GRPO}} - \Gamma_{\mathrm{oracle}}) Q$, where $\Gamma_{\mathrm{oracle}}$ and $\Gamma_{\mathrm{GRPO}}$ denote the matrices $\Gamma$ in Assumption \ref{assum:convergence} under the oracle and GRPO algorithms, respectively. 

It follows that $\Sigma_{\mathrm{GRPO}} - \Sigma_{\mathrm{oracle}} = O(G^{-2})$, where $\Sigma_{\mathrm{GRPO}}$ and $\Sigma_{\mathrm{oracle}}$ denote the asymptotic covariance matrices of $Q^\top \widehat{\theta}_{n,\mathrm{GRPO}}$ and $Q^\top \widehat{\theta}_{n,\mathrm{oracle}}$ defined in Step (iii) of the proof of Theorem \ref{thm:policyasympotic}. 
Now, define
\[
W_a:=-\frac12\,\Sigma_a^{1/2}(\theta^*)H^*\Sigma_a^{1/2}(\theta^*),
\qquad a\in\{\mathrm{GRPO},\mathrm{oracle}\}.
\]
By assumption, both $\Omega_{\mathrm{GRPO}}(\theta^*)$ and $\Omega_{\mathrm{oracle}}(\theta^*)$ are positive definite. So are $\Sigma_{\mathrm{oracle}}$ and $\Sigma_{\mathrm{GRPO}}$. Therefore, the matrix square-root map is locally Lipschitz so that
\[
\bigl\|\Sigma_{\mathrm{GRPO}}^{1/2}
      -\Sigma_{\mathrm{oracle}}^{1/2}\bigr\|
=O\big(\bigl\|\Sigma_{\mathrm{GRPO}}
      -\Sigma_{\mathrm{oracle}}\bigr\|\big) = O(G^{-2}).
\]
Hence, it follows from the negative definiteness of $H^*$ (Assumption \ref{assump:Hessian}) that
\begin{align*}
W_{\mathrm{GRPO}}-W_{\mathrm{oracle}}
&=\frac12\Bigl(\Sigma_{\mathrm{GRPO}}^{1/2}(-H^*)\Sigma_{\mathrm{GRPO}}^{1/2}
-\Sigma_{\mathrm{oracle}}^{1/2}(-H^*)\Sigma_{\mathrm{oracle}}^{1/2}\Bigr) \\
&=\frac12\Bigl[
\bigl(\Sigma_{\mathrm{GRPO}}^{1/2}-\Sigma_{\mathrm{oracle}}^{1/2}\bigr)(-H^*)\Sigma_{\mathrm{GRPO}}^{1/2}
+\Sigma_{\mathrm{oracle}}^{1/2}(-H^*)\bigl(\Sigma_{\mathrm{GRPO}}^{1/2}-\Sigma_{\mathrm{oracle}}^{1/2}\bigr)
\Bigr],
\end{align*}
if of the order $O(G^{-2})$ as well.

By Weyl's inequality,
\[
|w_{k,\mathrm{GRPO}}-w_{k,\mathrm{oracle}}|
\le
\|W_{\mathrm{GRPO}}-W_{\mathrm{oracle}}\|,
\qquad k=1,\dots,r.
\]
Therefore,
\[
w_{k,\mathrm{GRPO}}-w_{k,\mathrm{oracle}}=O(G^{-2}).
\]
The proof is hence completed. 
\end{proof}

\begin{proof}[Proof of Corollary \ref{coro:policyoptimal}:]
    Using the same calculations as in the proof of Corollary~\ref{coro:gradoptimal}, we can show that $\Omega_{\mathrm{oracle}}(\theta^*) \preceq \Omega(\theta^*)$ in the positive semidefinite order, under the conditional uncorrelation assumption in Corollary \ref{coro:policyoptimal}, where $\Omega(\theta^*)$ denotes the projected covariance matrix $\mathrm{Cov}(\widehat{g}(\theta^*))$ for a given meta-algorithm whose baseline is a function of $X$ only. The remainder of the proof follows very similarly to that of Corollary \ref{coro:policyoracle}.
\end{proof}

\subsection{Practical considerations}\label{sec:pracproof}
This section presents the proofs of Lemma \ref{lem:pgobjective}, Theorem \ref{thm:practicalgradMSE} and Proposition \ref{prop:target}.
\begin{proof}[Proof of Lemma \ref{lem:pgobjective}]
    By the law of large numbers, $\text{se}(Z) \stackrel{P}\to \text{std}_x(Z)$ as $G\to\infty$. Under the assumption that $\text{std}_x(Z)$ is bounded away from zero, it follows from the continuous mapping theorem that $1/\text{se}(Z) \stackrel{P}{\to} 1/\text{std}_x(Z)$. Under the bounded reward assumption (Assumption \ref{assump:boundedreward}), the bounded score assumption (Assumption \ref{assump:lipcon}) and the coverage assumption (Assumption \ref{assump:coverage}), it follows from that the gradient estimator is asymptotically equivalent to a version with $\text{se}(Z^{(\bullet)})$ replaced by $\text{std}(Z)$, as $G\to \infty$. Following arguments similar to those in the proof of Lemma \ref{lem:gradientUstat}, this version is a U-statistic with the kernel function given in Lemma \ref{lem:pgobjective}. This completes the proof.
\end{proof}

\begin{proof}[Proof of Theorem \ref{thm:practicalgradMSE}]
    According to the bias-variance decomposition, we have
    \begin{equation*}
         \text{MSE}(\widehat{g}(x;\theta)) =\Vert b(x;\theta)\Vert^2 +\text{trace}(\text{Var}(\widehat{g}(x;\theta)),
    \end{equation*}
    where $b(x;\theta):=\mathbb{E}[\widehat{g}(x;\theta)]-g^\dagger(x;\theta)$.
    
    Following similar arguments to those in the proof of Theorem \ref{thm:gradMSE}, we can derive that 
    $$
    \text{trace}(\text{Var}(\widehat{g}^\dagger(x;\theta)) =\frac{\text{trace}[\Sigma_{\text{1}}(x;\theta)]}{G}+O\Big(\frac{1}{\epsilon G^2}\Big). 
    $$
    It remains to upper bound the squared bias term. Notice that the bias can be represented by
    \begin{eqnarray*}
        b(x;\theta) = \sum_t \mathbb{E}\Big[(\mu_t^{(g)} - \mu^{(g)})W_t^{(g)}\frac{Z^{(g)} -\bar{Z}^{(-g)}}{\text{std}(Z)}\Big]+\kappa \mathbb{E} \Big[(1-\mu^{(g)}) \sum_t (\mu_t^{(g)}-1) W_t^{(g)}\Big].
    \end{eqnarray*}
    Since $\nabla_{\theta}\pi_{\theta}=\pi_{\theta}\nabla \log \pi_{\theta}$, it follows from the bounded score condition in Assumption \ref{assump:lipcon} that $\pi_{\theta}$ is Lipchitz continuous as a function of $\theta$. Under Assumptions \ref{assump:boundedreward}, \ref{assump:lipcon} and \ref{assump:coverage}, the gradient $\widehat{g}$ is bounded. As such, 
    \begin{eqnarray}\label{eqn:diffthetaold}
        \pi_{\theta}-\pi_{\theta_{\text{old}}}=O(\|\theta-\theta_{\text{old}}\|)=O(m\eta_{i+1}).
    \end{eqnarray}
    This together with the coverage assumption yields
    \begin{eqnarray}\label{eqn:Cauchy1}
        \mathbb{E}^{\pi_{\theta_{\text{old}}}}|\mu_t^{(g)} - 1|^2 = \mathbb{E}^{\pi_{\theta_{\text{old}}}}\left\{\frac{(\pi_\theta(Y_t^{(g)}|x,Y_{<t}^{(g)}) - \pi_{\theta_{\text{old}}}(Y_t^{(i)}|x,Y_{<t}^{(i)}))^2}{\pi_{\theta_{\text{old}}}^2(Y_t^{(i)}|x,Y_{<t}^{(i)})}\right\} \\ \le \frac{c m^2\eta_{i+1}^2}{\epsilon}\mathbb{E}^{\pi_{\theta_{\text{old}}}} [\pi_{\theta_{\text{old}}}^{-1}(Y_t^{(i)}|x,Y_{<t}^{(i)})]=O\Big(\frac{m^2 \eta_{i+1}^2}{\epsilon}\Big),
    \end{eqnarray}
    for some constant $c>0$. 
    
     Following a similar argument, we have
     \begin{eqnarray}\label{eqn:Cauchy2}
         \mathbb{E}^{\pi_{\theta_{\text{old}}}}|\mu^{(g)} - 1|^2=O\Big(\frac{m^2 \eta_{i+1}^2}{\epsilon}\Big).
     \end{eqnarray}
     It follows the Cauchy-Schwarz inequality that
    \begin{eqnarray}\label{eqn:thm12eq1}
    \begin{split}
        b^2(x;\theta)\le &2\Big\Vert\sum_t \mathbb{E}^{\pi_{\theta_{\text{old}}}}\Big\{(1 - \mu^{(g)})W_t^{(g)}\Big[\frac{Z^{(g)} -\bar{Z}^{(-g)}}{\text{std}_x(Z)}+\kappa(\mu_t^{(g)}-1)\Big]\Big\} \Big\Vert^2\\
        +&2\Big\Vert\sum_t \mathbb{E}^{\pi_{\theta_{\text{old}}}}\Big[(1 - \mu_t^{(g)})W_t^{(g)}\frac{Z^{(g)} -\bar{Z}^{(-g)}}{\text{std}_x(Z)}\Big]\Big\Vert^2\\ 
        =&O\Big(\frac{\kappa^2}{\epsilon^2}\mathbb{E}^{\pi_{\theta_{\text{old}}}}|\mu^{(g)} - 1|^2\Big)+O\Big(\sum_t \mathbb{E}^{\pi_{\theta_{\text{old}}}}|\mu_t^{(g)} - 1|^2\Big),
    \end{split}
    \end{eqnarray}
    where the last equality holds under the bounded reward assumption (Assumption \ref{assump:boundedreward}), the bounded score assumption (Assumption \ref{assump:lipcon}), the coverage assumption (Assumption \ref{assump:coverage}) and that $\text{std}_X(Z)$ is uniformly bounded away from $0$. 
    
    Combining \eqref{eqn:Cauchy1} -- \eqref{eqn:thm12eq1}, we obtain 
    \begin{eqnarray*}
        b^2(x;\theta)=O\Big(\Big[\frac{\kappa^2}{\epsilon^2}+1\Big]\frac{m^2 \eta_{i+1}^2}{\epsilon}\Big).
    \end{eqnarray*}
    This completes the proof of Theorem~\ref{thm:practicalgradMSE}.
\end{proof}

\begin{proof}[Proof of Proposition \ref{prop:target}]
    As discussed below Proposition \ref{prop:target}, the ground-truth gradient $g^\dagger(\theta)$ is asymptotically equivalent to the derivative of $\mathcal{J}(\theta)$ by \eqref{eqn:diffthetaold}. The rest of the proof follows very similarly to those of Lemma \ref{lemma:policynonasympotic} and Theorem \ref{thm:policyasympotic}. We omit the details for brevity. 
\end{proof}

\subsection{Auxiliary lemmas and their proofs}\label{sec:lemmas}

\begin{lemma}\label{lem:nonasymptotic-iteration}
    Consider a sequence $\{a_n\}_n$ that satisfies $0\leq a_n\leq a$ for some $a>0$ and
    \begin{equation}\label{eqn:lemma-recursion}
        a_{n+1}\leq \left(1-\frac{A}{n} +\frac{B}{n^2}\right)a_n + \frac{C}{n^2}, \quad \forall n\geq 0,
    \end{equation}
    where $A, B, C$ are three constants satisfying $A>1$ and $B,C\geq0$. Then for any constant $\varepsilon>0$, we have
    \begin{equation}\label{eqn:lemma-iteration-conclusion}
        a_n \leq \max\left\{\frac{(1+\varepsilon)C}{(A-1)n}, \frac{a n_0}{n}\right\}
    \end{equation}
    holds for all $n$, where $n_0$ is a constant that depends only on $A, B$ and $\varepsilon$. 
\end{lemma}

\begin{proof}[Proof of Lemma \ref{lem:nonasymptotic-iteration}]
    For any fixed $\varepsilon$, let $n_0$ denote $\lceil\frac{B(1+\varepsilon)}{\varepsilon(A-1)} \rceil+1$ where $\lceil x \rceil$ represents the smallest integer that is larger than $x$. 
    We prove \eqref{eqn:lemma-iteration-conclusion} by induction:

    \smallskip
    
    \noindent \textbf{1. Base step:} For any $n \le n_0$, it directly follows from the condition $a_n\le a$ that $a_n\le a n_0/n$. The inequality thus holds for any $n\le n_0$. 

    \smallskip

    \noindent \textbf{2. Induction step:} Suppose inequality \eqref{eqn:lemma-iteration-conclusion} holds for $n =k$. We aim to show it holds for $n = k+1$ as well. To ease notation, let $M = \max\{(1+\varepsilon)C/(A-1), an_0\}$. Then according to inequality \eqref{eqn:lemma-recursion},
    \begin{eqnarray}
        a_{k+1} - \frac{M}{k+1} &\leq& \left(1-\frac{A}{k} +\frac{B}{k^2}\right)a_k + \frac{C}{k^2} - \frac{M}{k+1}.\nonumber
    \end{eqnarray}
    Applying the induction hypothesis that $a_k \leq M/k$ yields
    \begin{eqnarray}
         a_{k+1} - \frac{M}{k+1} &\leq& \left(1-\frac{A}{k} +\frac{B}{k^2}\right)\frac{M}{k}+ \frac{C}{k^2} - \frac{M}{k+1}\nonumber\\
         &=& \frac{M}{k(k+1)} - \frac{AM-C}{k^2} +\frac{BM}{k^3}\nonumber\\
         &=&\frac{1}{k^3}\left[\frac{k^2M}{k+1}-(AM-C)k+BM\right]\nonumber\\
         &\leq& \frac{1}{k^3}\left[kM-(AM-C)k+BM\right]\nonumber\\
         &=& \frac{1}{k^3}\left\{k[(1-A)M + C] +BM\right\}\nonumber.
         %&\leq& \frac{1}{k^3}\left\{\frac{B(1+\varepsilon)}{\varepsilon(A-1)}[(1-A)M+C] +BM\right\}
    \end{eqnarray}
    Since  $M \geq \frac{(1+\varepsilon)C}{A-1}$, we have $(1-A)M +C\leq -\varepsilon C\leq0$. Combining the fact that $k > \frac{B(1+\varepsilon)}{\varepsilon(A-1)}$, we obtain
    \begin{eqnarray*}
        a_{k+1} - \frac{M}{k+1} &\leq& \frac{1}{k^3}\left\{\frac{B(1+\varepsilon)}{\varepsilon(A-1)}[(1-A)M + C] +BM\right\}\\
        &=&\frac{B}{k^3}\left\{\frac{C(1+\varepsilon)}{A-1} -M\right\} \le 0,
    \end{eqnarray*}
    where the last inequality follows from the definition of $M$ as well. 

    \smallskip
    
    \noindent Combining the results from both steps completes the proof of Lemma \ref{lem:nonasymptotic-iteration}.
\end{proof}

\begin{lemma}\label{lem:singlepoint}
    Assume the Hessian matrix $H^*$ in Assumption \ref{assump:Hessian} is a diagonal matrix $\text{diag}(-\lambda_1,\ldots,-\lambda_r)$, where $\lambda_1$, $\cdots$, $\lambda_r$ are the $r$ positive eigenvalues of $-H^*$. 
    Under Assumption \ref{assump:Hessian},  the set $Q^\top\Theta^* = \{Q^\top\theta:\theta\in\Theta^* \}$ consists of a single point $v^*$. 
\end{lemma}

\begin{proof}[Proof of Lemma \ref{lem:singlepoint}]
    Suppose there are two different points $v_1 ,v_2 \in Q^\top\Theta^*$ with $v_1= Q^\top\theta_1^*, v_2 = Q^\top\theta_2^*$. Applying Taylor's expansion leads to,
    \begin{eqnarray}\label{eqn:taylorexpansion}
        J(\theta_1^*) - J(\theta_2^*) = \nabla_\theta J(\theta_2^*) +\frac{1}{2}(\theta_2^* - \theta_1^*)^\top H(\widetilde{\theta}^*)(\theta_2^*-\theta_1^*),
    \end{eqnarray}
    for some $\widetilde{\theta}^*$ lying between $\theta_1^*$ and $\theta_2^*$. Since $\Theta^*$ is convex (Assumption \ref{assump:Hessian}), we have $\widetilde{\theta}^*\in \Theta^*$. Given that $J(\theta_1^*) = J(\theta_2^*)$ and $\nabla J(\theta_2^*) = 0$, it follows from \eqref{eqn:taylorexpansion} that
    \begin{equation*}
        0 = (\theta_2^* - \theta_1^*)^\top H(\widetilde{\theta}^*)(\theta_2^*-\theta_1^*) = (\theta_2^* - \theta_1^*)^\top \bar{Q} \text{diag}(\lambda_1,\cdots,\lambda_r,\bm{0}^\top) \bar{Q}^\top(\theta_2^*-\theta_1^*).
    \end{equation*}
    Expanding further, we obtain
    \begin{eqnarray}\label{eqn:expandfurther}
        0 = \sum_{k=1}^r\lambda_k(v_{2k}-v_{1k})^2,
    \end{eqnarray}
    where $v_{ik}$ denotes the $k$-th element of $v_i$. From \eqref{eqn:expandfurther}, it is immediate to see that $v_1 = v_2$, confirming that $Q^\top\Theta^*$ is a singleton. This completes the proof. 
\end{proof}

\begin{lemma}\label{lem:PL-condition}
    Consider a function $J(\theta)$ defined on a compact parameter set $\Theta\subseteq\mathbb{R}^d$. Suppose $J(\theta)$ is twice continuously differentiable, and its gradient vector $g(\theta)$ satisfies PL with constant $\mu$ (refer to Assumption \ref{assump:PL}). Then, for any maximizer $\theta^*\in\Theta$ of $J$, every positive eigenvalue of the negative Hessian matrix $-\nabla_\theta g(\theta^*)$ is no smaller than $\mu$. In particular, under the L-smoothness condition (Assumption \ref{assump:lipcon}), we have $\mu\le L$. 
\end{lemma}

\begin{proof}[Proof of Lemma \ref{lem:PL-condition}]
    Let $\theta^*\in\Theta$ be a maximizer of $J(\theta)$, $\lambda$ be a positive eigenvalue of $-\nabla_\theta g(\theta^*)$ and $v$ be its eigenvector with $\Vert v\Vert = 1$. Applying Talyor's theorem to the function $f(t) = J(\theta^* + tv)$ at $t = 0$, we obtain:
    \begin{equation*}
        J(\theta^* + tv) - J(\theta^*)  = tv^\top g(\theta^*) +\frac{1}{2}t^2v^\top \nabla_\theta g(\theta^*)v + o(t^2).
    \end{equation*}
    Since $\theta^*$ maximizes $J(\theta)$, we have $g(\theta^*)=0$. Moreover, since $v$ is an  eigenvector with unit $\ell_2$-norm, we have $v^\top \nabla_\theta g(\theta^*)v = -\lambda$. Therefore,
    \begin{equation}\label{eqn:lemPL-eq1}
        J(\theta^*) - J(\theta^* + tv) = \frac{1}{2}t^2\lambda + o(t^2).
    \end{equation}
    Applying the PL condition to the left hand side, we obtain
    \begin{eqnarray}\label{eqn:lemPL-eq2}
    \begin{split}
        J(\theta^*) - J(\theta^* + tv) \leq& \frac{1}{2\mu}\Vert g(\theta^*+tv)\Vert^2\\
        =& \frac{1}{2\mu}\Vert g(\theta^*+tv) - g(\theta^*)\Vert^2\\
        =& \frac{t^2}{2\mu} \Vert v^T\nabla_\theta g(\theta^*)\Vert^2 + o(t^2)\\
        =&\frac{\lambda^2t^2}{2\mu} + o(t^2).
    \end{split}
    \end{eqnarray}
    Combining \eqref{eqn:lemPL-eq1} and \eqref{eqn:lemPL-eq2} yields
\[
\frac{\lambda^2 t^2}{2\mu}
\ge
\frac{1}{2}\lambda t^2 + o(t^2).
\]
Dividing both sides by $t^2$ and letting $t \to 0$ gives
$\lambda^2 \ge \mu\lambda$. 
Since $\lambda>0$, it follows that $\lambda \ge \mu$. Therefore, all positive eigenvalues of $-\nabla_\theta g(\theta^*)$ are lower bounded by $\mu$. Using similar arguments, we can show that the smoothness constant $L$ must also satisfy $L \ge \mu$. 
\end{proof}

\begin{lemma}[Robins-Siegmund Theorem]\label{lem:Robins-Siemund}
    Let $(\Omega, \mathcal{F}, (\mathcal{F}_n)_{n\geq 1}, P)$ be a filtered probability space. For each $n\geq 1$, $V_n, A_n, B_n, C_n$ are four positive $\mathcal{F}_n$-measurable random variable such that
    \begin{equation*}
        \mathbb{E}[V_{n+1}|\mathcal{F}_n] \leq (1+A_n)V_n + B_n - C_n.
    \end{equation*}
    if $\sum_{n=1}^\infty A_n <\infty$ and $\sum_{n=1}^\infty B_n <\infty$ holds almost surely, then $V_n$ converges to some random variable $V_\infty$ almost surely. Moreover, $\sum_{n=1}^\infty C_n <\infty$ almost surely.
\end{lemma}
\begin{proof}[Proof of Lemma \ref{lem:Robins-Siemund}.]
    Refer to Theorem 1 of \citet{robbins_siegmund}. 
\end{proof}

\begin{lemma}\label{lem:CLT}
     Let $\{\Delta M_n, \mathcal{F}_n, n\geq 1\}$ is a martingale difference sequence satisfying 
    \begin{equation*}
        \frac{1}{n}\sum_{i=1}^n\mathbb{E}\left[(\Delta  M_i)(\Delta M_i)^\top|\mathcal{F}_{i-1}\right] \to \Lambda  \quad\text{  a.s.},
    \end{equation*}
    where  $\Lambda$ is a symmetric positive semidefinite random matrix and 
    \begin{equation*}
        \frac{1}{n}\sum_{i=1}^n\mathbb{E}\left[\Vert\Delta M_i\Vert^2 \mathbb{I}\{\Delta  M_i\geq \varepsilon\sqrt{n}\}|\mathcal{F}_{i-1}\right] \to 0  \quad\text{  a.s. for } \forall \varepsilon>0.
    \end{equation*}
    Let $\{H_n\}_n$ denote a sequence of random matrices satisfying $H_n \to H$ almost surely, where $H$ is a nonrandom matrix with all its eigenvalues larger than $1/2$. Then for a sequence of random vectors $\Delta_n$ that satisfies
    \begin{equation*}
        \Delta_{n+1}=\Big(I-\frac{H_n}{n}\Big)\Delta_n+\Delta M_n,
    \end{equation*}
    we have $\sqrt{n} \Delta_n \stackrel{d}{\to} N(0, \Sigma)$ where
    \begin{equation*}
        \Sigma = \int_0^\infty(e^{-(H-\frac{1}{2}I)u})^\top\Lambda e^{-(H-\frac{1}{2}I)u} du.
    \end{equation*}
\end{lemma}
\begin{proof}[Proof of Lemma \ref{lem:CLT}]
    The proof follows immediately from Propositions B.1 and B.2 of \citet{zhang2016central}.
\end{proof}

% \begin{lemma}\label{lem:CLT}
%     Let $v_n$ be a vector sequence defined on a filtered probability space $(\Omega, \mathcal{F}, (\mathcal{F}_n)_{n\geq 0}, P)$, satisfying
%     \begin{equation*}
%         v_{n+1} = v_n - \frac{h(v_n)}{n+1} + \frac{\Delta M_{n+1}}{n+1},
%     \end{equation*}
%     where $h$ is a vector-valued function, $v_0$ is a finite random vector and $\{\Delta M_n, \mathcal{F}_n, n\geq 1\}$ is a martingale difference sequence. 
%     Assume $v^*$ is an equilibrium point of $h = 0$ and $h$ is differentiable at $v^*$ and all eigenvalue of $Dh(\theta^*)$ are positive and larger than $1/2$. Moreover,
%     \begin{equation*}
%         \frac{1}{n}\sum_{k=1}^n\mathbb{E}\left[(\Delta M_k)^\top\Delta  M_k|\mathcal{F}_{k-1}\right] \to \Lambda  \quad\text{  a.s.},
%     \end{equation*}
%     where $\Lambda$ is a symmetric semi-positive definite random matrix,then
%     \begin{equation*}
%         \sqrt{n}(v_n - v^*) \to \mathcal{N}(0, \Sigma),
%     \end{equation*}
%     where
%     \begin{equation*}
%         \Sigma = \int_0^\infty(e^{-(Dh(\theta^*)-\frac{1}{2}I_d)u})^\top\Lambda e^{-(Dh(\theta^*)-\frac{1}{2}I_d)u} du.
%     \end{equation*}
% \end{lemma}

\section{Experiment Details}\label{app:prompt}
This section provides the prompts we used in Section~\ref{sec: exp}.

% \begin{table}[htbp]
% \caption{Prompt template for the In-Context
% Learning model.}\label{tab:single-turn-gpt}
% \begin{tcolorbox}[colback=white, colframe=black, boxrule=0.5pt, fontupper=\ttfamily, fontupper=\ttfamily, sharp corners]

%     Solve the arithmetic problem. Return ONLY the final integer.
    
%     \medskip
%     \textless arithmetic problem \textgreater

% \end{tcolorbox}

% \end{table}
% \begin{table}[htbp]
% \caption{Prompt template for the Base and Instruct models.}\label{tab:single-turn-gpt}
% \begin{tcolorbox}[colback=white, colframe=black, boxrule=0.5pt, fontupper=\ttfamily, fontupper=\ttfamily, sharp corners]

%     You are a calculator. Solve the arithmetic problem and return ONLY the final integer.

%     \medskip
%     Examples: \\
%     Problem: 15 + 25 =\\
%     Answer: 40
    
%     Problem: 20 + 30 - 15 =\\
%     Answer: 35
    
%     Problem: 3 * 4 =\\
%     Answer: 12
    
%     Problem: 24 / 6 =\\
%     Answer: 4
    
%     Problem: 25 - (10 + 5) =\\
%     Answer: 10
    
%     \medskip
%     Now solve:\\
%     Problem: \textless arithmetic problem \textgreater

% \end{tcolorbox}

% \end{table}

\begin{center}
\textbf{Prompt template for the In-Context Learning model.}
\end{center}

\begin{tcolorbox}[colback=white, colframe=black, boxrule=0.5pt, 
fontupper=\ttfamily, sharp corners]
Solve the arithmetic problem. Return ONLY the final integer.

\medskip
\textless arithmetic problem \textgreater
\end{tcolorbox}

\vspace{0.5em}

\begin{center}
\textbf{Prompt template for the Base and Instruct models.}
\end{center}
\begin{tcolorbox}[colback=white, colframe=black, boxrule=0.5pt, fontupper=\ttfamily, fontupper=\ttfamily, sharp corners]

    You are a calculator. Solve the arithmetic problem and return ONLY the final integer.

    \medskip
    Examples: \\
    Problem: 15 + 25 =\\
    Answer: 40
    
    Problem: 20 + 30 - 15 =\\
    Answer: 35
    
    Problem: 3 * 4 =\\
    Answer: 12
    
    Problem: 24 / 6 =\\
    Answer: 4
    
    Problem: 25 - (10 + 5) =\\
    Answer: 10
    
    \medskip
    Now solve:\\
    Problem: \textless arithmetic problem \textgreater

\end{tcolorbox}

\end{appendices}

\end{document}